%% file: iclr2026_conference.tex
\documentclass{article} 
\usepackage{iclr2026_conference,times}

\usepackage{amsthm}
\usepackage{amsmath,amsfonts,amssymb,bm,mathtools}
\usepackage{url}

\usepackage{thmtools}
\usepackage{thm-restate}

\usepackage{mdframed}

\newmdenv[
    leftline=false,
    rightline=false,
    topline=true,
    bottomline=true,
    linewidth=0.5pt,
    skipabove=0pt,
    skipbelow=0pt,
    leftmargin=0pt,
    rightmargin=0pt,
    innerleftmargin=0pt,
    innerrightmargin=0pt,
    innertopmargin=2pt,
    innerbottommargin=2pt
]{proofblock}

\input{macros}

\usepackage[hypertexnames=false]{hyperref}

\title{Global Convergence of Gradient Descent \\ for Score Matching in Gaussian Mixtures \\ via Reverse Fisher Divergence}

\author{
   Alexander Tyurin \\
   \phantom{,}AXXX, Moscow, Russia \\
   \phantom{,}Applied AI Institute, Moscow, Russia \\
}

\hypersetup{
  colorlinks   = true, 
  urlcolor     = blue, 
  linkcolor    = {red!75!black}, 
  citecolor   = {blue!50!black} 
}

\theoremstyle{plain}
\newtheorem{theorem}{Theorem}[section]
\newtheorem*{theorem*}{Theorem}

\newtheorem{lemma}[theorem]{Lemma}
\newtheorem{corollary}[theorem]{Corollary}
\theoremstyle{definition}

\theoremstyle{remark}

\usepackage{xcolor}
\definecolor{myblue}{RGB}{0,102,204}
\definecolor{myred}{RGB}{190,35,35}

\allowdisplaybreaks

\usepackage{tcolorbox}
\newtcolorbox{mytheobox}{
  boxrule=0pt,      
  left=1mm,
  right=1mm,
  top=1mm,
  bottom=1mm
}

\DeclareSymbolFont{extraup}{U}{zavm}{m}{n}
\DeclareMathSymbol{\varheart}{\mathalpha}{extraup}{86}
\DeclareMathSymbol{\vardiamond}{\mathalpha}{extraup}{87}

\iclrfinalcopy 
\begin{document}

\maketitle

\begin{abstract}
   The score matching problem is a central training objective in modern generative modeling, diffusion models, fitting unnormalized statistical models, and inverse problems. A standard approach is to minimize the \emph{forward Fisher divergence}, where the expectation is taken with respect to the teacher distribution. However, recent results show that even in simple Gaussian mixture model settings, this objective can lead to undesirable and initialization-dependent convergence behavior. In this paper, we study an alternative objective: the \emph{reverse Fisher divergence}, where the expectation is taken with respect to the student distribution. We analyze gradient descent (GD) for fitting Gaussian mixture models and show that this change in the objective leads to significantly better optimization properties. First, when the teacher distribution is a single Gaussian and the student is a Gaussian mixture model with fixed weights and identity covariances, we prove the global convergence of GD from \emph{arbitrary initializations}. Second, we extend the analysis to the case where the teacher is also a Gaussian mixture model and prove global convergence guarantees under a global random initialization scheme and a ${\scalebox{0.85}{$\widetilde{\Omega}(1)$}}$-separation assumption on the target means. In particular, with high probability, each student component converges near its closest teacher component, and we provide conditions under which the student distribution converges in total variation distance. Our proofs rely on a new Lyapunov-based analysis of the gradient descent dynamics, showing that the \emph{reverse Fisher divergence} has a much more favorable optimization landscape than the \emph{forward Fisher divergence}.
\end{abstract}

\section{Introduction}
The score matching problem has become a central training objective in many applications, including modern generative modeling, diffusion models, fitting unnormalized statistical models, and inverse problems \citep{hyvarinen2005estimation,song2019generative,chung2022diffusion}. In this task, we have two distributions, $p_{\mu}(x)$ and $\bar{p}(x)$, called the student and the teacher, respectively, where the former is parameterized by $\mu$. The main goal is to find $\mu$ such that the score function $\nabla_x \log p_{\mu}(x)$ of the student is close to the score function $\nabla_x \log \bar{p}(x)$ of the teacher. Matching scores rather than densities directly is appealing, for instance, when the target density, or both densities, are known only up to normalizing constants, since the score functions do not depend on these constants.

Empirically, score-based methods have been highly successful and have led to a broad range of applications in image generation, text-to-image generation, and audio synthesis~\citep{rombach2022high,kong2020diffwave,ramesh2022hierarchical}. However, the \emph{theoretical} properties of the score matching problem remain much less understood. \emph{How should we match the score functions? How can we guarantee the convergence of algorithms? What theoretical convergence rates can we achieve?}

In both theory and practice, a typical approach to matching scores is to minimize the \emph{forward Fisher divergence},
$\textstyle \mathcal L_{\mathrm{for}}(\mu)
=
\mathbb E_{\color{myred} x\sim \bar{p}}
\left[
\left\|
\nabla_x \log p_\mu(x)
-
\nabla_x \log \bar{p}(x)
\right\|^2
\right],$
where the expectation is taken with respect to the teacher distribution. This objective is minimized using the gradient descent (GD) method $\mu^{k+1} = \mu^k - \gamma \nabla \mathcal L_{\mathrm{for}}(\mu^k)$ (in practice, Adam~\citep{kingma2014adam} and other adaptive stochastic methods are often used, whereas in this paper we consider GD for the theoretical analysis).

\subsection{Related work}
Many theoretical works on diffusion and score-based models analyze the sampling process when an exact or sufficiently accurate score function is available~\citep{chen2022sampling,lee2023convergence,benton2024nearly,gupta2025faster}. These results provide important guarantees for the generative process, but they do not fully explain how and when gradient-based training can learn the score function. Another line of work studies the statistical and optimization properties of score estimation, including generalization and statistical complexity~\citep{koehler2022statistical,li2023generalization,wibisono2024optimal}. \citet{han2024neural} analyzed the convergence of gradient-based methods using the neural tangent kernel (NTK) approach. \citet{wang2024evaluating} analyzed GD on a finite-sum empirical approximation of the score-matching objective under an overparameterized neural network. 
\citet{gatmiry2024learning,chen2024learning} provide statistical and computational guarantees for recovering Gaussian mixtures. In contrast, our study belongs to the optimization side of this literature: we analyze the behavior and properties of GD dynamics for the score-matching problem.

\textbf{Score Matching with Gaussian Mixture Models (GMMs).} Analyzing the convergence of GD applied to the forward Fisher divergence for arbitrary distributions is nontrivial. Recent studies focus on score matching problems where the student and teacher are \emph{Gaussian mixture models} (GMMs), but even this setting leads to a complex analysis because the resulting optimization problem is nonconvex. Before we discuss the most relevant works, it is also worth mentioning that fitting GMMs has been extensively studied in the context of the Expectation-Maximization (EM) and gradient EM algorithms~\citep{xu2016global,daskalakis2017ten,dwivedi2020sharp,wu2021randomly,xu2024toward,zhou2025global}; there are several parallels with score matching for GMMs, and some EM-type proof techniques are useful in this context.
\begin{figure}[t]
    \centering
    \includegraphics[width=0.8\linewidth]{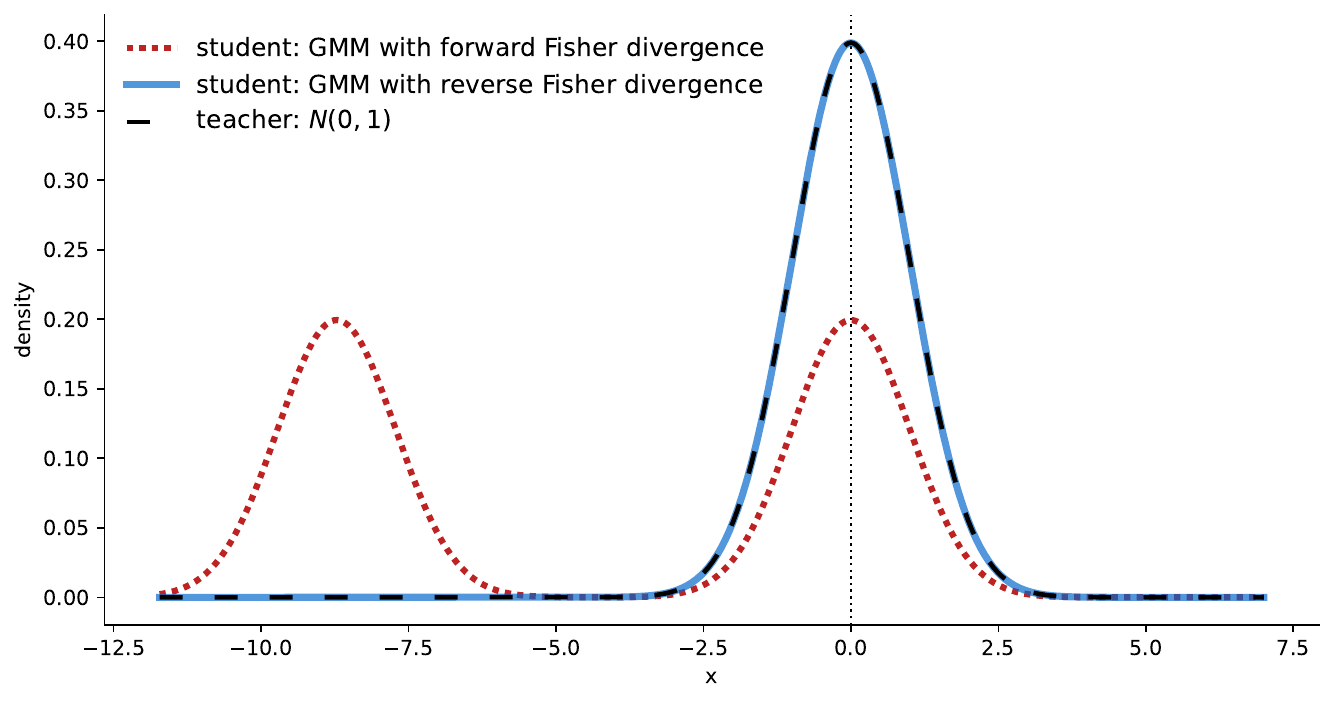}
    \caption{A result of running gradient descent (GD) on the \emph{forward} and \emph{reverse} Fisher divergences. The teacher/target distribution is a single Gaussian, while the student/model is a Gaussian mixture model with two components (see the setup in Sections~\ref{sec:prel} and \ref{sec:exp}). Even in this simple setting, GD on the \emph{forward Fisher divergence} converges to an incorrect distribution, whereas GD on the \emph{reverse Fisher divergence} almost perfectly recovers the teacher distribution.}
    \label{fig:gmm-n2-final-distributions}
\end{figure}

\citet{shah2023learning} considered the forward Fisher divergence with GD for fitting the mean parameters of a mixture of two symmetric Gaussians with constant separation. They study the DDPM/forward Fisher objective, where one can choose the noise level, and use this structure in a two-phase algorithm: first fitting the parameters in the large-noise regime and then refining them in the low-noise regime. Their random-initialization analysis is limited to the setting with two components. They also consider the problem of fitting GMMs with multiple components, but they achieve convergence only under the assumption that the initial student means are close to the ground truth. \citet{zhang2025convergence} analyzed the case where the teacher is a single Gaussian and the student is a GMM with multiple components, motivated by the fact that real models are often overparameterized and the target number of components may be unknown. Their theoretical results show that even in this simple setting, the forward Fisher divergence \emph{exhibits undesirable and unstable convergence behavior}: depending on the noise level and initialization, GD may converge to the ground truth, become stuck at a local maximum with nonzero loss, or drive only one student component toward the teacher while the remaining components diverge. 

\subsection{Motivation and contributions} 
\citet{zhang2025convergence} proved that the \emph{forward Fisher divergence} can lead to subtle and initialization-dependent optimization behavior. Ideally, a loss function should possess the following properties: \emph{absence of local maxima or spurious stationary points} and \emph{global convergence of GD}. The forward Fisher divergence clearly does \emph{not} satisfy these properties, even in the simple setting of fitting the score function of a GMM to that of a single Gaussian, which motivates the search for alternative objectives with more favorable optimization landscapes.

$\spadesuit$ In this paper, we study an alternative objective: the \emph{reverse Fisher divergence},
\begin{align}
\label{eq:main}
\textstyle \mathcal L_{\mathrm{rev}}(\mu)
=
\mathbb E_{\color{myblue} x\sim p_\mu}
\left[
\left\|
\nabla_x \log p_\mu(x)
-
\nabla_x \log \bar{p}(x)
\right\|^2
\right],
\end{align}
where the expectation is taken with respect to the student distribution instead of the teacher distribution. Although the difference between the forward Fisher divergence and \eqref{eq:main} is only the distribution under the expectation, we show that this change leads to significantly better optimization properties. 

$\clubsuit$ (\textbf{One target component}). We consider the \emph{reverse Fisher divergence} and prove the convergence of GD for score matching a student GMM distribution to a single Gaussian distribution from \emph{arbitrary initialization} (Theorem~\ref{thm:main_one}). In contrast to the previous results for the \emph{forward Fisher divergence}, which study the same score-matching setting for GMMs, our analysis establishes \emph{global convergence}. Moreover, the convergence rate is dimension-independent when $d$ is sufficiently large and depends only logarithmically on the initial distance to the target mean. These results provide strong theoretical evidence that the reverse Fisher divergence possesses a significantly more favorable optimization landscape and stronger convergence guarantees than the forward Fisher divergence.

$\vardiamond$ (\textbf{Multiple target components}). In Theorem~\ref{thm:mainmany_simple}, we extend this result to the case where the teacher is also a GMM. For arbitrary numbers $n$ and $m$ of student and teacher components, we establish \emph{global convergence guarantees} for the student mean parameters to the target means under random initialization\footnote{As expected, convergence is not possible under deterministic initialization; see Theorem~\ref{thm:folklore}.} and a $\tilde\Omega(1)$-separation assumption on the target means. In particular, we uniformly sample the initial means on a sufficiently large sphere of radius $R^0$ and show that, with high probability, each student mean converges to its closest teacher mean up to $\varepsilon$-accuracy. When the dimension satisfies $d \geq n,$ the rate is $\tilde\Theta\left(\nicefrac{n^2}{\varepsilon}\right),$ which depends only logarithmically on $R^0$ and is independent of $d.$ In Section~\ref{sec:tv}, we also analyze conditions under which the student distribution converges in total variation distance.

$\varheart$ (\textbf{Proof techniques}). In this work, we face mathematical difficulties similar to those in related work: the problem is nonstandard, with a nontrivial landscape and highly complex gradients. The fact that GD with the reverse Fisher divergence converges globally is somewhat surprising. To prove this, in Section~\ref{sec:single}, we develop a new proof technique. The main idea is to show that the method is equivalent to
$\mZ^{k+1} = \left(\mI - \frac{2 \gamma}{n} \mB^k\right) \mZ^{k}$ and to prove that the Lyapunov function $\|\mZ^{k}\|_F^2$ is non-increasing and converges at a sublinear rate. We show that $\mB^k$ is symmetric, positive semidefinite, and has bounded norm. Most importantly, it satisfies $\langle\mZ^{k},\mB^k \mZ^{k}\rangle \geq \Omega(\min\{\|\mZ^{k}\|_F^4, \|\mZ^{k}\|_F^2\}).$ In Section~\ref{sec:proof_technique}, we present new proof techniques to show convergence in the case of multiple target components, where we provide a two-stage analysis approach with careful control of errors.
\subsection{Preliminaries}
\label{sec:prel}
We consider the score matching problem, where $\bar{p}(x)$ and $p_{\mu}(x)$ denote the teacher and student distributions, respectively. Our goal is to find parameters $\mu$ such that the score function of the student distribution, $\nabla_x \log p_{\mu}(x)$, is as close as possible to the score function of the teacher distribution, $\nabla_x \log \bar{p}(x).$ Unlike most previous works, we consider the problem of minimizing the \emph{reverse Fisher divergence} \eqref{eq:main}. In our setting, we follow \citep{shah2023learning,zhang2025convergence} and assume that $\bar{p}$ and $p_{\mu}$ are GMMs of the form
\begin{align}
\label{eq:p}
\textstyle p_\mu(x)
=
\frac1n\sum\limits_{i=1}^n \varphi(x - \mu_i)
\quad \textnormal{ and } \quad
\bar{p}(x)
=
\frac1m\sum\limits_{a=1}^m \varphi(x - \bar\mu_a) \quad \forall x \in \R^d, 
\end{align}
where $\varphi(x) \eqdef (2 \pi)^{-d/2}\exp\left(-\nicefrac{\norm{x}^2}{2}\right)$ is the standard normal probability density function. Here, the target means $\bar\mu_1,\ldots,\bar\mu_m \in \R^d$ and the number of components $m \geq 1$ are fixed but unknown parameters, and $n \geq 1$ is the number of Gaussian components in the student distribution, parameterized by $\mu \equiv [\mu_1;\dots;\mu_n] \in \R^{n d}$.

\textbf{Gradient descent.} To minimize \eqref{eq:main}, we consider the gradient descent (GD) method:
\begin{align}
\label{eq:wsltovjir}
\mu^{k+1} = \mu^{k} - \gamma \nabla_\mu \mathcal L_{\mathrm{rev}}(\mu^{k}),
\end{align}
where $\gamma > 0$ is the step size, $\mu^{k} \equiv [\mu^k_1;\dots;\mu^k_n]$, and $\{\mu^0_i\}_{i \in [n]}$ are the initial means. Before we derive the gradient,
it is convenient to introduce the following notation.
Define
$\textstyle r_i(x)
\eqdef
\frac{\exp\left(-\frac12\|x-\mu_i\|^2\right)}{\sum_{j=1}^n\exp\left(-\frac12\|x-\mu_j\|^2\right)},$
$m_\mu(x)
\eqdef
\sum\limits_{i=1}^n r_i(x)\mu_i ,
$
and
$
\textstyle \mC_\mu(x)
\eqdef
\sum\limits_{i=1}^n
r_i(x)
\bigl(\mu_i-m_\mu(x)\bigr)
\bigl(\mu_i-m_\mu(x)\bigr)^\top .
$
Similarly, $\bar r_a(x)
\eqdef
\frac{
\exp\left(-\frac12\|x-\bar\mu_a\|^2\right)
}{
\sum_{b=1}^m
\exp\left(-\frac12\|x-\bar\mu_b\|^2\right)
},$ $
\bar m(x)
\eqdef
\sum_{a=1}^m \bar r_a(x)\bar\mu_a,$
and $
\bar \mC(x)
\eqdef
\sum\limits_{a=1}^m
\bar r_a(x)
\bigl(\bar\mu_a-\bar m(x)\bigr)
\bigl(\bar\mu_a-\bar m(x)\bigr)^\top.$ In Section~\ref{sec:gradient}, we derive the gradient of $\mathcal L_{\mathrm{rev}}$.
\begin{restatable}[Gradient of \eqref{eq:main} with GMMs]{proposition}{PROPGRADIENT}
\label{prop:grad-general}
Let $\nabla_{\mu}\mathcal L_{\mathrm{rev}}(\mu) \equiv [\nabla_{\mu_1}\mathcal L_{\mathrm{rev}}(\mu); \dots; \nabla_{\mu_n}\mathcal L_{\mathrm{rev}}(\mu)].$
Then, for all $\mu \in \R^{n d}$ and $\ell \in [n]$, the gradient with respect to the $\ell$\textsuperscript{th} mean component is
\[
\nabla_{\mu_\ell}\mathcal L_{\mathrm{rev}}(\mu)
=
\frac{2}{n}
\mathbb E_{x \sim \mathcal N(\mu_\ell,\mI)}
\left[
m_\mu(x)-\bar m(x)
+
\mC_\mu(x)\bigl(\mu_\ell-m_\mu(x)\bigr)
-
\bar \mC(x)\bigl(\mu_\ell-\bar m(x)\bigr)
\right].
\]
\end{restatable}

\section{Learning Single Gaussian Component}
\label{sec:single}
We present our first theoretical result for the case where the teacher distribution consists of a single component, i.e., $m = 1$. Equivalently, the teacher distribution is a Gaussian distribution. We consider the setting with multiple Gaussian components in the teacher in Section~\ref{sec:multiple}. Unlike the case of the forward Fisher divergence, where it is infeasible to establish global convergence guarantees \citep{zhang2025convergence}, we establish the \emph{global convergence} of GD for the reverse Fisher divergence from an \emph{arbitrary initialization} to the target mean $\bar\mu_1$ of the Gaussian distribution $\bar{p}.$
\begin{mytheobox}
\begin{theorem}[Global Convergence of GD]
   \label{thm:main_one}
   If $\gamma \leq n / (4 + 512 \log(n))$, then for all $\mu^0_1, \dots, \mu^0_n \in \R^d$, GD \eqref{eq:wsltovjir} with GMMs \eqref{eq:p} and $m = 1$ guarantees that
      $$\textstyle \norm{\mu^k_i - \bar\mu_1}^2
      \leq
      \max\left\{
         \left(1-\frac{\gamma}{2 n}\right)^{k/2} n (R^0)^2,
         \frac{8 n^2 \min\{n, d\}}{\gamma k}
      \right\}$$ for all $i \in [n], k \geq 1,$ where $R^0 \eqdef \max_{i \in [n]} \norm{\mu^0_i - \bar\mu_1}$.
\end{theorem}
\end{mytheobox}
We now present the proof of this result and highlight some interesting properties of the problem that enable us to establish global convergence.\\
\begin{proofblock}
\textit{Proof.} We first rewrite the update \eqref{eq:wsltovjir} minus $\bar\mu_{1}$ in matrix form to make the proof more compact and intuitive. We defer the derivations of Lemma~\ref{lemma:Z} to Section~\ref{sec:missing}.
\begin{restatable}{lemma}{LEMMAZZ}
   \label{lemma:Z}
   Let $m = 1$ and $\mZ^{k}
\eqdef
\bigl[(\mu^{k}_{1}-\bar\mu_{1})^\top;\ldots;
(\mu^{k}_{n}-\bar\mu_{1})^\top\bigr]
\in \R^{n\times d},$ stacking the vectors vertically. Using this notation and \eqref{eq:wsltovjir}, the matrix sequence $\{\mZ^{k}\}_{k \geq 0}$ evolves as
\begin{align*}
   \textstyle \mZ^{k+1} = \left(\mI - \frac{2 \gamma}{n} \mB^k\right) \mZ^{k},
\end{align*}
   where $\mB^{k} \in \R^{n \times n}$ and 
   $\mB^{k}_{\ell i} \eqdef \mathbb E_{x \sim \mathcal N(\mu^k_\ell,\mI)}
   \left[r_i^{k}(x) + r_i^{k}(x) \bigl(\mu^k_i-m_{\mu^k}(x)\bigr)^\top \bigl(\mu^k_\ell-m_{\mu^k}(x)\bigr)\right]$
is the $(\ell, i)$\textsuperscript{th} entry of $\mB^{k}$ and $r^k_i(x)
\eqdef \exp\left(-\frac12\|x-\mu^k_i\|^2\right) / \sum_{j=1}^n\exp\left(-\frac12\|x-\mu^k_j\|^2\right).$
\end{restatable}
Using Lemma~\ref{lemma:Z}, we can write the distance between the student and teacher means in matrix form:
\begin{align}
   \label{eq:ilagOzaJW}
   \textstyle \norm{\mZ^{k+1}}_F^2 = \norm{\left(\mI - \frac{2 \gamma}{n} \mB^k\right) \mZ^{k}}_F^2 = \norm{\mZ^{k}}_F^2 - \frac{4 \gamma}{n} \inp{\mZ^{k}}{\mB^k \mZ^{k}} + \frac{4 \gamma^2}{n^2} \norm{\mB^k \mZ^{k}}_F^2,
\end{align}
where $\norm{\cdot}_F$ is the standard Frobenius norm and $\inp{\cdot}{\cdot}$ is the standard matrix inner product. The matrix $\mB^k$ has the following useful properties, proved in Section~\ref{sec:missing}:
\begin{restatable}{lemma}{LEMMAB}
   \label{lemma:b}
   For all $k \geq 0$, matrix $\mB^k$ is symmetric and positive semidefinite, and $\norm{\mB^k}_{\textnormal{op}} \leq 1 + 128 \log(n).$ 
\end{restatable}
Due to Lemma~\ref{lemma:b}, \eqref{eq:ilagOzaJW}, and standard inequalities, $\norm{\mB^k \mZ^{k}}_F^2 \leq \norm{\mB^k}_{\textnormal{op}}\inp{\mZ^{k}}{\mB^k \mZ^{k}}$ and 
\begin{align*}
   \textstyle \norm{\mZ^{k+1}}_F^2 
   \leq \norm{\mZ^{k}}_F^2 - \frac{4 \gamma}{n} \inp{\mZ^{k}}{\mB^k \mZ^{k}} + \frac{(4 + 512 \log(n)) \gamma^2 }{n^2} \inp{\mZ^{k}}{\mB^k \mZ^{k}}.
\end{align*}
Since $\gamma \leq n / (4 + 512 \log(n))$ and $\inp{\mZ^{k}}{\mB^k \mZ^{k}} \geq 0,$ we get 
\begin{align}
   \textstyle \norm{\mZ^{k+1}}_F^2 &\leq \textstyle \norm{\mZ^{k}}_F^2 - \frac{\gamma}{n} \inp{\mZ^{k}}{\mB^k \mZ^{k}}. \label{eq:ICtmusbxPuxs}
\end{align}
We conclude that $\{\|\mZ^{k}\|_F^2\}_{k \geq 0}$ is non-increasing. However, this is not sufficient to prove the convergence rate, and we still need one more property of $\mB^k$ and $\mZ^k,$ proved in Section~\ref{sec:missing}:
\begin{restatable}{lemma}{LEMMAZORIG}
   \label{lemma:z_b}
   For all $k \geq 0,$
      $\textstyle \inp{\mZ^{k}}{\mB^k \mZ^{k}} \geq \min\left\{\frac{1}{4 n \min\{n, d\}} \norm{\mZ^{k}}^4_F, \frac{1}{2} \norm{\mZ^{k}}^2_F\right\}.$
\end{restatable}
Substituting this inequality into \eqref{eq:ICtmusbxPuxs},
   $$\textstyle \norm{\mZ^{k+1}}_F^2 \leq \textstyle \norm{\mZ^{k}}_F^2 - \frac{\gamma}{n} \min\left\{\frac{1}{4 n \min\{n, d\}} \norm{\mZ^{k}}^4_F, \frac{1}{2} \norm{\mZ^{k}}^2_F\right\},$$ and it remains to use Lemma~\ref{lemma:recursion} and $n \max\limits_{i \in [n]} \norm{\mu^k_i - \bar\mu_1}^2 \geq \norm{\mZ^{k}}_F^2 \geq \max\limits_{i \in [n]} \norm{\mu^k_i - \bar\mu_1}^2.$ \hfill $\square$
\end{proofblock}
\begin{mytheobox}
\begin{corollary}
   \label{cor:main_one}
   In view of Theorem~\ref{thm:main_one}, with $\gamma = n / (4 + 512 \log(n)),$ GD finds $\{\mu^k_i\}_{i \in [n]}$ such that $\norm{\mu^k_i - \bar\mu_1}^2 \leq \varepsilon$ for all $i \in [n]$ and 
   $\mathcal L_{\mathrm{rev}}(\mu^k) \leq \varepsilon$ 
   after at most 
\begin{align}
   \label{eq:rate_one}
   \textstyle 4096 \max\left\{(1 + \log(n))\log\left(\frac{n (R^0)^2}{\varepsilon}\right), \frac{(1 + \log(n)) n \min\{n, d\}}{\varepsilon}\right\}
\end{align}
   iterations.
\end{corollary}
\end{mytheobox}

\subsection{Discussion}
Theorem~\ref{thm:main_one} and Corollary~\ref{cor:main_one} provide global convergence guarantees for GD with the reverse Fisher divergence in the score-matching problem. This highlights the importance of this loss compared to the forward Fisher divergence, since the latter does not guarantee convergence in the same setting. With the forward Fisher divergence, one can guarantee convergence only when the initial means are close to the target mean (the large-noise regime) \citep{zhang2025convergence}, whereas the reverse Fisher divergence guarantees convergence in all regimes. Moreover, all $\{\mu^k_i\}_{i \in [n]}$ converge to $\bar\mu_1,$ which is an important property in the score matching problem. 

Intuitively, the reverse Fisher divergence, due to its averaging with respect to $p_{\mu}$, emphasizes the regions where $p_{\mu}$ places mass and tends to move such mass towards the target distribution $\bar{p}$, thereby yielding the global convergence of GD. On the other hand, the forward Fisher divergence places little weight on regions where $\bar{p}$ has small mass, and therefore may not penalize far-away modes of $p_{\mu}$; in some settings, it can even push them away, as formalized in \citep{zhang2025convergence}.

The convergence rate \eqref{eq:rate_one} has many nice properties: first, it does not depend on $d$ when $d$ is large; second, it has a logarithmic dependence on the initial distance $R^0;$ and third, the dependence on $\varepsilon$ is $\tilde\cO\left(\nicefrac{1}{\varepsilon}\right).$ The proof technique is arguably compact and intuitive (which we believe is a major advantage of this theory): we use the Lyapunov function $\|\mZ^{k}\|_F^2$ and show its monotonicity. The only technical parts are Lemmas~\ref{lemma:b} and \ref{lemma:z_b}, which establish the properties of the matrix $\mB^k.$ We show that $\mB^k$ is symmetric and positive semidefinite, and $\|\mB^k\|_{\textnormal{op}} \leq \tilde\cO(1).$ Moreover, in Lemma~\ref{lemma:z_b}, we prove that $\langle\mZ^{k}, \mB^k \mZ^{k}\rangle$ is larger than zero, with the gap $\Omega(\min\{\|\mZ^{k}\|^4_F, \|\mZ^{k}\|^2_F\}),$ which allows us to obtain the final two-regime convergence rate.

\section{Learning Multiple Gaussian Components}
\label{sec:multiple}
We are ready to extend our result to the setting with multiple Gaussian components in the teacher distribution, i.e., $m > 1$ in $\bar{p}.$ We will show that global convergence up to a small $\varepsilon$-accuracy is possible even in this setting. However, as expected, it can only be guaranteed with high probability and under a (non-local) random initialization. Indeed, before presenting our main result, let us first demonstrate a simple counterexample with a ``bad'' initialization.
\begin{theorem}[Folklore result]
   \label{thm:folklore}
   For $m = 2$ and $n \geq 1,$ let $\bar\mu_1 = -1,$ $\bar\mu_2 = 1,$ and $\mu^0_i = 0$ for all $i \in [n].$ Then $\mu^0$ is a stationary point of both the forward and reverse Fisher divergences. Consequently, GD initialized at $\mu^0$ does not move for all $\gamma > 0$, i.e., $\mu^k = \mu^0$ for all $k \geq 0$, and both losses remain unchanged.
\end{theorem}
This result and counterexample likely extend to other ``reasonable'' losses. Nevertheless, we can show that this example is unstable and establish global convergence under random initialization with the reverse Fisher divergence.
\begin{mytheobox}
\begin{restatable}[Theorem~\ref{thm:mainmany} when $d \geq n$]{theorem}{THMSIMPL}
\label{thm:mainmany_simple}
   Let $\varepsilon \in (0, 1]$ denote the target accuracy. We define the smallest and largest gaps in the target distribution as
      $\delta_{\min} \eqdef \min_{a \neq b}\norm{\bar\mu_a - \bar\mu_b} > 0$ and $\delta_{\max} \eqdef \max_{a \neq b}\norm{\bar\mu_a - \bar\mu_b},$
   choose the initialization radius as any
   $R^0 \geq \tilde\Theta\left(\max\left\{\frac{n m^2 \sqrt{d} \delta^2_{\max}}{\rho}, \sqrt{n} \bar{R}\right\}\right),$
   where $\bar{R} \eqdef \max_{b \in [m]} \norm{\bar \mu_b}$ and $\rho \in (0, 1].$ We sample the initial parameters as $\mu^0_{\ell} = R^0 \eta_{\ell} / \norm{\eta_{\ell}}$, where $\eta_{\ell} \sim \mathcal{N}(0, \mI)$ for all $\ell \in [n]$ (i.i.d.).
   Assume that $\delta_{\min} \geq \tilde\Omega\left(1\right)$ and $d \geq \tilde\Omega(n).$
   If $\gamma \leq \tilde{\Theta}\left(n\right)$ and GD \eqref{eq:wsltovjir} with GMMs \eqref{eq:p} is run for 
   $$\tilde{K} \eqdef \tilde\Theta\left(\frac{n^3}{\gamma \varepsilon}\right)$$
   iterations, then, with probability at least $1 - 4\rho$, we have that
      $\norm{\mu^{k}_{\ell} - \bar\mu_{a^*_{\ell}}}^2 \leq \varepsilon,$
   for all $\ell \in [n]$ and $k \geq \tilde{K}$, where $a^*_{\ell} = \arg\min_{a \in [m]} \norm{\mu^0_{\ell} - \bar{\mu}_a}$ is the index of the target mode closest to the initialization point $\mu^0_{\ell}$. In particular, $\tilde{K} = \tilde\Theta\left(\frac{n^2}{\varepsilon}\right)$ if $\gamma = \tilde{\Theta}\left(n\right).$
\end{restatable}
\end{mytheobox}
\subsection{Discussion}
This theorem states that, if the initial means $\{\mu^0_{\ell}\}_{\ell \in [n]}$ are sampled uniformly from the sphere of radius $R^0$, then, for every $\ell \in [n]$, the sequence $\{\mu^k_{\ell}\}_{k \geq 0}$ converges to a small-$\varepsilon$ neighborhood of the target mean $\bar{\mu}_{a^*_{\ell}}$, where $\bar{\mu}_{a^*_{\ell}}$ is the target mean closest to $\mu^0_{\ell}$. 
Notice that the initial means are sampled uniformly from the sphere of large radius $R^0$, and that the initialization algorithm does not depend on $\{\bar\mu_a\}_{a \in [m]}$.
Nevertheless, under the random initialization scheme, we obtain global convergence of the means to small neighborhoods of their closest target modes. The fact that $R^0$ is large is not a problem, since the dependence on $R^0$ is only logarithmic. Similarly to the single-target case, up to logarithmic factors, the convergence rate does not depend on $d$ when $d$ is large, and the dependence on $\varepsilon$ is $\tilde\cO\left(\nicefrac{1}{\varepsilon}\right).$

The theorem relies on two assumptions: $d \geq \tilde\Omega(n)$ and $\delta_{\min} \geq \tilde\Omega\left(1\right).$ The former is arguably mild; moreover, in Theorem~\ref{thm:mainmany} we only require $d \geq \tilde\Omega(1)$ and include $d \geq \tilde\Omega(n)$ here only for simplicity. The latter assumption $\delta_{\min} \geq \tilde\Omega\left(1\right)$ requires a logarithmic separation between the target means and is standard in the literature \citep{shah2023learning,zhou2025global}; analyzing GD using the reverse Fisher divergence under a near-zero separation is an important direction for future work.

\textit{Remark.} The full dependencies on all the parameters are presented in Theorem~\ref{thm:mainmany}. Notice the requirement on $\delta_{\min}$ with explicit logarithmic dependencies when $d \geq \tilde\Omega(n)$ is
   $\textstyle \delta_{\min} \geq \Omega\left(\sqrt{\log\left(n^{3} \max\left\{n, m\right\} (R^0)^4/ \varepsilon^2\right)} \sqrt{\log\left(2 d / \rho\right)}\right).$
The requirement depends on $\varepsilon$, which is expected since, intuitively, when $\mu^{k}_{\ell}$ iterates towards $\bar{\mu}_{a^*_{\ell}}$, all other target means $\{\bar{\mu}_{b}\}_{b \neq a^*_{\ell}}$ also “attract” $\mu^{k}_{\ell}$, and in reality it converges to some point not belonging to the set $\{\bar{\mu}_{b}\}_{b \in [m]}$. However, when $\delta_{\min} \geq \tilde\Omega(1)$, this point is extremely close to one of the targets. A similar observation about $\delta_{\min}$ and $\varepsilon$ appears in \citep{shah2023learning}, where the authors obtain convergence under only \emph{local initialization} close to the targets, whereas our approach allows \emph{global randomized initialization}.
\subsection{Proof technique}
\label{sec:proof_technique}
While the proof techniques of Theorem~\ref{thm:mainmany_simple} use ideas from Section~\ref{sec:single}, the proof as a whole is much more challenging and technical. Unlike Theorem~\ref{thm:mainmany_simple}, we divide the proof into two stages.

\textbf{Global convergence stage.} The fact that we initialize on the sphere with a large radius $R^0$ helps us significantly, since the initial means $\mu^0_i$ are far away from each other with high probability (w.h.p.). Since they are far apart, we can show that $\nabla_{\mu^k_\ell}\mathcal L_{\mathrm{rev}}(\mu^k) \approx \frac{2}{n} \left(\mu^k_{\ell} - \bar\mu_{a^*_\ell}\right).$ Thus, $\mu^{k+1}_{\ell} \approx \bar \mu_{a^*_\ell} + \left(1 - \nicefrac{2 \gamma}{n}\right)^{k+1} (\mu^0_{\ell} - \bar \mu_{a^*_\ell})$ during the first iterations, and each $\mu^k_{\ell}$ converges linearly towards the corresponding closest target mean. However, these approximations contain errors that require careful control: we derive auxiliary lemmas to control how these errors evolve and show that they do not introduce significant fluctuations, and that each $\mu^k_{\ell}$ indeed converges linearly towards its target.

\textbf{Local convergence stage.} After $\bar{K}$ iterations, each $\mu^{\bar{K}}_{\ell}$ is close to its target mean $\bar{\mu}_{a^*_\ell}.$ This is the final stage, where the gradients of $\mu^{\bar{K}}_{\ell}$ corresponding to the same group $I_a \eqdef \{i \in [n]\,:\, a^*_i = a\}$ start ``feeling'' each other, and we can no longer ignore the fact that the means from $I_a$ are close to each other. In this stage, for each group $I_a,$ we use almost the same analysis as in Section~\ref{sec:single}. However, the main difficulty here is to show that $\mu^{k}_{i}$ and $\mu^{k}_{j}$ corresponding to different groups, i.e., $i \in I_a,$ $j \in I_b,$ and $a \neq b,$ do not ``interfere with each other significantly,'' and that the means corresponding to $I_a$ converge towards $\bar\mu_a$ as if there were no other target means $\{\bar\mu_b\}_{b \neq a}$ and $\{\mu^{k}_i\}_{i \notin I_a},$ which also requires careful control of errors.

\subsection{Convergence of the total variation distance}
\label{sec:tv}
In general, Theorem~\ref{thm:mainmany_simple} guarantees the convergence of each $\mu^{k}_{\ell}$ only to a small neighborhood of the closest target mean. As an example, with $m = n = 2,$ it might be possible that $\mu^{k}_{1}$ and $\mu^{k}_{2}$ converge towards $\bar\mu_1,$ while none of the student's means converge near $\bar\mu_{2}.$ Nevertheless, intuitively, when $n$ is large, $\{\mu^{k}_{i}\}_{i \in [n]}$ should be distributed among $\{\bar\mu_{a}\}_{a \in [m]}$ according to how ``open'' each $\bar\mu_{a}$ is in space. Let us formalize this.
\begin{mytheobox}
\begin{restatable}{theorem}{THEOREMTV}
   \label{thm:thm_tv}
   Using the notation and assumptions of Theorem~\ref{thm:mainmany}, we additionally assume that $n \geq \frac{m^2}{2\varepsilon}\log\frac{2m}{\rho}.$ 
   Let $n_a$ be the number of student's components that converge to $\bar{\mu}_a,$ i.e., $n_a \eqdef \abs{\{\ell \in [n]\,:\, a = a^*_{\ell}\}}.$ 
   Then, with probability at least $1 - 5\rho,$ we have $\max_{a \in [m]} \abs{\frac{n_a}{n}-p_a} \leq \nicefrac{\sqrt{\varepsilon}}{m}$ and 
   \begin{align*}
      \textnormal{TV}(p_{\mu^k}, \bar{p}) &\eqdef \frac{1}{2} \int_{\R^d} \abs{p_{\mu^k}(x) - \bar{p}(x)} \, dx \leq \sqrt{\varepsilon} + \frac{1}{2} \sum_{a = 1}^m \abs{p_a - \frac1m} \quad \forall k \geq \tilde{K},
   \end{align*}
   where $\eta \sim \mathrm{Unif}(\mathbb S^{d-1})$ and $p_a \eqdef \Prob{\norm{R^0 \eta - \bar\mu_{a}} < \min_{b \in [m], b \neq a} \norm{R^0 \eta - \bar\mu_{b}}}$
   is the probability that $\bar\mu_a$ is the closest vector to $R^0 \eta$ among $\{\bar\mu_b\}_{b \in [m]}.$
\end{restatable}
\end{mytheobox}
The theorem extends Theorem~\ref{thm:mainmany} and shows that the total variation distance between the student's and teacher's distributions is less than or equal to $\sqrt{\varepsilon}$ up to $\frac{1}{2} \sum_{a = 1}^m \abs{p_a - \frac1m}.$ The latter depends on the probabilities $p_a,$ which characterize the probability that $\bar\mu_a$ is the closest target mean to a uniformly sampled random variable. Moreover, the student's components are approximately distributed between the target components according to probabilities $p_a$ since $\max_{a \in [m]} \abs{\nicefrac{n_a}{n}-p_a}$ is small.
\begin{mytheobox}
\begin{restatable}{corollary}{COROLTV}
Using the notation and assumptions of Theorem~\ref{thm:thm_tv}, if the vectors $\{\bar\mu_b\}_{b \in [m]}$ are symmetric, i.e., $p_a = \frac{1}{m}$ for all $a \in [m],$ then the total variation distance satisfies $\textnormal{TV}(p_{\mu^k}, \bar{p}) \leq \sqrt{\varepsilon}$ for all $k \geq \tilde{K}.$ In particular, if the vectors $\{\bar\mu_b\}_{b \in [m]}$ have the same norm and are pairwise orthogonal, then we have $p_a = \frac{1}{m}$ for all $a \in [m],$ and hence $\textnormal{TV}(p_{\mu^k}, \bar{p}) \leq \sqrt{\varepsilon}$ for all $k \geq \tilde{K}.$
\end{restatable}
\end{mytheobox}
Thus, if the target means are symmetric and $n \gg m,$ then it is possible to show the convergence of GD with the reverse Fisher divergence in terms of TV distance even when fitting one GMM to another, and the student's components are distributed between the target components almost evenly since $n_a \approx \nicefrac{n}{m}.$ A modified bound also holds when $p_a \neq \nicefrac{1}{m}.$ For instance, if $p_a \in [\nicefrac{1 - \delta}{m}, \nicefrac{1 + \delta}{m}],$ then $\textnormal{TV}(p_{\mu^k}, \bar{p}) \leq \sqrt{\varepsilon} + \delta / 2.$

\subsection{One-step convergence and non-isotropic generalization}
While proving Theorem~\ref{thm:mainmany}, we noticed an interesting convergence property of the problem. When $\mu^0_i$ is initialized on the sphere of radius $R^0,$ the method can converge ``very close'' to the closest target mean in one step. To achieve this, one should choose the step size $\gamma$ to be exactly equal to $\nicefrac{n}{2}.$ This is not surprising in view of our discussion in Section~\ref{sec:proof_technique}, where we explain that $\mu^{1}_{\ell} \approx \bar \mu_{a^*_\ell} + \left(1 - \nicefrac{2 \gamma}{n}\right) (\mu^0_{\ell} - \bar \mu_{a^*_\ell}).$ Exactly the same behavior is well known for GD and the optimization problem $\min_{x \in \R^d}\{f(x) \eqdef \|x\|^2\},$ where GD can also converge in one step because $f$ is \emph{isotropic}. We now consider a \emph{non-isotropic} score matching problem and extend our theoretical results to this setting.

Let $\mu_{1} \equiv [\mu_{1,1}; \dots; \mu_{1,n_{\min}}] \in \R^{n_{\min} \cdot d}$ and $\mu_{2} \equiv [\mu_{2,1}; \dots; \mu_{2,n_{\max}}] \in \R^{n_{\max} \cdot d}$ be two collections of parameters with $n_1 \equiv n_{\min} < n_2 \equiv n_{\max}.$ Define the
joint model on $\R^{2d}$ by $q(x,y)
\eqdef
p_{\mu_{1}}(x)p_{\mu_{2}}(y),$
and define the joint target distribution by
\(\bar q(x,y) \eqdef \bar p_1(x)\bar p_2(y)\), where \(\bar p_1\) and \(\bar p_2\) are arbitrary GMMs of the form \eqref{eq:p} with the same number \(m\) of target components $\bar\mu_{1,1}, \dots, \bar\mu_{1,m}$ and $\bar\mu_{2,1}, \dots, \bar\mu_{2,m},$ respectively.
We consider the reverse Fisher loss
$\mathcal L_{\mathrm{rev}}^{\mathrm{joint}}
(\mu_{1},\mu_{2})
\eqdef
\mathbb E_{(x,y)\sim q}
[
\|
\nabla_{(x,y)}\log q(x,y)
-
\nabla_{(x,y)}\log \bar q(x,y)
\|^2
].$
One can show that
\begin{align}
\label{eq:main_gener}
\mathcal L_{\mathrm{rev}}^{\mathrm{joint}}
(\mu_{1},\mu_{2})
=
\mathcal L^{(1)}_{\mathrm{rev}}(\mu_{1})
+
\mathcal L^{(2)}_{\mathrm{rev}}(\mu_{2}),
\end{align}
where $\mathcal L^{(1)}_{\mathrm{rev}}$ is equal to \eqref{eq:main} with $p_{\mu} \to p_{\mu_{1}}$ and $\bar p \to \bar p_1$ ($\mathcal L^{(2)}_{\mathrm{rev}}$ is similar). For this separable problem, we can directly apply Theorem~\ref{thm:mainmany_simple} and get the following result.
\begin{corollary}
   Under the assumptions and parameters of Theorem~\ref{thm:mainmany_simple}, we sample the initial parameters as $\mu^{0}_{j,\ell} = R^0 \eta_{j,\ell} / \norm{\eta_{j,\ell}}$, where $\eta_{j,\ell} \sim \mathcal{N}(0, \mI)$ for all $j \in \{1,2\}, \ell \in [n_j]$ (i.i.d.) and $R^0 = \tilde\Theta\left(\max\left\{\nicefrac{n_{\max} \sqrt{d} \delta^2_{\max}}{\rho}, \sqrt{n_{\max}} \bar{R}\right\}\right)$
   Assume that the smallest target gap $\min_{j\in\{1,2\}}\min_{a\neq b} \|\bar\mu_{j,a}-\bar\mu_{j,b}\| \geq \tilde\Omega\left(1\right)$ and $d \geq \tilde\Omega(n_{\max}).$
   If $\gamma \leq \tilde{\Theta}\left(n_{\min}\right)$ and GD is run for
   $\tilde{K} \eqdef \tilde\Theta\left(\nicefrac{n_{\max}^3}{\gamma \varepsilon}\right)$
   iterations, then, with probability at least $1 - 8\rho$, we have that
      $\|\mu^{k}_{j,\ell} - \bar\mu_{j, a^*_{j,\ell}}\|^2 \leq \varepsilon,$
   for all $j \in \{1, 2\}, \ell \in [n_j]$ and $k \geq \tilde{K}$, where $a^*_{j,\ell} = \arg\min_{a \in [m]} \|\mu^0_{j,\ell} - \bar{\mu}_{j,a}\|.$ In particular, $\tilde{K} = \tilde\Theta\left(\nicefrac{n_{\max}^3}{n_{\min} \varepsilon}\right)$ if $\gamma = \tilde{\Theta}\left(n_{\min}\right).$
\end{corollary}
In this more general case, the convergence rate depends on the ratio (condition number) between $n_{\max}$ and $n_{\min}.$ Moreover, in order to apply Theorem~\ref{thm:mainmany_simple}, we have to take $\gamma \leq \tilde\Theta(\min\{n_{\min}, n_{\max}\}) = \tilde\Theta(n_{\min}),$ and the one-step convergence phenomenon is unlikely to be possible for this problem. We can formalize this and prove a lower bound showing that we cannot take $\gamma \gg \tilde\Theta(n_{\min})$ and the method requires at least $\tilde\Omega(\nicefrac{n_{\max}}{n_{\min}})$ steps.
\begin{restatable}[Lower bound; Simplified version of Theorem~\ref{thm:lower_bound}]{theorem}{THMLOWERBOUNDSIMPL}
   \label{thm:lower_bound_simpl}
   We run GD on the problem \eqref{eq:main_gener} with $m = 1,$ $\bar\mu_{1,1} = \bar\mu_{2,1},$ $d \geq \tilde\Omega(n_{\max}),$ and $n_{\min} < \nicefrac{n_{\max}}{32}.$ Assume that $\varepsilon \in (0, 1],$ $\rho \in (0, 1],$ $R^0 = \tilde\Theta(\max\{1, \sqrt{n_{\max}} \bar{R}\}),$ where $\bar{R} \eqdef \norm{\bar \mu_{1,1}},$ and we sample the initial parameters as $\mu^0_{i,\ell} = R^0 \eta_{i,\ell} / \norm{\eta_{i,\ell}}$, where $\eta_{i,\ell} \sim \mathcal{N}(0, \mI)$ for all $i \in \{1,2\}, \ell \in [n_i]$ (i.i.d.). Under these assumptions, if the step size $\gamma \leq 2 n_{\min},$ then, with probability at least $1 - \rho,$ for all $k < \tilde{K}_1 \eqdef \tilde{\Theta}\left(\nicefrac{n_{\max}}{n_{\min}}\right),$ there exists $i \in \{1,2\}$ and $\ell \in [n_i]$ such that 
   $\|\mu^{k}_{i,\ell} - \bar\mu_{i,1}\|^2 > \varepsilon.$ If instead $\gamma > 2 n_{\min},$ then, with probability at least $1 - \rho,$ the method diverges.
\end{restatable}

\section{Experiments}
\label{sec:exp}
We compare the forward and reverse Fisher divergences numerically for the cases of a single target component, $m = 1,$ and multiple target components, $m > 1.$ In all experiments, we run GD with a step size of $0.01,$ use Gauss–Hermite quadrature to evaluate expectations numerically, and compare the convergence trajectories of GD under the different losses. GD starts from the same initial means for both losses. In Figure~\ref{fig:plottwo}, we visualize the convergence trajectories, which agree with the theory: the student parameters under the forward Fisher divergence may converge in an arbitrary direction, while the reverse Fisher divergence guarantees convergence of all parameters to a small neighborhood of one of the target components.
\begin{figure}[t]
    \centering
    \includegraphics[width=0.75\linewidth]{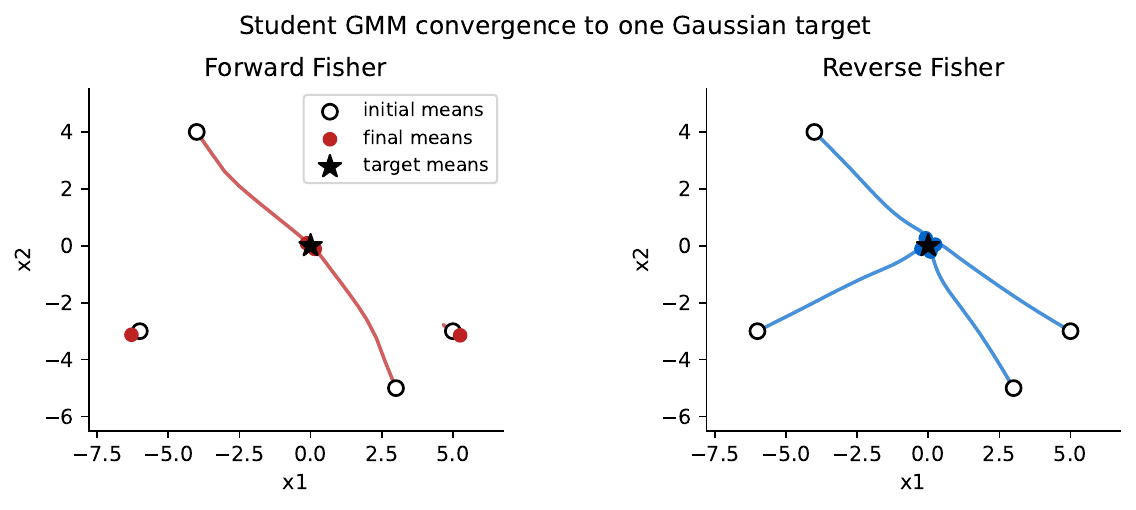}
    \includegraphics[width=0.75\linewidth]{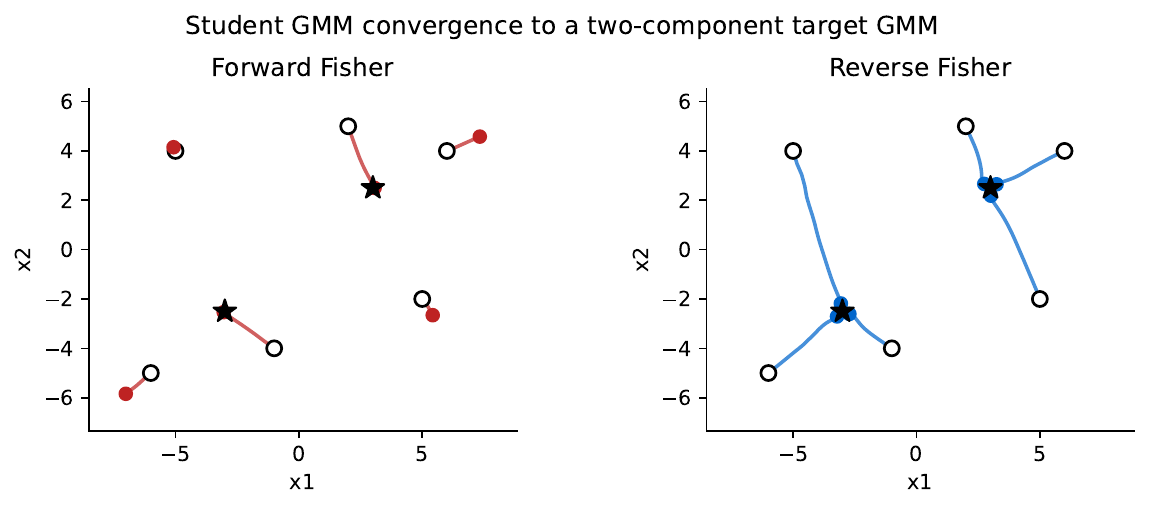}
    \caption{Convergence trajectories of GD with the forward and reverse Fisher divergences in two dimensions. First row: a single target component ($m = 1$). Second row: two target components ($m = 2$).}
    \label{fig:plottwo}
\end{figure}

\subsection{Computing gradients in practice}
A natural question concerns the computational complexity and feasibility of computing the gradient of the reverse Fisher divergence. Notice that this loss depends on the parameters through the expectation. For GMMs, computing the gradient is as feasible as in the forward loss case, where almost all formulas can be computed analytically, except for the outer expectation, which can be approximated, for instance, using sampling or numerical integration techniques. For more general distributions, depending on the domain, there may be different ways to find or approximate the gradient of \eqref{eq:main}. One generic way is to use a zeroth-order approximation. The idea is to approximate $\nabla \mathcal L_{\mathrm{rev}}(\mu)$ by $\nicefrac{\textnormal{dim}(\mu)}{B} \sum_{j=1}^{B} (\mathcal L_{\mathrm{rev}}(\mu + \alpha u_j) - \mathcal L_{\mathrm{rev}}(\mu - \alpha u_j)) (u_j / 2 \alpha),$ where $B \geq 1, \alpha > 0$ and $u_j$ are i.i.d. uniformly random directions from the sphere of radius $1.$ The function values can be approximated using Monte Carlo sampling, where samples from $p_{\mu + \alpha u_j}$ and $p_{\mu - \alpha u_j}$ can be approximately drawn using Langevin-type samplers. Developing effective algorithms and adapting current approaches to reverse-divergence methods is an important present and future research direction, and development has already begun in many related topics \citep{yang2019variational,yusemi,caibatch,cai2024eigenvi}.

\section{Conclusion}
The comparison of the forward and reverse Fisher divergences is not new and has been carried out in other contexts (e.g., \citep{margossian2025variational,lai2026unified}). Our works analyzes the reverse Fisher divergence and provides strong theoretical evidence that \emph{it has a much more favorable optimization landscape and better convergence properties formalized by our theory.} Even for the score matching problem for GMMs, we believe that this result is interesting and may motivate new research directions for designing algorithms and developing approaches based on the reverse Fisher divergence. In future work, it would be interesting to extend the result to more general families of distributions and to consider more realistic settings with stochastic gradients instead of exact gradients.

\section*{Acknowledgments}
The main observations, ideas, proofs, and research directions were proposed by the authors. However, during the research phase, we used ChatGPT Pro to accelerate technical calculations and quickly check different research directions before obtaining the final result. ChatGPT Pro helped us significantly in proving one of the key lemmas, Lemma~\ref{lemma:z_b}, by suggesting an interesting way to bound the trace of a square matrix and derive the lower bound \eqref{eq:zpfbtoigdmEHSImuC} for the dot product. All results in this paper were manually checked and written in LaTeX by the authors.

\bibliography{iclr2026_conference}
\bibliographystyle{iclr2026_conference}

\appendix

\newpage

\tableofcontents

\newpage

\section{Table of Notations}

\begin{table}[h]
\centering
\begin{tabular}{ll}
\toprule
Notation & Description \\
\midrule
$d$ & Dimension \\
$[A_1,\ldots,A_n]$ & Horizontal concatenation of matrices or vectors \\
$[A_1;\ldots;A_n]$ & Vertical concatenation of matrices or vectors \\
$\mathrm{dist}(\mathcal{A},\mathcal{B})$ & Euclidean distance between sets $\mathcal{A}$ and $\mathcal{B}$, i.e.,
$\mathrm{dist}(\mathcal{A},\mathcal{B}) = \inf_{x\in\mathcal{A},\,y\in\mathcal{B}} \|x-y\|$ \\
$[x,y]$ & Closed line segment joining $x$ and $y$, i.e., $\{(1-\lambda)x+\lambda y:\lambda\in[0,1]\}$ \\
$\bar{\mu}_a$ & Target mode $a$ \\
$\bar{R}$ & Maximum norm of a target mode, $\bar{R} = \max_{a \in [m]} \|\bar{\mu}_a\|$ \\
$\delta_{\min}$ & Minimum distance between distinct target modes \\
$\delta_{\max}$ & Maximum distance between distinct target modes \\
$\mu_\ell^k$ & Parameter $\ell$ at iteration $k$ \\
$\mu_\ell^0$ & Initial parameter $\ell$ \\
$a_\ell^*$ & Index of the target mode closest to $\mu_\ell^0$ \\
$R^0$ & Initialization radius \\
$\gamma$ & Gradient descent step size \\
$\rho$ & Failure probability parameter \\
$\varepsilon$ & Target accuracy \\
$\bar{p}$ & Teacher probability density function \\
$p_{\mu}$ & Student probability density function parameterized by $\mu$ \\
$\varphi(x)$ & Standard Gaussian density, $\varphi(x) = (2\pi)^{-d/2}\exp(-\|x\|^2/2)$ \\
$\mathbb E_{x \sim p}$ & Expectation with respect to the distribution $p$ \\
$\Gamma(t)$ & Gamma function, $\Gamma(t)=\int_0^\infty x^{t-1}e^{-x}\,dx$ \\
$\mI$ & Identity matrix \\
$\|\cdot\|$ & Euclidean norm of a vector \\
$\|\cdot\|_F$ & Frobenius norm of a matrix \\
$\|\cdot\|_{\mathrm{op}}$ & Operator (spectral) norm of a matrix \\
$\mathrm{Tr}(A)$ & Trace of a square matrix $A$ \\
$\langle A,B\rangle$ & Frobenius inner product, $\mathrm{Tr}(A^\top B)$ \\
$\mathrm{dim}(\mu)$ & Dimension of the parameter space containing $\mu$ \\
$g = \cO(f)$ & There exists $C > 0$ such that $g(z) \le C \, f(z)$ for all $z \in \cZ$. \\
$g = \Omega(f)$ & There exists $C > 0$ such that $g(z) \ge C \, f(z)$ for all $z \in \cZ$. \\
$g = \Theta(f)$ & There exist $C_1, C_2 > 0$ such that $C_1 f(z) \le g(z) \le C_2 f(z)$ for all $z \in \cZ$. \\
$\tilde{\cO},\tilde{\Omega},$ and $\tilde{\Theta}$ & The same as $\cO$, $\Omega,$ and $\Theta,$ but up to logarithmic factors. \\
\bottomrule
\end{tabular}
\caption{Frequently used notation.}
\end{table}

\newpage
\section{Generic Lemmas}
In this section, we state several useful generic lemmas that will be used to prove the convergence of GD.

\begin{lemma}
   \label{lemma:gamma}
   Let the Gamma function be defined by
   \[
      \Gamma(t) \eqdef \int_0^\infty x^{t-1} e^{-x} \, dx,
      \qquad t > 0.
   \]
   Then, for any integer $d \geq 2,$
   \[
      \Gamma\left(\frac{d}{2}\right)
      \leq
      \sqrt{\frac{d - 1}{2}}\,
      \Gamma\left(\frac{d - 1}{2}\right).
   \]
\end{lemma}
\begin{proof}
Since $\Gamma$ is log-convex, $\log\left(\Gamma\left(\frac{x}{2} + \frac{y}{2}\right)\right) \leq \frac{1}{2}\log\left(\Gamma(x)\right) + \frac{1}{2} \log\left(\Gamma(y)\right)$ for all $x, y \geq 0.$ Taking $x = \frac{d - 1}{2}$ and $y = \frac{d + 1}{2},$ we get $\log\left(\Gamma\left(\frac{d}{2}\right)\right) \leq \frac{1}{2}\log\left(\Gamma\left(\frac{d - 1}{2}\right)\right) + \frac{1}{2} \log\left(\Gamma\left(\frac{d + 1}{2}\right)\right).$ Due to $\Gamma\left(\frac{d + 1}{2}\right) = \frac{d - 1}{2}\Gamma\left(\frac{d - 1}{2}\right),$ we obtain $\Gamma\left(\frac{d}{2}\right) \leq \sqrt{\frac{d - 1}{2}} \Gamma\left(\frac{d - 1}{2}\right).$
\end{proof}

\begin{lemma}
\label{lemma:close}
Let $\rho \in (0,1]$ and $\mu^0_{\ell} = R^0 \eta_{\ell} / \norm{\eta_{\ell}}$, where $R^0 > 0$ and
$\eta_{\ell} \sim \mathcal{N}(0,\mI_d)$ for all $\ell \in [n]$, independently. The vectors $\{\bar\mu_a\}_{a \in [m]}$ are deterministic.
Assume that $d > 3$ and define
\[
\delta_{\min} \eqdef \min_{a \neq b}\norm{\bar\mu_a - \bar\mu_b} > 0.
\]
Let
\[
a^*_i \eqdef \arg\min_{a \in [m]} \norm{\mu^0_i - \bar\mu_a}
\]
with ties broken arbitrarily, and let
\[
\Omega_1 \eqdef
\left\{
\forall i \in [n]:
\norm{\mu^0_i - \bar\mu_{a^*_i}}^2
\leq
\min_{b \neq a^*_i}
\norm{\mu^0_i - \bar\mu_b}^2
-
\Delta
\right\}.
\]
Then choosing
\[
\Delta
=
\frac{2 R^0 \delta_{\min} \rho}{n m^2 \sqrt{d}}
\]
guarantees that
\[
\Prob{\Omega_1} \geq 1-\rho .
\]
\end{lemma}

\begin{proof}
For fixed $i \in [n]$, define
$g_i
\eqdef
\min_{b \neq a^*_i}
\norm{\mu^0_i - \bar\mu_b}^2
-
\norm{\mu^0_i - \bar\mu_{a^*_i}}^2 .$
Then $\Omega_1 = \{g_i \geq \Delta \ \forall i \in [n]\}$, and hence
\[
\Prob{\Omega_1^c}
\leq
\sum_{i=1}^n \Prob{g_i < \Delta}
=
n \Prob{g_1 < \Delta} = n \Prob{\norm{\mu^0_1 - \bar\mu_{b^*_1}}^2 - \norm{\mu^0_1 - \bar\mu_{a^*_1}}^2 < \Delta},
\]
where $\Omega_1^c$ is the complementary event of $\Omega_1$ and $b^*_1 \eqdef \arg\min\limits_{a \in [m], a \neq a^*_1} \norm{\mu^0_1 - \bar\mu_a}.$ Thus
\begin{align}
\label{eq:sezYUFHvRIgXyZV}
\Prob{\Omega_1^c}
\leq n \sum_{a,b \in [m], a \neq b} \Prob{\abs{\norm{\mu^0_1 - \bar\mu_{b}}^2 - \norm{\mu^0_1 - \bar\mu_{a}}^2} < \Delta},
\end{align}
Let $u = \eta_1/\norm{\eta_1}$, so that $\mu^0_1 = R^0 u$ and $u$ is uniform on the unit sphere. For fixed $a \neq b$,
\[
\norm{R^0 u - \bar\mu_b}^2
-
\norm{R^0 u - \bar\mu_a}^2
=
\norm{\bar\mu_b}^2
-
\norm{\bar\mu_a}^2
-
2 R^0 \inp{u}{\bar\mu_b - \bar\mu_a}.
\]
Thus,
\begin{align*}
&\Prob{
\abs{\norm{\mu^0_1 - \bar\mu_b}^2
-
\norm{\mu^0_1 - \bar\mu_a}^2}
<
\Delta
} \\
&=
\Prob{\abs{\inp{u}{\frac{\bar\mu_b - \bar\mu_a}{\norm{\bar\mu_b - \bar\mu_a}}} - \frac{\norm{\bar\mu_b}^2 - \norm{\bar\mu_a}^2}{2 R^0 \norm{\bar\mu_b - \bar\mu_a}}} < \frac{\Delta}{2 R^0 \norm{\bar\mu_b - \bar\mu_a}}}, \\
&=
\Prob{\abs{\inp{u}{\frac{\bar\mu_b - \bar\mu_a}{\norm{\bar\mu_b - \bar\mu_a}}} -\bar{t}} < \bar{b}},
\end{align*}
where $\bar{t} \eqdef \frac{\norm{\bar\mu_b}^2 - \norm{\bar\mu_a}^2}{2 R^0 \norm{\bar\mu_b - \bar\mu_a}}$ and $\bar{b} \eqdef \frac{\Delta}{2 R^0 \norm{\bar\mu_b - \bar\mu_a}}.$ It is well-known (e.g. \citep{spruill2007asymptotic}) that $Z \eqdef \inp{u}{\frac{\bar\mu_b - \bar\mu_a}{\norm{\bar\mu_b - \bar\mu_a}}}$ has the density 
\begin{align}
   \label{eq:BLbUPxXmJZHAm}
   p(x) = \frac{\Gamma\left(\frac{d}{2}\right)}{\sqrt{\pi} \Gamma\left(\frac{d - 1}{2}\right)} (1 - x^2)^{(d - 3) / 2}
\end{align}
for all $x \in [-1, 1],$ and $p(x) = 0$ otherwise, where $\Gamma$ is the gamma function. Since $d>3$, the density $p$ is symmetric and uniquely maximized at $0$. Hence
\begin{align*}
&\Prob{
\abs{\norm{\mu^0_1 - \bar\mu_b}^2
-
\norm{\mu^0_1 - \bar\mu_a}^2}
<
\Delta
} \\
&=
\Prob{\abs{Z-\bar t}<\bar b} \leq
\Prob{\abs{Z}<\bar b} 
\leq
2\bar b \, p(0) 
=
2\bar b
\frac{\Gamma\left(\frac{d}{2}\right)}
{\sqrt{\pi}\Gamma\left(\frac{d-1}{2}\right)}.
\end{align*}
Using Lemma~\ref{lemma:gamma},
\begin{align*}
   \Prob{
\abs{\norm{\mu^0_1 - \bar\mu_b}^2
-
\norm{\mu^0_1 - \bar\mu_a}^2}
<
\Delta
} \leq \frac{1}{\sqrt{\pi}} \sqrt{\frac{d - 1}{2}} \frac{\Delta}{R^0 \norm{\bar\mu_b - \bar\mu_a}} \leq \frac{\sqrt{d} \Delta}{2 R^0 \norm{\bar\mu_b - \bar\mu_a}}.
\end{align*}
Substituting into \eqref{eq:sezYUFHvRIgXyZV} and using the definition of $\delta_{\min}$,
\begin{align*}
   \Prob{\Omega_1^c} \leq \frac{n m^2 \sqrt{d} \Delta}{2 R^0 \delta_{\min}}.
\end{align*}
It remains to use our choice of $\Delta.$
\end{proof}

\begin{lemma}
\label{lemma:far}
Let $\rho \in (0,1]$ and $\mu^0_{\ell} = R^0 \eta_{\ell} / \norm{\eta_{\ell}}$, where $R^0 > 0$ and
$\eta_{\ell} \sim \mathcal{N}(0,\mI_d)$ for all $\ell \in [n]$, independently.
Assume $d \geq 3$. Let
\[
\Omega_2 \eqdef
\left\{
\forall i \neq j \in [n]:
\norm{\mu^0_i - \mu^0_j}^2
\geq \Delta_2
\right\}.
\]
Then choosing
\[
\Delta_2
=
(R^0)^2
\left(
\frac{\sqrt{2\pi}\rho}{\sqrt d\, n^2}
\right)^{\frac{2}{d-1}}
\]
guarantees that
\[
\Prob{\Omega_2} \geq 1-\rho .
\]
\end{lemma}

\begin{proof}
Let $u_i \eqdef \eta_i/\norm{\eta_i}$. Then $u_1,\dots,u_n$ are independent
and uniformly distributed on the unit sphere $\mathbb S^{d-1}$, and
$\mu_i^0=R^0u_i$. For any fixed pair $i\neq j$,
\[
\norm{\mu_i^0-\mu_j^0}^2
=
(R^0)^2\norm{u_i-u_j}^2.
\]
Therefore,
\begin{align*}
   \Prob{\norm{\mu_i^0-\mu_j^0}^2 \leq \Delta_2}
   &= \Prob{- 2 (R^0)^2 \inp{u_i}{u_j} \leq \Delta_2 - 2 (R^0)^2} \\
   &= \Prob{\inp{u_i}{u_j} \geq t},
\end{align*}
where $t \eqdef 1 - \frac{\Delta_2}{2 (R^0)^2}.$
Conditioning on $u_i$, $Z \eqdef \inp{u_i}{u_j}$ has the density \eqref{eq:BLbUPxXmJZHAm}. Thus,
\begin{align*}
   \Prob{\norm{\mu_i^0-\mu_j^0}^2 \leq \Delta_2}
   &\leq
   \frac{\Gamma\left(\frac{d}{2}\right)}
   {\sqrt{\pi} \Gamma\left(\frac{d - 1}{2}\right)}
   \int_{t}^{1} (1 - x^2)^{(d - 3) / 2} \, dx \\
   &\leq
   \frac{\Gamma\left(\frac{d}{2}\right) (1 - t)}
   {\sqrt{\pi} \Gamma\left(\frac{d - 1}{2}\right)}
   (1 - t^2)^{(d - 3) / 2} \\
   &\leq
   \frac{\Gamma\left(\frac{d}{2}\right)}
   {\sqrt{\pi} \Gamma\left(\frac{d - 1}{2}\right)}
   (1 - t^2)^{(d - 1) / 2}.
\end{align*}
Due to Lemma~\ref{lemma:gamma},
\begin{align*}
   \Prob{\norm{\mu_i^0-\mu_j^0}^2 \leq \Delta_2}
   \leq
   \frac{\sqrt{d}}{\sqrt{2 \pi}}
   (1 - t^2)^{(d - 1) / 2}.
\end{align*}
Moreover,
\[
1-t^2
=
1-\left(1-\frac{\Delta_2}{2(R^0)^2}\right)^2
=
\frac{\Delta_2}{(R^0)^2}
-
\frac{\Delta_2^2}{4(R^0)^4}
\leq
\frac{\Delta_2}{(R^0)^2}.
\]
Therefore, by the union bound,
\[
\Prob{\Omega_2^c}
\leq
\sum_{i<j}
\Prob{\norm{\mu_i^0-\mu_j^0}^2 \leq \Delta_2}
\leq
\frac{\sqrt{d} n^2}{\sqrt{2 \pi}}
\left(
\frac{\Delta_2}{(R^0)^2}
\right)^{(d - 1) / 2}.
\]
Choosing
\[
\Delta_2
=
(R^0)^2
\left(
\frac{\sqrt{2\pi}\rho}{\sqrt d\, n^2}
\right)^{\frac{2}{d-1}}
\]
gives $\Prob{\Omega_2^c} \leq \rho$, and hence
$\Prob{\Omega_2} \geq 1-\rho$.
\end{proof}

\begin{lemma}
\label{lemma:dist}
Let $\rho \in (0,1],$ $d \geq 3,$ $n \geq 2,$ and $m \geq 2.$ Let $\mu^0_{\ell} = R^0 \eta_{\ell} / \norm{\eta_{\ell}}$, where $R^0 > 0$ and
$\eta_{\ell} \sim \mathcal{N}(0,\mI_d)$ for all $\ell \in [n]$, independently. Moreover, $\{\bar\mu_a\}_{a \in [m]}$ are deterministic vectors such that $\delta_{\min} \eqdef \min_{a \neq b}\norm{\bar\mu_a - \bar\mu_b} > 0$ and $\norm{\bar\mu_a} \leq R^0$ for all $a \in [m].$ Let
\[
\Omega_3 \eqdef \left\{ \forall i \neq j \in [n] \textnormal{ s.t. } a^*_i \neq a^*_j: \textnormal{dist}([\mu^0_i, \bar{\mu}_{a^*_i}], [\mu^0_j, \bar{\mu}_{a^*_j}]) \geq \Delta_3 \right\},
\]
where $a^*_i = \arg\min_{a \in [m]} \norm{\mu^0_i - \bar \mu_a}$ for all $i \in [n].$
Then choosing
\[
\Delta_3
=
\frac{\delta_{\min}}{8}
\left(\frac{\rho}{m n^2}\right)^{\frac{1}{d-2}}
\]
guarantees that
\[
\Prob{\Omega_3} \geq 1-\rho .
\]
\end{lemma}

\begin{proof}
   Clearly,
   \begin{align*}
      \Prob{\Omega_3^c} 
      &\leq \sum_{i \neq j \in [n]} \Prob{\left\{a^*_i \neq a^*_j\right\} \cap \left\{\textnormal{dist}([\mu^0_i, \bar{\mu}_{a^*_i}], [\mu^0_j, \bar{\mu}_{a^*_j}]) < \Delta_3\right\}}.
   \end{align*}
   Fix $i \neq j$, and assume that $a^*_i \neq a^*_j$. Next, $Q \eqdef \left\{x \in \R^d \,:\, \norm{x - \bar{\mu}_{a^*_i}} \leq \norm{x - \bar{\mu}_{a^*_j}}\right\}.$ This set is convex; thus, $[\mu^0_i, \bar{\mu}_{a^*_i}] \subseteq Q,$ and 
   \begin{align}
      \label{eq:bpHaLdpUoTJasZGKA}
      \frac{\delta_{\min}}{2} \leq \frac{1}{2}\norm{\bar{\mu}_{a^*_i} - \bar{\mu}_{a^*_j}} \leq \frac{1}{2}\norm{a - \bar{\mu}_{a^*_i}} + \frac{1}{2}\norm{a - \bar{\mu}_{a^*_j}} \leq \norm{a - \bar{\mu}_{a^*_j}}
   \end{align}
   for all $a \in [\mu^0_i, \bar{\mu}_{a^*_i}]$. Clearly,
   \begin{align*}
      &\Prob{\left\{a^*_i \neq a^*_j\right\} \cap \textnormal{dist}([\mu^0_i, \bar{\mu}_{a^*_i}], [\mu^0_j, \bar{\mu}_{a^*_j}]) < \Delta_3} \\
      &\leq\underbrace{\Prob{\left\{a^*_i \neq a^*_j\right\} \cap \exists b \in [\mu^0_j, \bar{\mu}_{a^*_j}] \textnormal{ such that } \textnormal{dist}([\mu^0_i, \bar{\mu}_{a^*_i}], b) < \Delta_3 \textnormal{ and } \norm{b - \bar{\mu}_{a^*_j}} < \frac{\delta_{\min}}{4}}}_{P_1} \\
      &+\quad \underbrace{\Prob{\left\{a^*_i \neq a^*_j\right\} \cap \exists b \in [\mu^0_j, \bar{\mu}_{a^*_j}] \textnormal{ such that } \textnormal{dist}([\mu^0_i, \bar{\mu}_{a^*_i}], b) < \Delta_3 \textnormal{ and } \norm{b - \bar{\mu}_{a^*_j}} \geq \frac{\delta_{\min}}{4}}}_{P_2}.
   \end{align*}
   Consider any point $b \in [\mu^0_j, \bar{\mu}_{a^*_j}].$ If $\norm{b - \bar{\mu}_{a^*_j}} < \frac{\delta_{\min}}{4},$ then 
   \begin{align*}
      \norm{a - b} 
      &= \norm{(a - \bar{\mu}_{a^*_j}) - (b - \bar{\mu}_{a^*_j})} \geq \norm{a - \bar{\mu}_{a^*_j}} - \norm{b - \bar{\mu}_{a^*_j}} \geq \frac{\delta_{\min}}{4} > \Delta_3
   \end{align*}
   for all $a \in [\mu^0_i, \bar{\mu}_{a^*_i}]$ due to \eqref{eq:bpHaLdpUoTJasZGKA}. Thus,
   \begin{align*}
      P_1 = 0.
   \end{align*}
   Otherwise, consider the case $\norm{b - \bar{\mu}_{a^*_j}} \geq \frac{\delta_{\min}}{4}.$ Let $\mathcal A = \{x \in \R^d\,:\,x = \bar{\mu}_{a^*_j} + \alpha (\mu^0_i - \bar{\mu}_{a^*_j}) + \beta (\bar{\mu}_{a^*_i} - \bar{\mu}_{a^*_j}), \alpha, \beta \in \R\}$ be the affine space.
   Notice that if $\norm{b - \bar{\mu}_{a^*_j}} \geq \frac{\delta_{\min}}{4}$ and $b \in [\mu^0_j, \bar{\mu}_{a^*_j}],$ then
   \begin{align*}
      b = \bar{\mu}_{a^*_j} + t (\mu^0_j - \bar{\mu}_{a^*_j})
   \end{align*}
   for some $t \geq \frac{\delta_{\min}}{4 \norm{\mu^0_j - \bar{\mu}_{a^*_j}}} \geq \frac{\delta_{\min}}{8 R^0}.$ We get
   \begin{align*}
      P_2 
      &\leq \Prob{\cup_{t \in [\nicefrac{\delta_{\min}}{8 R^0}, 1]}\textnormal{dist}(\mathcal A, \bar{\mu}_{a^*_j} + t (\mu^0_j - \bar{\mu}_{a^*_j})) < \Delta_3} \\
      &= \Prob{\cup_{t \in [\nicefrac{\delta_{\min}}{8 R^0}, 1]} \textnormal{dist}(\mathcal B_{a^*_j}, t (\mu^0_j - \bar{\mu}_{a^*_j})) < \Delta_3} \\
      &= \Prob{\cup_{t \in [\nicefrac{\delta_{\min}}{8 R^0}, 1]} \textnormal{dist}(\mathcal B_{a^*_j}, \mu^0_j - \bar{\mu}_{a^*_j}) < \frac{\Delta_3}{t}} \\
      &\leq \Prob{\textnormal{dist}(\mathcal B_{a^*_j}, \mu^0_j - \bar{\mu}_{a^*_j}) < \frac{8 \Delta_3 R^0}{\delta_{\min}}},
   \end{align*}
   where $\mathcal B_{a^*_j} \eqdef  \{x \in \R^d\,:\,x = \alpha (\mu^0_i - \bar{\mu}_{a^*_j}) + \beta (\bar{\mu}_{a^*_i} - \bar{\mu}_{a^*_j}), \alpha, \beta \in \R\}$ is the linear space. Since $a^*_j$ is random, we further bound
   \begin{align*}
      P_2 \leq \sum_{a \in [m]} \Prob{\textnormal{dist}(\mathcal B_{a}, \mu^0_j - \bar{\mu}_{a}) < \frac{8 \Delta_3 R^0}{\delta_{\min}}},
   \end{align*}
   where
   \[
      \mathcal B_a \eqdef  \{x \in \R^d\,:\,x = \alpha (\mu^0_i - \bar{\mu}_{a}) + \beta (\bar{\mu}_{a^*_i} - \bar{\mu}_{a}), \alpha, \beta \in \R\}.
   \]
   The space $\mathcal B_a$ has dimension at most $2$ and is independent of
   $\mu_j^0$. Let $\mathcal L_a$ be a linear subspace of dimension exactly $2$ such that
   $\mathcal B_a \subseteq \mathcal L_a$. Then, for all $r>0$,
   \begin{align*}
      \Prob{\textnormal{dist}(\mathcal B_a,\mu_j^0-\bar\mu_a)<r}
      \leq
      \Prob{\textnormal{dist}(\mathcal L_a,\mu_j^0-\bar\mu_a)<r}.
   \end{align*}
   Therefore, for $r > 0,$
   \begin{align*}
      \Prob{\textnormal{dist}(\mathcal B_{a}, \mu^0_j - \bar{\mu}_{a}) < r} \leq \Prob{\norm{\mP \mu^0_j - \mP \bar{\mu}_{a}} < r},
   \end{align*}
   where $\mP$ is the orthogonal projection onto $\mathcal L_a^\perp.$ Since $\mP = \mU \mL \mU^\top,$ where $\mU$ is an orthogonal matrix and $\mL$ is a diagonal matrix with diagonal entries $[1, \dots, 1, 0, 0] \in \R^d,$
   \begin{align*}
      \Prob{\textnormal{dist}(\mathcal B_{a}, \mu^0_j - \bar{\mu}_{a}) < r} \leq \Prob{\norm{\mL \mU^\top \mu^0_j - \mL \mU^\top\bar{\mu}_{a}} < r},
   \end{align*}
   It is well known that the first $d - 2$ coordinates of $\mL \mU^\top \mu^0_j$ are uniformly distributed in the ball of radius $R^0$ (e.g., \citep[formula~(6), with $k=s-2$]{khokhlov}). Thus,
   \begin{align*}
      \Prob{\textnormal{dist}(\mathcal B_{a}, \mu^0_j - \bar{\mu}_{a}) < r} \leq \frac{\textnormal{Vol}_{d-2}(B(0, R^0) \bigcap B(v, r))}{\textnormal{Vol}_{d-2}(B(0, R^0))} \leq \frac{\textnormal{Vol}_{d-2}(B(0, r))}{\textnormal{Vol}_{d-2}(B(0, R^0))} = \left(\frac{r}{R^0}\right)^{d - 2}.
   \end{align*}
   Here, $B(a, r) \eqdef \{x \in \R^{d-2}\,:\, \norm{x - a} \leq r\},$ $\textnormal{Vol}_{d-2}(\cdot)$ is the volume of a set, and $v$ is the vector of the first $d - 2$ coordinates of $\mL \mU^\top\bar{\mu}_{a}.$
   Applying this with $r = 8 \Delta_3 R^0 / \delta_{\min}$, we get
   \begin{align*}
      P_2 \leq m \left(\frac{8 \Delta_3}{\delta_{\min}}\right)^{d - 2}.
   \end{align*}
   Therefore,
   \begin{align*}
      \Prob{\Omega_3^c}
      \leq
      m n^2
      \left(\frac{8 \Delta_3}{\delta_{\min}}\right)^{d - 2}
      \leq \rho
   \end{align*}
   by the choice of $\Delta_3$. Hence,
   \[
      \Prob{\Omega_3} \geq 1-\rho .
   \]
\end{proof}

\begin{lemma}
\label{lemma:norm}
Let $\rho \in (0,1]$ and $\mu^0_{\ell} = R^0 \eta_{\ell} / \norm{\eta_{\ell}}$, where $R^0 > 0$ and
$\eta_{\ell} \sim \mathcal{N}(0,\mI_d)$ for all $\ell \in [n]$, independently.
Let
\[
\Omega_4^{\textnormal{full}} \eqdef
\left\{\norm{\begin{bmatrix}
      \left(\mu^0_{1}\right)^\top \\
      \vdots \\
      \left(\mu^0_{n}\right)^\top \\
   \end{bmatrix}}_{\textnormal{op}} \leq \Delta_4\right\}.
\]
Then choosing
\[
\Delta_4
= R^0 \sqrt{\frac{n}{d} + \max\left\{\frac{4}{3}\log\left(\frac{2 d}{\rho}\right), \sqrt{\frac{4 n}{d} \log\left(\frac{2 d}{\rho}\right)}\right\}}
\]
guarantees that
\[
\Prob{\Omega_4^{\textnormal{full}}} \geq 1-\rho .
\]
\end{lemma}

\begin{proof}
   Without loss of generality, we assume that $R^0 = 1.$ It is sufficient to rescale the final result. For all $v \in \R^d,$
   \begin{align*}
      \norm{\begin{bmatrix}
      \left(\mu^0_{1}\right)^\top \\
      \vdots \\
      \left(\mu^0_{n}\right)^\top \\
   \end{bmatrix} v}^2 =
   \norm{\begin{bmatrix}
      \left(\mu^0_{1}\right)^\top v \\
      \vdots \\
      \left(\mu^0_{n}\right)^\top v \\
   \end{bmatrix}}^2 = \sum_{i=1}^{n} \left(\left(\mu^0_{i}\right)^\top v\right)^2 = v^\top \left(\sum_{i=1}^{n} \mu^0_{i} (\mu^0_{i})^\top\right) v
   \end{align*}
   Let $\mA_i \eqdef \mu^0_{i} (\mu^0_{i})^\top$ for all $i \in [n].$ Thus,
   \begin{align}
      \label{eq:CvnhQLXebnWEBxrRQq}
      \norm{\begin{bmatrix}
      \left(\mu^0_{1}\right)^\top \\
      \vdots \\
      \left(\mu^0_{n}\right)^\top \\
   \end{bmatrix}}_{\textnormal{op}}^2 = \norm{\sum_{i=1}^{n} \mA_i}_{\textnormal{op}}
   \end{align}
   
   Let us also temporarily define $\xi = [\xi_1, \dots, \xi_d]^\top,$ where $\{\xi_i\}_{i = 1}^d$ are i.i.d. normal random variables, then, for all $i \neq j,$
      $\Exp{\frac{\xi_i \xi_j}{\norm{\xi}^2}} = c$
   for some $c \in \R.$ At the same time, $\Exp{\frac{-\xi_i \xi_j}{\norm{\xi}^2}} = c$ since $\xi_i$ and $-\xi_i$ have the same distribution. Thus, $c = 0.$ Moreover, $\Exp{\frac{\xi_i^2}{\norm{\xi}^2}} = \bar{c}$ for all $i \in [d]$ and some $\bar{c} \in \R.$ Therefore, $1 = \sum_{i=1}^{d} \Exp{\frac{\xi_i^2}{\norm{\xi}^2}} = d \bar{c}$ and $\bar{c} = \frac{1}{d}.$ We can conclude that
   \begin{align*}
      \Exp{\mA_i} = \frac{1}{d} \mI_d
   \end{align*}
   for all $i \in [n].$ 
   Next,
   \begin{align*}
      \norm{\mA_i - \Exp{\mA_i}}_{\textnormal{op}} = \norm{\mu^0_{i} (\mu^0_{i})^\top - \frac{1}{d} \mI_d}_{\textnormal{op}} \leq 1
   \end{align*}
   for all $i \in [n]$ since $v^\top (\mu^0_{i} (\mu^0_{i})^\top - \frac{1}{d} \mI_d) v = (v^\top \mu^0_{i})^2 - \frac{1}{d} \norm{v}^2 \leq \left(1 - \frac{1}{d}\right) \norm{v}^2 \leq \norm{v}^2$ and $v^\top (\mu^0_{i} (\mu^0_{i})^\top - \frac{1}{d} \mI_d) v \geq - \frac{1}{d} \norm{v}^2$ for all $v \in \R^d.$
   Finally, 
   \begin{align*}
      \Exp{(\mA_i - \Exp{\mA_i})^2} 
      &= \Exp{\left(\mu^0_{i} (\mu^0_{i})^\top - \frac{1}{d} \mI_d\right)^2} \\
      &= \Exp{\left(\mu^0_{i} (\mu^0_{i})^\top \mu^0_{i} (\mu^0_{i})^\top - \frac{2}{d} \mu^0_{i} (\mu^0_{i})^\top + \frac{1}{d^2} \mI_d\right)} \\
      &= \left(\Exp{\mu^0_{i} (\mu^0_{i})^\top} - \frac{2}{d^2} \mI_d + \frac{1}{d^2} \mI_d\right) = \frac{1}{d} \left(1 - \frac{1}{d}\right) \mI_d
   \end{align*}
   and 
   \begin{align*}
      \sigma^2 \eqdef \norm{\sum_{i=1}^{n} \Exp{(\mA_i - \Exp{\mA_i})^2}}_{\textnormal{op}} \leq \frac{n}{d}.
   \end{align*}
   Using \eqref{eq:CvnhQLXebnWEBxrRQq}, 
   \begin{align*}
      \norm{\begin{bmatrix}
      \left(\mu^0_{1}\right)^\top \\
      \vdots \\
      \left(\mu^0_{n}\right)^\top \\
   \end{bmatrix}}_{\textnormal{op}}^2 \leq \norm{\sum_{i=1}^{n} (\mA_i - \Exp{\mA_i})}_{\textnormal{op}} + \norm{\sum_{i=1}^{n} \Exp{\mA_i}}_{\textnormal{op}} = \norm{\sum_{i=1}^{n} (\mA_i - \Exp{\mA_i})}_{\textnormal{op}} + \frac{n}{d},
   \end{align*}
   where $\{\mA_i - \Exp{\mA_i}\}_{i=1}^n$ are i.i.d. symmetric random matrices with mean zero, $\norm{\mA_i - \Exp{\mA_i}}_{\textnormal{op}} \leq 1$ for all $i \in [n],$ and $\sigma^2 \leq \frac{n}{d}.$ We can use the matrix Bernstein’s inequality (Theorem 5.4.1 in \citep{vershynin}) to get
   \begin{align*}
      \Prob{\norm{\sum_{i=1}^{n} (\mA_i - \Exp{\mA_i})}_{\textnormal{op}} \geq t} \leq 2 d \exp\left(-\frac{t^2}{2 \sigma^2 + \frac{2 t}{3}}\right) \leq 2 d \exp\left(-\min\left\{\frac{t^2}{4 \sigma^2}, \frac{3 t}{4}\right\}\right).
   \end{align*}
   Thus, with probability at least $1 - \rho,$
   \begin{align*}
      \norm{\begin{bmatrix}
      \left(\mu^0_{1}\right)^\top \\
      \vdots \\
      \left(\mu^0_{n}\right)^\top \\
   \end{bmatrix}}_{\textnormal{op}}^2 \leq t + \frac{n}{d},
   \end{align*}
   for 
   \begin{align*}
      t = \max\left\{\frac{4}{3}\log\left(\frac{2 d}{\rho}\right), \sqrt{\frac{4 n}{d} \log\left(\frac{2 d}{\rho}\right)}\right\}.
   \end{align*}
\end{proof}

\begin{lemma}
\label{lemma:norm2}
Let $\rho \in (0,1]$ and $\mu^0_{\ell} = R^0 \eta_{\ell} / \norm{\eta_{\ell}}$, where $R^0 > 0$ and
$\eta_{\ell} \sim \mathcal{N}(0,\mI_d)$ for all $\ell \in [n]$, independently. Moreover, let $\{I_1, \dots, I_{k}\}$ be a (potentially random) partition of $[n]$ of (potentially random) size $k \in [n].$
Let
\[
\Omega_4 \eqdef
\left\{\forall a \in [k]: \norm{\begin{bmatrix}
      \left(\mu^0_{i_{a,1}}\right)^\top \\
      \vdots \\
      \left(\mu^0_{i_{a,n_a}}\right)^\top \\
   \end{bmatrix}}_{\textnormal{op}} \leq \Delta_4\right\}, \qquad I_a \equiv \{i_{a,1}, \dots, i_{a,n_a}\}.
\]
Then choosing
\[
\Delta_4
= R^0 \sqrt{\frac{n}{d} + \max\left\{\frac{4}{3}\log\left(\frac{2 d}{\rho}\right), \sqrt{\frac{4 n}{d} \log\left(\frac{2 d}{\rho}\right)}\right\}}
\]
guarantees that
\[
\Prob{\Omega_4} \geq 1-\rho .
\]
\end{lemma}

\begin{proof}
   Clearly, if the event
   \begin{align*}
      \norm{\begin{bmatrix}
      \left(\mu^0_{1}\right)^\top \\
      \vdots \\
      \left(\mu^0_{n}\right)^\top \\
   \end{bmatrix}}_{\textnormal{op}} \leq \Delta_4
   \end{align*}
   holds, then event $\Omega_4$ holds since 
   \begin{align*}
      \sup_{\norm{v} = 1} \norm{\begin{bmatrix}
      \left(\mu^0_{1}\right)^\top \\
      \vdots \\
      \left(\mu^0_{n}\right)^\top \\
   \end{bmatrix} v}^2 \geq 
   \sup_{\norm{v} = 1} \norm{\begin{bmatrix}
      \left(\mu^0_{i_{a,1}}\right)^\top \\
      \vdots \\
      \left(\mu^0_{i_{a,n_a}}\right)^\top \\
   \end{bmatrix} v}^2
   \end{align*}
   for all $a \in [k].$ It remains to use Lemma~\ref{lemma:norm}.
\end{proof}

\begin{lemma}
\label{lemma:dist_clust}
Let $\rho \in (0,1].$ Let $\mu^0_{\ell} = R^0 \eta_{\ell} / \norm{\eta_{\ell}}$, where $R^0 > 0$ and
$\eta_{\ell} \sim \mathcal{N}(0,\mI_d)$ for all $\ell \in [n]$, independently. Moreover, $\{\bar\mu_a\}_{a \in [m]}$ are non-equal deterministic vectors.
Let
\begin{align}
\label{eq:IMEDMwOCneYdqoINSF}
\Omega_5 \eqdef
\left\{
\max_{a\in[m]}
\left|\frac{n_a}{n}-p_a\right|
\leq \Delta_5
\right\},
\end{align}
where $I_a \eqdef \left\{i \in [n]\,:\, a^*_i = a\right\},$ $a^*_i = \arg\min_{a \in [m]} \norm{\mu^0_i - \bar \mu_a},$ $n_{a} \eqdef |I_{a}|,$ and $p_a \eqdef \Prob{a_1^*=a}$ for all $a \in [m].$
Then choosing
\[
\Delta_5
=
\sqrt{\frac{1}{2n}\log\frac{2m}{\rho}}
\]
guarantees that
\[
\Prob{\Omega_5}\geq 1-\rho.
\]
\end{lemma}

\begin{proof}
   Clearly,
   \begin{align*}
      n_a = \sum_{i=1}^{n} \mathbf{1}[a^*_i = a].
   \end{align*}
   Since $\mu_1^0,\ldots,\mu_n^0$ are independent and identically distributed,
   the random variables $\mathbf{1}[a_i^*=a]$ are independent Bernoulli random
   variables with mean $p_a$. 
   By Hoeffding's inequality, for every $t>0$,
   \begin{align*}
      \Prob{
      \left|\frac{n_a}{n}-p_a\right| \geq t
      }
      \leq
      2\exp(-2nt^2).
   \end{align*}
   Taking the union bound over $a\in[m]$, we obtain
   \begin{align*}
      \Prob{
      \max_{a\in[m]}
      \left|\frac{n_a}{n}-p_a\right|
      \geq t
      }
      &\leq
      \sum_{a=1}^m
      \Prob{
      \left|\frac{n_a}{n}-p_a\right| \geq t
      } \leq
      2m\exp(-2nt^2).
   \end{align*}
   Choosing
   \[
      t=\Delta_5
      \eqdef
      \sqrt{\frac{1}{2n}\log\frac{2m}{\rho}}
   \]
   gives
   \[
      2m\exp(-2n\Delta_5^2)=\rho.
   \]
   Hence,
   \[
      \Prob{
      \max_{a\in[m]}
      \left|\frac{n_a}{n}-p_a\right|
      \leq \Delta_5
      }
      \geq 1-\rho.
   \]
\end{proof}
\begin{lemma}
   \label{lemma:dist_1}
   For all $a, b \in \R^d,$
   \begin{align*}
      \mathbb E_{x \sim \mathcal N(a,\mI)} \left[\min\left\{\exp\left(\frac12\|x-a\|^2-\frac12\|x-b\|^2\right), 1\right\}\right] \leq 2 \exp\left(-\frac{1}{32}\norm{b - a}^2\right).
   \end{align*}
\end{lemma}

\begin{proof}
   If $a = b,$ then the bound is true. Assume that $a \neq b.$ Then
   \begin{align*}
      &\mathbb E_{x \sim \mathcal N(a,\mI)} \left[\min\left\{\exp\left(\frac12\|x-a\|^2-\frac12\|x-b\|^2\right), 1\right\}\right] \\
      &=\mathbb E_{\xi \sim \mathcal N(0,\mI)} \left[\min\left\{\exp\left(\frac12\|\xi\|^2-\frac12\|\xi-(b - a)\|^2\right), 1\right\}\right] \\
      &=\mathbb E_{\xi \sim \mathcal N(0,\mI)} \left[\min\left\{\exp\left(\inp{\xi}{b - a}-\frac12\|b - a\|^2\right), 1\right\}\right].
   \end{align*}
   Note that $\inp{\xi}{\frac{b - a}{\norm{b - a}}} \sim \mathcal{N}(0, 1).$ Thus
   \begin{align*}
      \Prob{\inp{\xi}{b - a} \geq c} = \Prob{\inp{\xi}{\frac{b - a}{\norm{b - a}}} \geq \frac{c}{\norm{b - a}}} \leq \exp\left(-\frac{c^2}{2 \norm{b - a}^2}\right)
   \end{align*}
   for all $c \geq 0$ and 
   \begin{align*}
      \Prob{\inp{\xi}{b - a} \geq \frac{\norm{b - a}^2}{4}} \leq \exp\left(-\frac{\norm{b - a}^2}{32}\right)
   \end{align*}
   for $c = \frac{\norm{b - a}^2}{4}.$
   Using this bound, we get
   \begin{align*}
      &\mathbb E_{x \sim \mathcal N(a,\mI)} \left[\min\left\{\exp\left(\frac12\|x-a\|^2-\frac12\|x-b\|^2\right), 1\right\}\right] \\
      &\leq\mathbb E_{\xi \sim \mathcal N(0,\mI)} \left[\min\left\{\exp\left(\inp{\xi}{b - a}-\frac12\|b - a\|^2\right), 1\right\} \mathbf{1}\left[\inp{\xi}{b - a} < \frac{\norm{b - a}^2}{4}\right]\right] \\
      &\quad + \mathbb E_{\xi \sim \mathcal N(0,\mI)} \left[\min\left\{\exp\left(\inp{\xi}{b - a}-\frac12\|b - a\|^2\right), 1\right\} \mathbf{1}\left[\inp{\xi}{b - a} \geq \frac{\norm{b - a}^2}{4}\right]\right] \\ 
      &\leq \exp\left(-\frac{1}{4}\|b - a\|^2\right) + \exp\left(-\frac{\norm{b - a}^2}{32}\right) \leq 2 \exp\left(-\frac{1}{32}\norm{b - a}^2\right).
   \end{align*}
\end{proof}

\begin{lemma}
   \label{lemma:dist_2}
   For all $a, b, c \in \R^d$ such that $\norm{b - a} < \norm{c - a},$ we have
   \begin{align*}
      &\mathbb E_{x \sim \mathcal N(a,\mI)} \left[\min\left\{\exp\left(\frac12\|x-b\|^2-\frac12\|x-c\|^2\right), 1\right\}\right] \\
      &\leq \exp\left(-\frac{1}{2}\left(\frac12\|c - a\|^2 - \frac12\|b - a\|^2\right)\right) + 2 \exp\left(- \frac{\left(\frac12\|c - a\|^2 - \frac12\|b - a\|^2\right)^2}{2 \norm{c - b}^2}\right).
   \end{align*}
\end{lemma}

\begin{proof}
   If $b = c,$ then the bound is true. Assume that $b \neq c.$ Then
   \begin{align*}
      &\mathbb E_{x \sim \mathcal N(a,\mI)} \left[\min\left\{\exp\left(\frac12\|x-b\|^2-\frac12\|x-c\|^2\right), 1\right\}\right] \\
      &=\mathbb E_{\xi \sim \mathcal N(0,\mI)} \left[\min\left\{\exp\left(\frac12\|\xi - (b - a)\|^2-\frac12\|\xi-(c - a)\|^2\right), 1\right\}\right] \\
      &=\mathbb E_{\xi \sim \mathcal N(0,\mI)} \left[\min\left\{\exp\left(\inp{\xi}{c - b} - \left(\frac12\|c - a\|^2 - \frac12\|b - a\|^2\right)\right), 1\right\}\right].
   \end{align*}
   Note that $Z \eqdef \inp{\xi}{\frac{c - b}{\norm{c - b}}} \sim \mathcal{N}(0, 1).$ We also define $\delta \eqdef \norm{c - b}$ and $\kappa \eqdef \frac12\|c - a\|^2 - \frac12\|b - a\|^2 \geq 0.$ Then,
   \begin{align*}
      \mathbb E_{x \sim \mathcal N(a,\mI)} \left[\min\left\{\exp\left(\frac12\|x-b\|^2-\frac12\|x-c\|^2\right), 1\right\}\right] = \mathbb E_{Z \sim \mathcal N(0,1)} \left[\min\left\{\exp\left(\delta Z - \kappa\right), 1\right\}\right].
   \end{align*}
   Next,
\begin{align*}
&\mathbb E_{Z \sim \mathcal N(0,1)}
\left[\min\left\{\exp\left(\delta Z - \kappa\right), 1\right\}\right] \\
&=
\mathbb E_{Z \sim \mathcal N(0,1)}
\left[\exp\left(\delta Z - \kappa\right) \mathbf{1}\left[\delta Z \leq \kappa\right]\right]
+
\Prob{\delta Z \geq \kappa} \\
&=
e^{-\kappa}
\mathbb E_{Z \sim \mathcal N(0,1)}
\left[
e^{\delta Z}
\mathbf{1}\left[Z \leq \frac{\kappa}{\delta}\right]
\right]
+
\Prob{Z \geq \frac{\kappa}{\delta}} \\
&=
e^{-\kappa+\delta^2/2}
\Prob{Z \leq \frac{\kappa}{\delta}-\delta}
+
\Prob{Z \geq \frac{\kappa}{\delta}},
\end{align*}
since 
\begin{align*}
&e^{-\kappa}
\mathbb E_{Z \sim \mathcal N(0,1)}
\left[
e^{\delta Z}
\mathbf{1}\left[Z \leq \frac{\kappa}{\delta}\right]
\right]
=
e^{-\kappa}
\frac{1}{\sqrt{2\pi}}
\int_{-\infty}^{\kappa/\delta}
e^{\delta z} e^{-z^2/2}\,dz
=
e^{-\kappa}
\frac{1}{\sqrt{2\pi}}
\int_{-\infty}^{\kappa/\delta}
\exp\left(-\frac12(z-\delta)^2+\frac{\delta^2}{2}\right)\,dz
 \\
&=
e^{-\kappa+\delta^2/2}
\frac{1}{\sqrt{2\pi}}
\int_{-\infty}^{\kappa/\delta-\delta}
e^{-u^2/2}\,du
=
e^{-\kappa+\delta^2/2} \Prob{Z \leq \frac{\kappa}{\delta}-\delta}.
\end{align*}
If $\kappa \geq \delta^2,$ then the Gaussian tail bound gives
\begin{align*}
&\mathbb E_{Z \sim \mathcal N(0,1)}
\left[\min\left\{\exp\left(\delta Z - \kappa\right), 1\right\}\right] \\
&= e^{-\kappa+\delta^2/2} \Prob{Z \leq \frac{\kappa}{\delta}-\delta} + \Prob{Z \geq \frac{\kappa}{\delta}} \leq e^{-\kappa/2} + e^{- \frac{\kappa^2}{2 \delta^2}}.
\end{align*}
Otherwise,
\begin{align*}
   \mathbb E_{Z \sim \mathcal N(0,1)}
\left[\min\left\{\exp\left(\delta Z - \kappa\right), 1\right\}\right] \leq e^{-\kappa+\delta^2/2 - \frac{1}{2}(\frac{\kappa}{\delta}-\delta)^2} + e^{- \frac{\kappa^2}{2 \delta^2}} = 2 e^{- \frac{\kappa^2}{2 \delta^2}}.
\end{align*}

Therefore,
\begin{align*}
\mathbb E_{Z \sim \mathcal N(0,1)}
\left[\min\left\{\exp\left(\delta Z - \kappa\right), 1\right\}\right]
&\leq e^{-\kappa/2} + 2 e^{- \frac{\kappa^2}{2 \delta^2}}.
\end{align*}
\end{proof}

\begin{lemma}
   \label{lemma:seq}
   Let $\alpha > 0$ and let $\{z_k\}_{k=0}^{\bar{k}+1}$ be a non-negative sequence such that
   \begin{align*}
      z_{k+1} \leq z_{k} - \alpha z^2_{k}
   \end{align*}
   for all $k \in \{0, \dots, \bar{k}\}.$ Then, for all $k \in \{1, \dots, \bar{k}+1\},$
   \begin{align*}
      z_k \leq \frac{1}{\alpha k}.
   \end{align*}
\end{lemma}

\begin{proof}
   Since
   \begin{align*}
      z_{k+1} \leq z_k - \alpha z_k^2,
   \end{align*}
   and $z_{k+1} \geq 0,$ we have
   \begin{align*}
      z_k - z_{k+1} \geq \alpha z_k^2.
   \end{align*}
   In particular, $z_{k+1} \leq z_k.$ Therefore, if $z_k = 0$ for some $k,$ then $z_{k+1}=0,$ and the claim is trivial for all subsequent iterates. Let $\hat{k} + 1$ be the first index such that $z_{\hat{k} + 1} = 0.$ If it does not exist, then we take $\hat{k} = \infty.$ If $\hat{k} + 1 = 0,$ then the claim is true. We now assume that $\hat{k} + 1 \geq 1.$ For all $k \in \{0, \dots, \min\{\hat{k} - 1, \bar{k}\}\},$
   dividing by $z_k z_{k+1},$ we get
   \begin{align*}
      \frac{1}{z_{k+1}} - \frac{1}{z_k}
      =
      \frac{z_k - z_{k+1}}{z_k z_{k+1}}
      \geq
      \alpha \frac{z_k}{z_{k+1}}
      \geq
      \alpha,
   \end{align*}
   where we used $z_{k+1} \leq z_k.$ Summing this inequality, we obtain
   \begin{align*}
      \frac{1}{z_{k+1}}
      \geq
      \frac{1}{z_0}
      +
      \alpha (k + 1)
   \end{align*}
   and
   \begin{align*}
      z_{k+1} \leq \frac{1}{\alpha (k + 1)}
   \end{align*}
   for all $k \in \{0, \dots, \min\{\hat{k} - 1, \bar{k}\}\}.$
\end{proof}

\begin{lemma}
   \label{lemma:recursion}
   Let $\alpha > 0,$ $\beta \in (0,1),$ and let
   $\{z_k\}_{k=0}^{\bar{k}+1}$ be a non-negative sequence such that
   \begin{align*}
      z_{k+1} \leq z_{k} - \min\{\alpha z^2_{k}, \beta z_{k}\}
   \end{align*}
   for all $k \in \{0, \dots, \bar{k}\}.$ Then, for all
   $k \in \{1, \dots, \bar{k}+1\},$
   \begin{align*}
      z_k
      \leq
      \max\left\{
         (1-\beta)^{k/2} z_0,
         \frac{2}{\alpha k}
      \right\}.
   \end{align*}
\end{lemma}

\begin{proof}
   Since
   \begin{align*}
      z_{k+1} \leq z_k - \min\{\alpha z_k^2,\beta z_k\},
   \end{align*}
   and $\min\{\alpha z_k^2,\beta z_k\} \geq 0,$ we have
   \begin{align*}
      z_{k+1} \leq z_k.
   \end{align*}
   Therefore, if $z_k = 0$ for some $k,$ then $z_{k+1}=0,$ and the claim is
   trivial for all subsequent iterates. Fix $k \in \{1,\dots,\bar{k}+1\}$ and define
   \begin{align*}
      \tau \eqdef \frac{\beta}{\alpha}.
   \end{align*}
   We distinguish two cases. First, suppose that $z_j \geq \tau$ for all
   $j \in \{0,\dots,k-1\}.$ Then
      $\alpha z_j^2 \geq \beta z_j,$
   and hence
      $\min\{\alpha z_j^2,\beta z_j\} = \beta z_j.$
   Therefore,
   \begin{align*}
      z_{j+1}
      \leq
      z_j - \beta z_j
      =
      (1-\beta) z_j.
   \end{align*}
   Unrolling the recursion, we get
   \begin{align*}
      z_k
      \leq
      (1-\beta)^k z_0 \leq (1-\beta)^{k/2} z_0.
   \end{align*}

   We now consider the case where there exists
   $j \in \{0,\dots,k-1\}$ such that $z_j < \tau.$ Let $m$ be the first such
   index. Then, for all $j \in \{0,\dots,m-1\},$ we have $z_j \geq \tau,$
   and therefore, arguing as above,
   \begin{align*}
      z_m
      \leq
      (1-\beta)^m z_0.
   \end{align*}
   Since the sequence is non-increasing, for all $j \geq m,$ we have
   \begin{align*}
      z_j \leq z_m < \tau.
   \end{align*}
   Hence, for all $j \in \{m,\dots,k-1\},$
      $\min\{\alpha z_j^2,\beta z_j\} = \alpha z_j^2.$

   If $m > k/2,$ then
   \begin{align*}
      z_k
      \leq
      z_m
      \leq
      (1-\beta)^m z_0
      \leq
      (1-\beta)^{k/2} z_0.
   \end{align*}

   It remains to consider the case $m \leq k/2.$ If $z_j = 0$ for some
   $j \in \{m,\dots,k\},$ then the claim is trivial. Thus, without loss of
   generality, we assume that $z_j > 0$ for all $j \in \{m,\dots,k\}.$ For all
   $j \in \{m,\dots,k-1\},$ we have
   \begin{align*}
      z_{j+1}
      \leq
      z_j - \alpha z_j^2.
   \end{align*}
   Hence,
   \begin{align*}
      z_j - z_{j+1}
      \geq
      \alpha z_j^2.
   \end{align*}
   Dividing by $z_jz_{j+1},$ we get
   \begin{align*}
      \frac{1}{z_{j+1}} - \frac{1}{z_j}
      =
      \frac{z_j - z_{j+1}}{z_jz_{j+1}}
      \geq
      \alpha \frac{z_j}{z_{j+1}}
      \geq
      \alpha,
   \end{align*}
   where we used $z_{j+1} \leq z_j.$ Summing this inequality from $j=m$ to
   $j=k-1,$ we obtain
   \begin{align*}
      \frac{1}{z_k}
      \geq
      \frac{1}{z_m}
      +
      \alpha(k-m)
      \geq
      \alpha(k-m).
   \end{align*}
   Since $m \leq k/2,$ we have $k-m \geq k/2.$ Thus,
   \begin{align*}
      \frac{1}{z_k}
      \geq
      \frac{\alpha k}{2},
   \end{align*}
   and therefore
   \begin{align*}
      z_k
      \leq
      \frac{2}{\alpha k}.
   \end{align*}

   Combining the two cases, we obtain
   \begin{align*}
      z_k
      \leq
      \max\left\{
         (1-\beta)^{k/2} z_0,
         \frac{2}{\alpha k}
      \right\}.
   \end{align*}
\end{proof}

\begin{lemma}
\label{lemma:log}
Let $a>1$, $b>1$, and
\[
x \geq 4a\sqrt{\max\{1,\log(a^4 b)\}}.
\]
Then
\[
x \geq a\sqrt{\log(bx^4)}.
\]
\end{lemma}

\begin{proof}
Let
\[
L \eqdef \max\{1,\log(a^4b)\}.
\]
Since $x \geq 4a\sqrt{L}$, we have
\[
\frac{x^2}{a^2} \geq 16L.
\]
Denote $y \eqdef x^2/a^2$. Then $y \geq 16L$ and
\[
\log(bx^4)
=
\log(a^4 b y^2)
=
\log(a^4b)+2\log y
\leq
L+2\log y.
\]
Since $y \geq 16L \geq 16$, we have
\[
2\log y \leq \frac{15}{16}y.
\]
Moreover, $L \leq y/16$. Therefore,
\[
\log(bx^4)
\leq
L+2\log y
\leq
\frac{y}{16}+\frac{15y}{16}
=
y.
\]
Thus
\[
a\sqrt{\log(bx^4)}
\leq
a\sqrt{y}
=
x.
\]
\end{proof}

\begin{lemma}
\label{lemma:log-linear}
Let $a \geq 1$, $b \geq 1$, and
\[
x \geq 4\sqrt{a\max\{1,\log(ab)\}}.
\]
Then
\[
x \geq \sqrt{a\log(bx^3)}.
\]
\end{lemma}

\begin{proof}
Let
\[
L \eqdef \max\{1,\log(ab)\}.
\]
Since $x \geq 4\sqrt{aL}$, we have
\[
\frac{x^2}{a} \geq 16L.
\]
Denote $y \eqdef x^2/a$. Then $y \geq 16L$ and
\[
\log(bx^3)
=
\log\left(ba^{3/2}y^{3/2}\right)
=
\log(ab)+\frac{1}{2}\log a+\frac{3}{2}\log y.
\]
Since $b \geq 1$, we have $\log a \leq \log(ab) \leq L$, and therefore
\[
\log(bx^3)
\leq
\frac{3}{2}L+\frac{3}{2}\log y.
\]
Since $y \geq 16L \geq 16$, we have
\[
\log y \leq \frac{1}{2}y.
\]
Moreover, $L \leq y/16$. Therefore,
\[
\log(bx^3)
\leq
\frac{3y}{32}+\frac{3y}{4}
=
\frac{27y}{32}
\leq
y.
\]
Thus
\[
\sqrt{a\log(bx^3)}
\leq
\sqrt{ay}
=
x.
\]
\end{proof}

\newpage\section{Gradient of the Reverse Fisher Divergence for Gaussian Mixture Models}
\label{sec:gradient}
\PROPGRADIENT*

\begin{proof}
   The loss \eqref{eq:main} is differentiable. Thus, using standard differential calculus, it is sufficient to find a vector $a$ such that $d \mathcal L_{\mathrm{rev}}(\mu) = \inp{a}{d \mu_{\ell}}$, where $d \mathcal L_{\mathrm{rev}}(\mu)$ is the differential with respect to $\mu_\ell$. Direct calculations yield 
   \begin{align*}
      \mathcal L_{\mathrm{rev}}(\mu) = \mathbb E_{x\sim p_\mu}
\left[
\|m_\mu(x)-\bar m(x)\|^2
\right].
   \end{align*}
   Thus
   \begin{align*}
      d \mathcal L_{\mathrm{rev}}(\mu) = d \mathbb E_{x\sim p_\mu}
\left[
\|m_\mu(x)-\bar m(x)\|^2
\right] = d \left(\int_{x \in \R^d} \|m_\mu(x)-\bar m(x)\|^2 p_{\mu}(x) \, dx\right).
   \end{align*}
   Interchanging the order of integration and differentiation,
   \begin{align}
      \label{eq:FDsGncgfeDo}
      d \mathcal L_{\mathrm{rev}}(\mu) 
      &= \int_{x \in \R^d} \left(d \|m_\mu(x)-\bar m(x)\|^2\right) p_{\mu}(x) \, dx + \int_{x \in \R^d} \|m_\mu(x)-\bar m(x)\|^2  (d p_{\mu}(x)) \, dx.
   \end{align}
   We now consider the terms separately. Clearly,
   \begin{align}
      \label{eq:HHZhAaJ}
      & d \|m_\mu(x)-\bar m(x)\|^2 = 2 \inp{m_\mu(x)-\bar m(x)}{d (m_\mu(x)-\bar m(x))} = 2 \inp{m_\mu(x)-\bar m(x)}{d m_\mu(x)}
   \end{align}
   because $\bar m(x)$ does not depend on $\mu_{\ell}.$ Next,
   \begin{align*}
      &d m_\mu(x) = \sum_{i=1}^n (d r_i(x)) \mu_i + \sum_{i=1}^n r_i(x) d \mu_i = \sum_{i=1}^n (d r_i(x)) \mu_i + r_{\ell}(x) d \mu_\ell.
   \end{align*}
   Since 
   \begin{align*}
      &d r_{\ell}(x) = d \left(\frac{\exp\left(-\frac12\|x-\mu_{\ell}\|^2\right)}{\sum_{k=1}^n\exp\left(-\frac12\|x-\mu_k\|^2\right)}\right) = r_{\ell} (x) (1 - r_{\ell}(x)) \inp{x - \mu_{\ell}}{d \mu_{\ell}}.
   \end{align*}
   and 
   \begin{align*}
      &d r_{i}(x) = d \left(\frac{\exp\left(-\frac12\|x-\mu_{i}\|^2\right)}{\sum_{k=1}^n\exp\left(-\frac12\|x-\mu_k\|^2\right)}\right) = - r_i(x) r_{\ell}(x) \inp{x - \mu_{\ell}}{d \mu_{\ell}}
   \end{align*}
   for all $i \neq \ell,$ we get
   \begin{align*}
      d m_\mu(x) 
      &= - \sum_{i = 1}^{n} r_i(x) r_{\ell}(x) \inp{x - \mu_{\ell}}{d \mu_{\ell}} \mu_i + r_{\ell} (x) \inp{x - \mu_{\ell}}{d \mu_{\ell}} \mu_{\ell} + r_{\ell}(x) d \mu_\ell \\
      &= - r_{\ell}(x) \inp{x - \mu_{\ell}}{d \mu_{\ell}} m_\mu(x) + r_{\ell} (x) \inp{x - \mu_{\ell}}{d \mu_{\ell}} \mu_{\ell} + r_{\ell}(x) d \mu_\ell \\
      &= r_{\ell} (x) (\mu_{\ell} - m_\mu(x)) \inp{x - \mu_{\ell}}{d \mu_{\ell}} + r_{\ell}(x) d \mu_\ell
   \end{align*}
   due to the definition of $m_\mu(x).$ We now consider $d p_{\mu}(x):$
   \begin{align*}
      d p_{\mu}(x) = d \left(\frac1n\sum_{i=1}^n \varphi(x - \mu_i)\right) = \frac{1}{n} \varphi(x - \mu_{\ell}) \inp{x - \mu_{\ell}}{d \mu_{\ell}}.
   \end{align*}
   Substituting $d m_\mu(x)$ and $d p_{\mu}(x)$ into \eqref{eq:HHZhAaJ} and \eqref{eq:FDsGncgfeDo}, we get
   \begin{align*}
      d \mathcal L_{\mathrm{rev}}(\mu) 
      &= 2 \int_{x \in \R^d} \inp{m_\mu(x)-\bar m(x)}{r_{\ell} (x) (\mu_{\ell} - m_\mu(x)) \inp{x - \mu_{\ell}}{d \mu_{\ell}} + r_{\ell}(x) d \mu_\ell} p_{\mu}(x) \, dx \\
      &\quad + \frac{1}{n} \int_{x \in \R^d} \|m_\mu(x)-\bar m(x)\|^2 \varphi(x - \mu_{\ell}) \inp{x - \mu_{\ell}}{d \mu_{\ell}} \, dx.
   \end{align*}
   Since 
   $r_{\ell} (x) p_{\mu}(x) = \frac{
\exp\left(-\frac12\|x-\mu_{\ell}\|^2\right)
}{
\sum_{k=1}^n
\exp\left(-\frac12\|x-\mu_k\|^2\right)
} \frac1n\sum_{i=1}^n \varphi(x - \mu_i) = \frac{1}{n} \varphi(x - \mu_{\ell}),$
\begin{align*}
   d \mathcal L_{\mathrm{rev}}(\mu) 
   &= \frac{2}{n} \int_{x \in \R^d} \inp{m_\mu(x)-\bar m(x)}{(\mu_{\ell} - m_\mu(x)) \inp{x - \mu_{\ell}}{d \mu_{\ell}} + d \mu_\ell} \varphi(x - \mu_{\ell}) \, dx \\
   &\quad + \frac{1}{n} \int_{x \in \R^d} \|m_\mu(x)-\bar m(x)\|^2 \inp{x - \mu_{\ell}}{d \mu_{\ell}} \varphi(x - \mu_{\ell}) \, dx \\
   &= \frac{2}{n} \mathbb E_{x\sim \mathcal N(\mu_\ell,\mI)} \left[\inp{m_\mu(x)-\bar m(x)}{(\mu_{\ell} - m_\mu(x)) \inp{x - \mu_{\ell}}{d \mu_{\ell}} + d \mu_\ell} \right] \\
   &\quad + \frac{1}{n} \mathbb E_{x\sim \mathcal N(\mu_\ell,\mI)} \left[ \|m_\mu(x)-\bar m(x)\|^2 \inp{x - \mu_{\ell}}{d \mu_{\ell}} \right] \\
   &= \inp{\frac{2}{n} \mathbb E_{x\sim \mathcal N(\mu_\ell,\mI)} \left[m_\mu(x)-\bar m(x)\right]}{d \mu_\ell} \\
   &\quad + \inp{\frac{1}{n} \mathbb E_{x\sim \mathcal N(\mu_\ell,\mI)} \left[ \left(-\norm{m_\mu(x)}^2 + \norm{\bar m(x)}^2 + 2 \inp{m_\mu(x)-\bar m(x)}{\mu_{\ell}}\right)(x - \mu_{\ell})\right]}{d \mu_{\ell}}.
\end{align*}
Using Stein’s identity,
\begin{equation}
\label{eq:gReTdImfiWpFzuWodJ}
\begin{aligned}
   d \mathcal L_{\mathrm{rev}}(\mu) 
   &= \inp{\frac{2}{n} \mathbb E_{x\sim \mathcal N(\mu_\ell,\mI)} \left[m_\mu(x)-\bar m(x)\right]}{d \mu_\ell} \\
   &\quad + \inp{\frac{1}{n} \mathbb E_{x\sim \mathcal N(\mu_\ell,\mI)} \left[ \nabla_x \left(-\norm{m_\mu(x)}^2 + \norm{\bar m(x)}^2 + 2 \inp{m_\mu(x)-\bar m(x)}{\mu_{\ell}}\right)\right]}{d \mu_{\ell}}.
\end{aligned}
\end{equation}
Let $d_x$ be the differential with respect to $x.$ Then,
\begin{equation}
\label{eq:tDypQoOQHznCPpdSg}
\begin{aligned}
   &d_x \left(-\norm{m_\mu(x)}^2 + \norm{\bar m(x)}^2 + 2 \inp{m_\mu(x)-\bar m(x)}{\mu_{\ell}}\right) \\
   &= \left(-2 \inp{m_\mu(x)}{d_x m_\mu(x)} + 2 \inp{\bar m(x)}{d_x \bar m(x)} + 2 \inp{d_x m_\mu(x) - d_x \bar m(x)}{\mu_{\ell}}\right).
\end{aligned}
\end{equation}
Next,
\begin{align*}
   d_x m_\mu(x) 
   &= \sum_{i=1}^n d_x r_i(x) \mu_i = \sum_{i=1}^n \left(d_x \left(\frac{\exp\left(-\frac12\|x-\mu_i\|^2\right)}{\sum_{k=1}^n\exp\left(-\frac12\|x-\mu_k\|^2\right)}\right)\right) \mu_i \\
   &=\sum_{i=1}^n \left(-r_i(x) \inp{x - \mu_i}{d x} + r_i(x) \left(\sum_{k=1}^n r_k(x) \inp{x - \mu_k}{d x}\right)\right) \mu_i \\
   &=\sum_{i=1}^n \left(-r_i(x) \inp{x - \mu_i}{d x} + r_i(x) \inp{x - m_{\mu}(x)}{d x}\right) \mu_i \\
   &=\sum_{i=1}^n \left(r_i(x) \inp{\mu_i - m_{\mu}(x)}{d x}\right) \mu_i \\
   &= \sum_{i=1}^n r_i(x) (\mu_i - m_{\mu}(x)) (\mu_i - m_{\mu}(x))^\top d x = \mC_{\mu} (x) d x
\end{align*}
because $\sum_{i=1}^n r_i(x) m_{\mu}(x) (\mu_i - m_{\mu}(x)) = 0.$ Similarly,
\begin{align*}
   d_x \bar{m} (x) 
   &= \sum_{i=1}^m d_x \bar{r}_i(x) \bar{\mu}_i = \bar \mC (x) d x.
\end{align*}
Substituting into \eqref{eq:tDypQoOQHznCPpdSg},
\begin{align*}
   &d_x \left(-\norm{m_\mu(x)}^2 + \norm{\bar m(x)}^2 + 2 \inp{m_\mu(x)-\bar m(x)}{\mu_{\ell}}\right) \\
   &= -2 \inp{m_\mu(x)}{\mC_{\mu} (x) d x} + 2 \inp{\bar m(x)}{\bar \mC (x) d x} + 2 \inp{\mC_{\mu} (x) d x - \bar \mC (x) d x}{\mu_{\ell}} \\
   &= \inp{2 \mC_{\mu} (x) (\mu_{\ell} - m_\mu(x)) - 2 \bar \mC (x) (\mu_{\ell} - \bar m(x))}{d x}.
\end{align*}
We have found the gradient that we substitute into \eqref{eq:gReTdImfiWpFzuWodJ}:
\begin{align*}
   d \mathcal L_{\mathrm{rev}}(\mu) 
   &= \inp{\frac{2}{n} \mathbb E_{x\sim \mathcal N(\mu_\ell,\mI)} \left[m_\mu(x)-\bar m(x) + \mC_{\mu} (x) (\mu_{\ell} - m_\mu(x)) - \bar \mC (x) (\mu_{\ell} - \bar m(x))\right]}{d \mu_\ell},
\end{align*}
which yields the result.
\end{proof}

\newpage
\section{Missing Steps in the Proof of Theorem~\ref{thm:main_one}}
\label{sec:missing}
The full proof of Theorem~\ref{thm:main_one} is provided in Section~\ref{sec:single}. Here, we provide auxiliary lemmas. The following lemma rewrites the GD steps in matrix format.
\LEMMAZZ*

\begin{proof}
Using Proposition~\ref{prop:grad-general}, in the case when $m = 1,$
\begin{align*}
   \nabla_{\mu_\ell}\mathcal L_{\mathrm{rev}}(\mu)
   =
   \frac{2}{n}
   \mathbb E_{x \sim \mathcal N(\mu_\ell,\mI)}
   \left[
   m_\mu(x)-\bar\mu_1
   +
   \mC_\mu(x)\bigl(\mu_\ell-m_\mu(x)\bigr)
   \right]
\end{align*}
since $\bar \mC(x) = 0$ and $\bar m(x) = \bar\mu_1.$ Then, using \eqref{eq:wsltovjir},
for all $\ell \in [n],$
\begin{align*}
   &\mu^{k+1}_{\ell} - \bar\mu_{1} \\
   &= \mu^{k}_{\ell} - \bar\mu_{1} - \gamma \nabla_{\mu^k_\ell}\mathcal L_{\mathrm{rev}}(\mu^k) \\
   &= \mu^{k}_{\ell} - \bar\mu_{1} - \frac{2 \gamma}{n} \mathbb E_{x \sim \mathcal N(\mu^k_\ell,\mI)} \left[m_{\mu^k}(x) - \bar\mu_{1} + \mC^k(x)\bigl(\mu^k_\ell-m_{\mu^k}(x)\bigr)\right]  \\
   &= \mu^{k}_{\ell} - \bar\mu_{1} - \frac{2 \gamma}{n} \mathbb E_{x \sim \mathcal N(\mu^k_\ell,\mI)} \left[\sum_{i=1}^n r_i^{k}(x)(\mu^k_i - \bar\mu_{1}) + \sum_{i=1}^n r_i^{k}(x) \bigl(\mu^k_i-m_{\mu^k}(x)\bigr) \bigl(\mu^k_i-m_{\mu^k}(x)\bigr)^\top \bigl(\mu^k_\ell-m_{\mu^k}(x)\bigr)\right].
\end{align*}
Since $\sum_{i=1}^n r_i^{k}(x) b \bigl(\mu^k_i-m_{\mu^k}(x)\bigr)^\top = 0$ for all $b \in \R^d,$
\begin{align*}
   \mu^{k+1}_{\ell} - \bar\mu_{1}
   &= \mu^{k}_{\ell} - \bar\mu_{1} - \frac{2 \gamma}{n} \mathbb E_{x \sim \mathcal N(\mu^k_\ell,\mI)} \left[\sum_{i=1}^n r_i^{k}(x)(\mu^k_i - \bar\mu_{1}) + \sum_{i=1}^n r_i^{k}(x) \bigl(\mu^k_i - \bar\mu_{1}\bigr) \bigl(\mu^k_i - m_{\mu^k}(x)\bigr)^\top \bigl(\mu^k_\ell-m_{\mu^k}(x)\bigr)\right] \\
   &= \mu^{k}_{\ell} - \bar\mu_{1} - \frac{2 \gamma}{n}  \sum_{i=1}^n \mathbb E_{x \sim \mathcal N(\mu^k_\ell,\mI)} \left[r_i^{k}(x) + r_i^{k}(x) \bigl(\mu^k_i - m_{\mu^k}(x)\bigr)^\top \bigl(\mu^k_\ell-m_{\mu^k}(x)\bigr) \right] \bigl(\mu^k_i - \bar\mu_{1}\bigr).
\end{align*}
We can conclude that 
\begin{align*}
   \mu^{k+1}_{\ell} - \bar\mu_{1}
   &= \mu^{k}_{\ell} - \bar\mu_{1} - \frac{2 \gamma}{n} \sum_{i=1}^n \mB^{k}_{\ell i} (\mu^{k}_{i} - \bar\mu_{1}).
\end{align*}
\end{proof}

\LEMMAB*

\begin{proof}
Let us define $p^{k}_{\mu}(x) \eqdef \frac{1}{n} \sum_{i=1}^n \varphi(x - \mu^k_i).$ Then $n p^{k}_{\mu}(x) r_{j}^{k}(x) = \varphi(x - \mu^k_{j})$ for all $j \in [n]$ and 
\begin{equation}
\label{eq:BPpMhcWpz}
\begin{aligned}
\mB^{k}_{\ell i} 
&\eqdef \mathbb E_{x \sim \mathcal N(\mu^k_\ell,\mI)}
\left[r_i^{k}(x) + r_i^{k}(x) \bigl(\mu^k_i-m_{\mu^k}(x)\bigr)^\top \bigl(\mu^k_\ell-m_{\mu^k}(x)\bigr)\right] \\
&=\int_{x \in \R^d}
\left[r_i^{k}(x) + r_i^{k}(x) \bigl(\mu^k_i-m_{\mu^k}(x)\bigr)^\top \bigl(\mu^k_\ell-m_{\mu^k}(x)\bigr)\right] \varphi(x - \mu^k_\ell) \, dx\\
&=n \int_{x \in \R^d}
\left[r_i^{k}(x) r^{k}_{\ell} (x) + r_i^{k}(x) r^{k}_{\ell} (x) \bigl(\mu^k_i-m_{\mu^k}(x)\bigr)^\top \bigl(\mu^k_\ell-m_{\mu^k}(x)\bigr)\right] p^{k}_{\mu}(x) \, dx \\
&=n \mathbb E_{x \sim p^{k}_{\mu}}
\left[r_i^{k}(x) r^{k}_{\ell} (x) + r_i^{k}(x) r^{k}_{\ell} (x) \bigl(\mu^k_i-m_{\mu^k}(x)\bigr)^\top \bigl(\mu^k_\ell-m_{\mu^k}(x)\bigr)\right],
\end{aligned}
\end{equation}
where $p^{k}_{\mu}(x) \eqdef \frac{1}{n} \sum_{i=1}^n \varphi(x - \mu^k_i).$ Thus, $\mB^k$ is symmetric. 
For all $v \in \R^{n},$
\begin{align*}
v^\top \mB^k v 
&= n \int_{x \in \R^d}
\left[\left(\sum_{j=1}^{n} r_{j}^{k}(x) v_j\right)^2 + \norm{\sum_{j=1}^{n} r_{j}^{k}(x) v_j (\mu^k_{j}-m_{\mu^k}(x))}^2\right] p^{k}_{\mu}(x) \, dx,
\end{align*}
so it is positive semidefinite. Using Jensen's inequality,
\begin{align*}
v^\top \mB^k v 
&\leq n \int_{x \in \R^d}
\left[\sum_{j=1}^{n} r_{j}^{k}(x) v_j^2 + \sum_{j=1}^{n} r_{j}^{k}(x) v_j^2 \norm{\mu^k_{j}-m_{\mu^k}(x)}^2\right] p^{k}_{\mu}(x) \, dx \\
&=
\sum_{j=1}^{n} \left[ \int_{x \in \R^d} n p^{k}_{\mu}(x) r_{j}^{k}(x) \, dx\right] v_j^2 + \sum_{j=1}^{n} \left[ \int_{x \in \R^d} \norm{\mu^k_{j}-m_{\mu^k}(x)}^2 n p^{k}_{\mu}(x) r_{j}^{k}(x) \, dx \right] v_j^2 \\
&=
\sum_{j=1}^{n} \left[ \int_{x \in \R^d} n p^{k}_{\mu}(x) r_{j}^{k}(x) \, dx\right] v_j^2 + \sum_{j=1}^{n} \mathbb E_{x \sim \mathcal N(\mu^k_j,\mI)} \left[ \norm{\mu^k_{j}-m_{\mu^k}(x)}^2\right] v_j^2 \\
&\leq 
\sum_{j=1}^{n} v_j^2 + 128 \log(n) \sum_{j=1}^{n} v_j^2 = (1 + 128 \log(n)) \norm{v}^2
\end{align*}
since $n p^{k}_{\mu}(x) r_{j}^{k}(x) = \varphi(x - \mu^k_{j}).$ The last inequality follows from the following steps. Let $S \eqdef \{i \in [n]\,:\, \norm{\mu^k_{j}-\mu^k_i}^2 > 64 \log(n)\}.$ Assume $n \geq 2$; the case $n = 1$ is trivial. Then
\begin{align*}
   &\mathbb E_{x \sim \mathcal N(\mu^k_j,\mI)} \left[ \norm{\mu^k_{j}-m_{\mu^k}(x)}^2\right] =\mathbb E_{x \sim \mathcal N(\mu^k_j,\mI)} \left[ \norm{\mu^k_{j}-\sum\limits_{i=1}^n r^k_i(x) \mu^k_i}^2\right] \\
   &\leq \sum\limits_{i=1}^n \mathbb E_{x \sim \mathcal N(\mu^k_j,\mI)} \left[r^k_i(x)\right] \norm{\mu^k_{j}-\mu^k_i}^2 \\
   &= \sum\limits_{i \in S} \mathbb E_{x \sim \mathcal N(\mu^k_j,\mI)} \left[r^k_i(x)\right] \norm{\mu^k_{j}-\mu^k_i}^2 + \sum\limits_{i \not\in S} \mathbb E_{x \sim \mathcal N(\mu^k_j,\mI)} \left[r^k_i(x)\right] \norm{\mu^k_{j}-\mu^k_i}^2 \\
   &\leq \sum\limits_{i \in S} \mathbb E_{x \sim \mathcal N(\mu^k_j,\mI)} \left[\min\left\{\exp\left(\frac12\|x-\mu^k_{j}\|^2-\frac12\|x-\mu^k_i\|^2\right), 1\right\}\right] \norm{\mu^k_{j}-\mu^k_i}^2 + 64 \log(n) \\
   &\overset{\textnormal{L.\ref{lemma:dist_1}}}{\leq} 2 \sum\limits_{i \in S} \exp\left(-\frac{1}{32}\norm{\mu^k_{j}-\mu^k_i}^2\right) \norm{\mu^k_{j}-\mu^k_i}^2 + 64 \log(n) \\
   &\leq 128 n \log(n) \exp\left(-2 \log(n)\right) + 64 \log(n),
\end{align*}
where in the second inequality we use the definition of $r^k_i(x),$ and in the last inequality we use that $\exp\left(-\frac{x}{32}\right) x$ is decreasing for all $x \geq 32.$ Thus, $\mathbb E_{x \sim \mathcal N(\mu^k_j,\mI)} \left[ \norm{\mu^k_{j}-m_{\mu^k}(x)}^2\right] \leq 128 \log(n).$
\end{proof}

\LEMMAZORIG*

\begin{proof}
   Let $p^{k}_{\mu}(x) \eqdef \frac{1}{n} \sum_{i=1}^n \varphi(x - \mu^k_i).$ Using Lemma~\ref{lemma:Z} and \eqref{eq:BPpMhcWpz},
\begin{align*}
   &\inp{\mZ^{k}}{\mB^k \mZ^{k}} \\
   &= n \sum_{i=1}^n \sum_{j=1}^n \mathbb E_{x \sim p^{k}_{\mu}}
   \left[r_i^{k}(x) r^{k}_{j} (x)  + r_i^{k}(x) r^{k}_{j} (x) \bigl(\mu^k_i-m_{\mu^k}(x)\bigr)^\top \bigl(\mu^k_j-m_{\mu^k}(x)\bigr)\right] (\mu^{k}_{i} - \bar \mu_{1})^\top (\mu^{k}_{j} - \bar \mu_{1}) \\
   &= n \mathbb E_{x \sim p^{k}_{\mu}}
   \left[\sum_{i=1}^n r_i^{k}(x) (\mu^{k}_{i} - \bar \mu_{1})^\top \left(\sum_{j=1}^n r^{k}_{j} (x) (\mu^{k}_{j} - \bar \mu_{1})\right)\right] \\
   &\quad  + n \mathbb E_{x \sim p^{k}_{\mu}} 
   \left[\sum_{i=1}^n r_i^{k}(x) \bigl(\mu^k_i-m_{\mu^k}(x)\bigr)^\top \left(\sum_{j=1}^n  r^{k}_{j}(x) \bigl(\mu^k_j-m_{\mu^k}(x)\bigr) (\mu^{k}_{j} - \bar \mu_{1})^\top \right) (\mu^{k}_{i} - \bar \mu_{1})\right] \\
   &= n \mathbb E_{x \sim p^{k}_{\mu}}
   \left[\norm{\sum_{i=1}^n r_i^{k}(x) (\mu^{k}_{i} - \bar \mu_{1})}^2\right] + n \mathbb E_{x \sim p^{k}_{\mu}} 
   \left[\norm{\sum_{i=1}^n  r^{k}_{i}(x) \bigl(\mu^k_i-m_{\mu^k}(x)\bigr) (\mu^{k}_{i} - \bar \mu_{1})^\top}_F^2\right].
\end{align*}
Using the definition of $m_{\mu^k}(x)$ and $\sum_{i=1}^n r_i^{k}(x) \bigl(\mu^k_i-m_{\mu^k}(x)\bigr) b^\top = 0$ for all $b \in \R^d,$
\begin{align*}
   &\inp{\mZ^{k}}{\mB^k \mZ^{k}} \\
   &= n \mathbb E_{x \sim p^{k}_{\mu}}
   \left[\norm{m_{\mu^k}(x) - \bar \mu_{1}}^2\right] + n \mathbb E_{x \sim p^{k}_{\mu}} 
   \left[\norm{\sum_{i=1}^n  r^{k}_{i}(x) \bigl(\mu^k_i-m_{\mu^k}(x)\bigr) (\mu^{k}_{i} - m_{\mu^k}(x))^\top}_F^2\right] \\
   &= n \mathbb E_{x \sim p^{k}_{\mu}}
   \left[\norm{m_{\mu^k}(x) - \bar \mu_{1}}^2\right] + n \mathbb E_{x \sim p^{k}_{\mu}} 
   \textnormal{tr}\left(\left(\sum_{i=1}^n  r^{k}_{i}(x) \bigl(\mu^k_i-m_{\mu^k}(x)\bigr) (\mu^{k}_{i} - m_{\mu^k}(x))^\top\right)^2\right).
\end{align*}
Let $\kappa \eqdef \frac{n}{\min\{n, d\}}.$ Due to $\left(\textnormal{tr}(\mA)\right)^2 \leq \textnormal{rank}(\mA) \textnormal{tr}(\mA^2)$ for all symmetric $\mA \in \R^{d \times d},$ we get
\begin{align*}
   \inp{\mZ^{k}}{\mB^k \mZ^{k}} 
   &\geq n \mathbb E_{x \sim p^{k}_{\mu}}
   \left[\norm{m_{\mu^k}(x) - \bar \mu_{1}}^2\right] + \kappa \mathbb E_{x \sim p^{k}_{\mu}} 
   \left[\left(\sum_{i=1}^n  r^{k}_{i}(x) \norm{\mu^k_i-m_{\mu^k}(x)}^2\right)^2\right],
\end{align*}
where the rank of our matrix under the trace is $\leq \min\{n, d\}.$ Using Jensen's inequality,
\begin{align*}
   \inp{\mZ^{k}}{\mB^k \mZ^{k}}  &\geq n \mathbb E_{x \sim p^{k}_{\mu}}
   \left[\norm{m_{\mu^k}(x) - \bar \mu_{1}}^2\right] + \kappa \mathbb E_{x \sim p^{k}_{\mu}} 
   \left[\left(\sum_{i=1}^n  r^{k}_{i}(x) \norm{\mu^k_i - \bar\mu_1}^2 - \norm{m_{\mu^k}(x) - \bar\mu_1}^2\right)^2\right] \\
   &\geq n \mathbb E_{x \sim p^{k}_{\mu}}
   \left[\norm{m_{\mu^k}(x) - \bar \mu_{1}}^2\right] + \kappa \left(\mathbb E_{x \sim p^{k}_{\mu}} 
   \left[\sum_{i=1}^n  r^{k}_{i}(x) \norm{\mu^k_i - \bar\mu_1}^2\right] - \mathbb E_{x \sim p^{k}_{\mu}} 
   \left[\norm{m_{\mu^k}(x) - \bar\mu_1}^2\right]\right)^2 \\
   &= n \mathbb E_{x \sim p^{k}_{\mu}}
   \left[\norm{m_{\mu^k}(x) - \bar \mu_{1}}^2\right] + \kappa \left(\frac{1}{n} \norm{\mZ^{k}}^2_F - \mathbb E_{x \sim p^{k}_{\mu}} 
   \left[\norm{m_{\mu^k}(x) - \bar\mu_1}^2\right]\right)^2
\end{align*}
since 
\begin{align*}
   &\mathbb E_{x \sim p^{k}_{\mu}} \left[\sum_{i=1}^n  r^{k}_{i}(x) \norm{\mu^k_i - \bar\mu_1}^2\right] = \sum_{i=1}^n \int_{x \in \R^d} \left[r^{k}_{i}(x) p^{k}_{\mu}(x) \norm{\mu^k_i - \bar\mu_1}^2\right] \, d x \\
   &= \frac{1}{n} \sum_{i=1}^n \int_{x \in \R^d} \left[\norm{\mu^k_i - \bar\mu_1}^2\right] \varphi(x - \mu^k_i) \, d x = \frac{1}{n} \sum_{i=1}^n \norm{\mu^k_i - \bar\mu_1}^2 = \frac{1}{n} \norm{\mZ^{k}}^2_F
\end{align*}
Notice also that 
\begin{align*}
   \mathbb E_{x \sim p^{k}_{\mu}} \left[\norm{m_{\mu^k}(x) - \bar\mu_1}^2\right] \leq \mathbb E_{x \sim p^{k}_{\mu}} \left[\sum_{i=1}^n  r^{k}_{i}(x) \norm{\mu^k_i - \bar\mu_1}^2\right] = \frac{1}{n} \norm{\mZ^{k}}^2_F.
\end{align*}
Thus, 
\begin{align}
   \label{eq:zpfbtoigdmEHSImuC}
   \inp{\mZ^{k}}{\mB^k \mZ^{k}} 
   &\geq \min_{0 \leq x \leq \bar{x}} \left[n x + \kappa \left(\bar{x} - x\right)^2\right],
\end{align}
where $\bar{x} \eqdef \frac{1}{n} \norm{\mZ^{k}}^2_F.$ Since 
\begin{align*}
   \min_{0 \leq x \leq \bar{x}} \left[n x + \kappa \left(\bar{x} - x\right)^2\right] \geq \min\left\{\min_{0 \leq x \leq \frac{\bar{x}}{2}}\kappa \left(\bar{x} - x\right)^2, \min_{\frac{\bar{x}}{2} \leq x \leq \bar{x}} n x\right\} = \min\left\{\frac{\kappa}{4} \bar{x}^2, \frac{n}{2} \bar{x}\right\},
\end{align*}
we obtain
\begin{align*}
   \inp{\mZ^{k}}{\mB^k \mZ^{k}} \geq \min\left\{\frac{\kappa}{4 n^2} \norm{\mZ^{k}}^4_F, \frac{1}{2} \norm{\mZ^{k}}^2_F\right\}
   =
   \min\left\{
   \frac{1}{4 n \min\{n,d\}} \norm{\mZ^{k}}^4_F,
   \frac{1}{2} \norm{\mZ^{k}}^2_F
   \right\}.
\end{align*}
\end{proof}

\newpage

\newpage
\section{Proof of Theorem~\ref{thm:mainmany}}
\newcommand{\hatz}{65 \sqrt{\log\left(2^{16} n (R^0)^4\right)} \sqrt{\frac{n}{d} + \max\left\{\frac{4}{3}\log\left(\frac{2 d}{\rho}\right), \sqrt{\frac{4 n}
{d} \log\left(\frac{2 d}{\rho}\right)}\right\}}}
\newcommand{\hatK}{\frac{n \log\left(R^0\right)}{2 \gamma} + \frac{2 n \min\{n, d\} \max\left\{4 n, 8 \tilde{z}^2\right\}}{\gamma \varepsilon}}
\newcommand{\gammastep}{\frac{n}{36 \tilde{z}^2}}
\begin{theorem}[Main Result for Multiple Target Modes]
   \label{thm:mainmany}
   Let $\varepsilon \in (0, 1]$ denote the target accuracy. We define the smallest and largest gaps in the target distribution as
   \begin{align*}
      \delta_{\min} \eqdef \min_{a \neq b}\norm{\bar\mu_a - \bar\mu_b} > 0 \quad \textnormal{ and } \quad  \delta_{\max} \eqdef \max_{a \neq b}\norm{\bar\mu_a - \bar\mu_b},
   \end{align*}
   and the initialization radius as
   $$R^0 \geq \max\left\{50 + \frac{2^{20} n m^2 \sqrt{d} \delta^2_{\max}}{\rho}, 2^{12} \sqrt{n \max\{\log\left(2^{66} n^3 \bar{R}^4\right), 1\}} \bar{R}\right\},$$
   where $\bar{R} \eqdef \max_{b \in [m]} \norm{\bar \mu_b}$ and $\rho \in (0, 1].$ We sample the initial parameters as $\mu^0_{\ell} = R^0 \eta_{\ell} / \norm{\eta_{\ell}}$, where $\eta_{\ell} \sim \mathcal{N}(0, \mI)$ for all $\ell \in [n]$, and the random vectors $\{\eta_{\ell}\}_{\ell=1}^n$ are i.i.d.
   Assume that the minimal gap satisfies
   $$\delta_{\min} \geq 2^{10} \sqrt{\log\left(\frac{2^{20} n^{3} \max\left\{n, m\right\} (R^0)^4}{\varepsilon^2}\right)} \sqrt{\frac{n}{d} + \max\left\{\frac{4}{3}\log\left(\frac{2 d}{\rho}\right), \sqrt{\frac{4 n}{d} \log\left(\frac{2 d}{\rho}\right)}\right\}}$$
   and $d \geq 8 + \max\left\{2 \log\left(\frac{m n^2}{\rho}\right), 8 \log\left(\frac{n^4}{2 \pi \rho^2}\right)\right\}.$
   If
   $\gamma \leq \gammastep$, 
   where $\tilde{z} \eqdef \hatz$, and GD \eqref{eq:wsltovjir} is run for 
   \begin{align*}
      \tilde{K} \eqdef \hatK
   \end{align*}
   iterations, then, with probability at least $1 - 4\rho$, we have that
   \begin{align*}
      \norm{\mu^{k}_{\ell} - \bar\mu_{a^*_{\ell}}}^2 \leq \varepsilon,
   \end{align*}
   for all $\ell \in [n]$ and $k \geq \tilde{K}$, where $a^*_{\ell} = \arg\min\limits_{a \in [m]} \norm{\mu^0_{\ell} - \bar{\mu}_a}$ is the index of the target mode closest to the initialization point $\mu^0_{\ell}$.
\end{theorem}

\begin{proof}
   Let 
   \begin{align*}
      I_a \eqdef \left\{i \in [n]\,:\, a^*_i = a\right\},
   \end{align*} 
   where $a^*_i = \arg\min_{a \in [m]} \norm{\mu^0_i - \bar \mu_a},$ $n_{a} \eqdef |I_{a}|,$ and $A^*$ is the set of $a \in [m]$ such that $I_{a}$ is non-empty.
   We have the initialization $\mu^0_{\ell} = R^0 \eta_{\ell} / \norm{\eta_{\ell}},$ where $\eta_{\ell} \sim \mathcal{N}(0, \mI)$ for all $\ell \in [n]$ and $\{\eta_{\ell}\}$ are i.i.d. Consider an event $\Omega = \Omega_1 \bigcap \Omega_2 \bigcap \Omega_3 \bigcap \Omega_4,$ where
   \begin{align}
      \Omega_1 
      &\eqdef \left\{\forall i \in [n]: \norm{\mu^0_i - \bar \mu_{a^*_i}}^2 \leq \min_{b \neq a^*_i}\norm{\mu^0_i - \bar \mu_b}^2 - \Delta_1\right\}, \label{eq:zolMqjeooMG}\\
      \Omega_2 &\eqdef \left\{ \forall i \neq j \in [n]: \norm{\mu^0_i - \mu^0_j}^2 \geq \Delta_2 \right\}, \label{eq:EvPlrGmTnwYCRQa} \\
      \Omega_3 &\eqdef \left\{ \forall i \neq j \in [n] \textnormal{ s.t. } a^*_i \neq a^*_j: \textnormal{dist}([\mu^0_i, \bar{\mu}_{a^*_i}], [\mu^0_j, \bar{\mu}_{a^*_j}]) \geq \Delta_3 \right\}, \label{eq:yZUzuDZUiFFPhgyiwgy} \\
      \Omega_4 &\eqdef
\left\{\forall a \in A^*: \norm{\begin{bmatrix}
      \left(\mu^0_{i_{a,1}}\right)^\top \\
      \vdots \\
      \left(\mu^0_{i_{a,n_a}}\right)^\top \label{eq:KfgolNiUc}
   \end{bmatrix}}_{\textnormal{op}} \leq \Delta_4\right\}, \qquad I_a \equiv \{i_{a,1}, \dots, i_{a,n_a}\}.
   \end{align}
   Using Lemmas~\ref{lemma:close}, \ref{lemma:far}, \ref{lemma:dist}, and \ref{lemma:norm2}, we get $\Prob{\Omega} \geq 1 - 4 \rho$ with $\Delta_1 = \frac{2 R^0 \delta_{\min} \rho}{n m^2 \sqrt{d}},$ $\Delta_2 = (R^0)^2 \left(\frac{\sqrt{2\pi}\rho}{\sqrt d\, n^2}\right)^{\frac{2}{d-1}},$ $\Delta_3 = \frac{\delta_{\min}}{8} \left(\frac{\rho}{m n^2}\right)^{\frac{1}{d-2}},$ and $$\Delta_4
= R^0 \sqrt{\frac{n}{d} + \max\left\{\frac{4}{3}\log\left(\frac{2 d}{\rho}\right), \sqrt{\frac{4 n}{d} \log\left(\frac{2 d}{\rho}\right)}\right\}}.$$ Starting from here, we assume that the event $\Omega$ holds.
   We now fix any $\ell \in [n].$ Using Proposition~\ref{prop:grad-general}, we have
   \begin{align*}
      \nabla_{\mu^k_\ell}\mathcal L_{\mathrm{rev}}(\mu^k)
      =
      \frac{2}{n}
      \mathbb E_{x \sim \mathcal N(\mu^k_\ell,\mI)}
      \left[
      m_{\mu^k}(x)-\bar m(x)
      +
      \mC_{\mu^k}(x)\bigl(\mu^k_\ell-m_{\mu^k}(x)\bigr)
      -
      \bar \mC(x)\bigl(\mu^k_\ell-\bar m(x)\bigr)
      \right].
   \end{align*}
   We define 
   \[
   r_i^{k}(x)
   \eqdef
   \frac{
   \exp\left(-\frac12\|x-\mu^k_i\|^2\right)
   }{
   \sum_{j\in [n]}
   \exp\left(-\frac12\|x-\mu^k_j\|^2\right)
   }.
   \]
   For all $i\in I_{a}$, we also define 
   \[
   r_i^{k,a}(x)
   \eqdef
   \frac{
   \exp\left(-\frac12\|x-\mu^k_i\|^2\right)
   }{
   \sum_{j\in I_{a}}
   \exp\left(-\frac12\|x-\mu^k_j\|^2\right)
   }, \qquad
   m^k_{a}(x)
   \eqdef
   \sum_{i\in I_{a}} r_i^{k,a}(x)\mu^k_i,
   \]
   and 
   \[
   \mC^k_{a}(x) \eqdef \sum_{i\in I_{a}} r_i^{k,a}(x) \bigl(\mu^k_i-m^k_{a}(x)\bigr) \bigl(\mu^k_i-m^k_{a}(x)\bigr)^\top.
   \]
   We now add and subtract $m^k_{a^*_\ell}(x) - \bar\mu_{a^*_\ell} + \mC^k_{a^*_\ell}(x)\bigl(\mu^k_\ell-m^k_{a^*_\ell}(x)\bigr),$ to get
   \begin{align}
      \label{eq:grad_1}
      &\nabla_{\mu^k_\ell}\mathcal L_{\mathrm{rev}}(\mu^k)
      = 
      \frac{2}{n} \mathbb E_{x \sim \mathcal N(\mu^k_\ell,\mI)} \left[m^k_{a^*_\ell}(x) - \bar\mu_{a^*_\ell} + \mC^k_{a^*_\ell}(x)\bigl(\mu^k_\ell-m^k_{a^*_\ell}(x)\bigr) + \mathcal{E}^k_{\ell}(x)\right],
   \end{align}
   where we define 
   \begin{align*}
      \mathcal{E}^k_{\ell}(x) 
      &\eqdef
      (m_{\mu^k}(x) - m^k_{a^*_\ell}(x)) - (\bar m(x) - \bar \mu_{a^*_\ell})
      +
      (\mC_{\mu^k}(x) - \mC^k_{a^*_\ell}(x))(\mu^k_\ell - m^k_{a^*_\ell}(x)) \\
      &\quad + \mC_{\mu^k}(x)(m^k_{a^*_\ell}(x) - m_{\mu^k}(x)) - \bar \mC(x)\bigl(\mu^k_\ell-\bar m(x)\bigr).
   \end{align*} 
   Let $\hat\tau^k_{a}(x) \eqdef \sum_{i \in [n], i \not\in I_a} r_i^{k}(x)$ and $\bar \tau_{a}(x) \eqdef \sum_{b \in [m] , b \neq a} \bar r_b(x).$  In Section~\ref{sec:auxlemmas}, we prove the following Lemma~\ref{lemma:bound1}.
   \begin{restatable}{lemma}{LEMMABOUND}
      \label{lemma:bound1}
   Let $R_{\max} \eqdef \max\{\max_{i \in [n]} \norm{\mu_i^k}, \max_{a \in [m]} \norm{\bar\mu_a}\}$ for some $k \geq 0.$ Then,
   \begin{align*}
      \norm{\mathcal{E}^k_{\ell}(x)} \leq 52 \max\{R_{\max}, R_{\max}^3\} \max\{\hat\tau^k_{a^*_\ell}(x), \bar \tau_{a^*_\ell}(x)\}.
   \end{align*}
   for all $x \in \R^d$ and $\ell \in [n].$
   \end{restatable}
   Next, we add and subtract $\mu^k_{\ell} - \bar\mu_{a^*_\ell}$ once again to get
   \begin{align}
      \label{eq:grad_2}
      &\nabla_{\mu^k_\ell}\mathcal L_{\mathrm{rev}}(\mu^k)
      = 
      \frac{2}{n} \mathbb E_{x \sim \mathcal N(\mu^k_\ell,\mI)}
      \left[\mu^k_{\ell} - \bar\mu_{a^*_\ell} + \bar{\mathcal{E}}^k_{\ell}(x) + \mathcal{E}^k_{\ell}(x)\right],
   \end{align}
   where
   \begin{align*}
      \bar{\mathcal{E}}^k_{\ell}(x) \eqdef m^k_{a^*_\ell}(x) - \mu^k_{\ell} + \mC^k_{a^*_\ell}(x)\bigl(\mu^k_\ell-m^k_{a^*_\ell}(x)\bigr).
   \end{align*}
   Let $\tau^k_{\ell}(x) \eqdef \sum_{i \in I_{a^*_{\ell}}, i \neq \ell} r_i^{k,a^*_{\ell}}(x).$ In Section~\ref{sec:auxlemmas}, we prove the following Lemma~\ref{lemma:bound2}.
   \begin{restatable}{lemma}{LEMMABOUNDTWO}
      \label{lemma:bound2}
   Let $R_{\max} \eqdef \max\{\max_{i \in [n]} \norm{\mu_i^k}, \max_{a \in [m]} \norm{\bar\mu_a}\}$ for some $k \geq 0.$ Then,
   \begin{align*}
      \norm{\bar{\mathcal{E}}^k_{\ell}(x)} \leq 10 \max\{R_{\max}, R_{\max}^3\} \tau^k_{\ell}(x).
   \end{align*}
   for all $x \in \R^d$ and $\ell \in [n].$
   \end{restatable}
   \textbf{(Global convergence stage).} Using \eqref{eq:grad_2} and the update rule $\mu^{k+1}_{\ell} = \mu^{k}_{\ell} - \gamma \nabla_{\mu^k_\ell}\mathcal L_{\mathrm{rev}}(\mu^k),$
   \begin{align}
      \label{eq:rwPVzXoMnmQzZXKel}
      \mu^{k+1}_{\ell} = \bar \mu_{a^*_\ell} + \left(1 - \frac{2 \gamma}{n}\right)^{k+1} (\mu^0_{\ell} - \bar \mu_{a^*_\ell}) - E_{\ell}^k,
   \end{align}
   where 
   \begin{align*}
      E_{\ell}^k \eqdef \frac{2 \gamma}{n}\sum_{j=0}^{k} \left(1 - \frac{2 \gamma}{n}\right)^{k - j} \mathbb E_{x \sim \mathcal N(\mu^j_\ell,\mI)}
      \left[\bar{\mathcal{E}}^j_{\ell}(x) + \mathcal{E}^j_{\ell}(x)\right]
   \end{align*}
   and 
   \begin{align}
      \label{eq:UEwmx}
      \norm{E_{\ell}^k} \leq \max_{j \in \{0, \dots, k\}}\mathbb E_{x \sim \mathcal N(\mu^j_\ell,\mI)}
      \left[\norm{\bar{\mathcal{E}}^j_{\ell}(x) + \mathcal{E}^j_{\ell}(x)}\right].
   \end{align}
   Let 
   \begin{align}
      \label{eq:TUyQcvjVGXjffeyd}
      \bar{K} \eqdef \flr{\log\left(\frac{R^0}{4 \sqrt{\bar{\Delta}}}\right) / \left(-\log\left(1 - \frac{2 \gamma}{n}\right)\right)},
   \end{align}
   where $\bar{\Delta} > 0$ is a free parameter.
   Then
   \begin{align}
      \label{eq:BVNUOpDnYoTieBnA}
      \left(1 - \frac{2 \gamma}{n}\right)^{k} \norm{\mu^0_{\ell} - \bar \mu_{a^*_\ell}} \geq \left(1 - \frac{2 \gamma}{n}\right)^{k} \frac{R^0}{2} \geq 2 \sqrt{\bar{\Delta}}
   \end{align}
   for all $0 \leq k \leq \bar{K}.$
   Using mathematical induction, for all $0 \leq k \leq \bar{K},$ we assume that 
   \begin{align}
      &\norm{\mu^k_{\ell}} \leq R, \label{eq:R}\\
      &\norm{\mu^k_{\ell} - \bar \mu_{a^*_\ell}}^2 \leq \min_{b \neq a^*_\ell} \left[\norm{\mu^k_{\ell} - \bar \mu_b}^2 - \frac{\norm{\bar \mu_{a^*_\ell} - \bar \mu_b}^2}{2}\right] \label{eq:Delta1}, \\
      &\norm{\mu^{k}_{\ell} - \mu^{k}_{j}}^2 \geq \hat\Delta_2 \label{eq:Delta2} 
   \end{align}
   for all $\ell \neq j \in [n],$ where 
   \begin{align}
      \label{eq:kjdsqYLZ}
      R \eqdef 2 R^0, \quad E_1 \eqdef \frac{\delta_{\min}}{8 R}, \quad \hat\Delta_2 \eqdef 32 \log\left(3472 n R^4\right), \quad \bar{\Delta} \eqdef \frac{\hat\Delta_2}{2}.
   \end{align}
   For our choice of parameters, we prove the following useful lemma in Section~\ref{sec:auxlemmas}.
   \begin{restatable}{lemma}{LEMMAPARAMS}
      For our choice of parameters, 
      \begin{equation}
      \begin{aligned}
         \label{eq:tdbKKNEBxox}
         &\Delta_1 \geq \delta^2_{\max}, \quad \Delta_2 \geq \hat\Delta_2 \geq 32 \log(434 n R^3 / E_1), \quad \delta_{\min} \geq \sqrt{32 \log\left(434 R^3 m / E_1\right)}, \\
         &E_1 \leq \frac{R}{2}, \quad \frac{4 \sqrt{\bar{\Delta} \Delta_2}}{R^0} \geq 2 \sqrt{\hat\Delta_2}, \quad \Delta_3 \geq 2 \sqrt{\hat\Delta_2}, \quad E_1 \leq \frac{\sqrt{\hat\Delta_2}}{2}.
      \end{aligned}
      \end{equation}
   \end{restatable}
   The case $k = 0$ of the mathematical induction is true: i) \eqref{eq:R} follows from $\norm{\mu^0_{\ell}} = R^0;$ ii) \eqref{eq:Delta1} follows from \eqref{eq:zolMqjeooMG} and $\Delta_1 \geq \delta^2_{\max}$ (see \eqref{eq:tdbKKNEBxox}); iii) \eqref{eq:Delta2} follows from \eqref{eq:EvPlrGmTnwYCRQa} and $\hat\Delta_2 \leq \Delta_2$ (see \eqref{eq:tdbKKNEBxox}). We now prove the next step of the induction $k \to k + 1.$ For all $\ell \in [n],$ using Lemmas~\ref{lemma:bound1} and \ref{lemma:bound2}, 
   \begin{align}
      &\mathbb E_{x \sim \mathcal N(\mu^k_\ell,\mI)} \left[\norm{\bar{\mathcal{E}}^k_{\ell}(x) + \mathcal{E}^k_{\ell}(x)}\right] \leq \mathbb E_{x \sim \mathcal N(\mu^k_\ell,\mI)} \left(\norm{\bar{\mathcal{E}}^k_{\ell}(x)} + \norm{\mathcal{E}^k_{\ell}(x)}\right) \nonumber \\
      &\leq 62 \max\{R, R^3\} \mathbb E_{x \sim \mathcal N(\mu^k_\ell,\mI)} (\hat\tau^k_{a^*_\ell}(x) + \bar \tau_{a^*_\ell}(x) + \tau^k_{\ell}(x)) \label{eq:xdpFroas}
   \end{align}
   First,
   \begin{align}
      \label{eq:mLEYuytmsTVADUzUp}
      &\mathbb E_{x \sim \mathcal N(\mu^k_\ell,\mI)} \tau^k_{\ell}(x) = \sum_{i \in I_{a^*_{\ell}}, i \neq \ell} \mathbb E_{x \sim \mathcal N(\mu^k_\ell,\mI)} r_i^{k,a^*_{\ell}}(x),
   \end{align}
   where
   \begin{align*}
      \mathbb E_{x \sim \mathcal N(\mu^k_\ell,\mI)} r_i^{k,a^*_{\ell}}(x) 
      &\leq \mathbb E_{x \sim \mathcal N(\mu^k_\ell,\mI)} \left[\min\left\{\exp\left(\frac12\|x-\mu^k_{\ell}\|^2-\frac12\|x-\mu^k_i\|^2\right), 1\right\}\right] \\
      &\leq 2 \exp\left(-\frac{1}{32} \norm{\mu^k_{\ell} - \mu^k_{i}}^2\right)
   \end{align*}
   for all $i \in I_{a^*_{\ell}},$ $i \neq \ell,$ due to Lemma~\ref{lemma:dist_1}. Using \eqref{eq:Delta2},
   \begin{align*}
      \mathbb E_{x \sim \mathcal N(\mu^k_\ell,\mI)} r_i^{k,a^*_{\ell}}(x) 
      &\leq 2 \exp\left(-\frac{\hat\Delta_2}{32}\right).
   \end{align*}
   Substituting into \eqref{eq:mLEYuytmsTVADUzUp},
   \begin{align}
      \label{eq:RpwaEvNpBJdp}
      &\mathbb E_{x \sim \mathcal N(\mu^k_\ell,\mI)} \tau^k_{\ell}(x) \leq 2 n \exp\left(-\frac{\hat\Delta_2}{32}\right).
   \end{align}
   Similarly,
   \begin{align*}
      \mathbb E_{x \sim \mathcal N(\mu^k_\ell,\mI)} \hat\tau^k_{a^*_\ell}(x) \leq 2 n \exp\left(-\frac{\hat\Delta_2}{32}\right)
   \end{align*}
   because
   \begin{align*}
      \mathbb E_{x \sim \mathcal N(\mu^k_\ell,\mI)} r_i^{k}(x) &\leq \mathbb E_{x \sim \mathcal N(\mu^k_\ell,\mI)} \left[\min\left\{\exp\left(\frac12\|x-\mu^k_{\ell}\|^2-\frac12\|x-\mu^k_i\|^2\right), 1\right\}\right] \\
      &\leq 2 \exp\left(-\frac{1}{32} \norm{\mu^k_{\ell} - \mu^k_{i}}^2\right)
   \end{align*}
   for all $i \not\in I_{a^*_{\ell}}.$ Finally,
   \begin{align*}
      \mathbb E_{x \sim \mathcal N(\mu^k_\ell,\mI)} \bar \tau_{a^*_\ell}(x) = \sum_{b \in [m] , b \neq a^*_{\ell}} \mathbb E_{x \sim \mathcal N(\mu^k_\ell,\mI)} \bar r_b(x).
   \end{align*}
   Due to Lemma~\ref{lemma:dist_2},
   \begin{align*}
      \mathbb E_{x \sim \mathcal N(\mu^k_\ell,\mI)} \bar{r}_b(x) 
      &\leq \mathbb E_{x \sim \mathcal N(\mu^k_\ell,\mI)} \left[\min\left\{\exp\left(
-\frac12
\left(
\|x-\bar\mu_b\|^2-\|x-\bar\mu_{a^*_\ell}\|^2
\right)
\right), 1\right\}\right] \\
&\leq \exp\left(-\frac{\delta_{ab}^k}{2}\right) + 2\exp\left(-\frac{(\delta_{ab}^k)^2}{2\Delta_{ab}^2}
\right),
   \end{align*}
   where $\delta_{ab}^k
\eqdef
\frac{1}{2}\|\mu_\ell^k-\bar\mu_b\|^2
-
\frac{1}{2}\|\mu_\ell^k-\bar\mu_{a^*_\ell}\|^2$
and $
\Delta_{ab} \eqdef \|\bar\mu_{a^*_\ell}-\bar\mu_b\|.$
Using \eqref{eq:Delta1},
\begin{align*}
      &\mathbb E_{x \sim \mathcal N(\mu^k_\ell,\mI)} \bar{r}_b(x) 
\leq \exp\left(-\frac{\Delta_{ab}^2}{8}\right) + 2\exp\left(-\frac{\Delta_{ab}^2}{32}\right) \leq 3\exp\left(-\frac{\Delta_{ab}^2}{32}\right) \leq 3\exp\left(-\frac{\delta_{\min}^2}{32}\right)
   \end{align*}
and 
\begin{align*}
   \mathbb E_{x \sim \mathcal N(\mu^k_\ell,\mI)} \bar \tau_{a^*_\ell}(x) \leq 3 m \exp\left(-\frac{\delta_{\min}^2}{32}\right).
\end{align*}
Notice that the derived bounds are also valid for $k - 1, \dots, 0.$ Substituting the bounds into \eqref{eq:xdpFroas}, we get 
\begin{align*}
   &\norm{E_{\ell}^k} \leq \max_{j \in \{0, \dots, k\}} \mathbb E_{x \sim \mathcal N(\mu^j_\ell,\mI)} \left[\norm{\bar{\mathcal{E}}^j_{\ell}(x) + \mathcal{E}^j_{\ell}(x)}\right] \\
   &\leq 62 \max\{R, R^3\} \left(4 n \exp\left(-\hat\Delta_2 / 32\right) + 3 m \exp\left(-\delta_{\min}^2 / 32\right)\right) \\
   &\leq 434 \max\{R, R^3\} \max\left\{n \exp\left(-\hat\Delta_2 / 32\right), m \exp\left(-\delta_{\min}^2 / 32\right)\right\}.
\end{align*}
Due to $R \geq 1,$ $\hat\Delta_2 \geq 32 \log(434 n R^3 / E_1)$ (see \eqref{eq:tdbKKNEBxox}), and $\delta_{\min} \geq \sqrt{32 \log\left(434 R^3 m / E_1\right)}$ (see \eqref{eq:tdbKKNEBxox}). Thus
\begin{align}
   \label{eq:KBsuZNNToRYMmnxS}
   \norm{E_{\ell}^k} \leq \max_{j \in \{0, \dots, k\}} \mathbb E_{x \sim \mathcal N(\mu^j_\ell,\mI)} \left[\norm{\bar{\mathcal{E}}^j_{\ell}(x) + \mathcal{E}^j_{\ell}(x)}\right] \leq E_1.
\end{align}
   Using \eqref{eq:rwPVzXoMnmQzZXKel} and the choice of $\gamma$ and $R$, $\norm{\bar \mu_{a^*_\ell}} \leq \bar{R} \leq \frac{R}{2},$ $\norm{\mu^0_{\ell}} = R^0 \leq \frac{R}{2},$ we get
   \begin{align*}
      \norm{\mu^{k+1}_{\ell}} \leq \max\left\{\norm{\bar \mu_{a^*_\ell}}, \norm{\mu^0_{\ell}}\right\} + \norm{E_{\ell}^k} \leq \frac{R}{2} + E_1 \leq R
   \end{align*}
   since $E_1 \leq \frac{R}{2}$ (see \eqref{eq:tdbKKNEBxox}).
   For all $i, j \in [n],$ using \eqref{eq:rwPVzXoMnmQzZXKel} and \eqref{eq:KBsuZNNToRYMmnxS},
   \begin{align*}
      &\norm{\mu^{k+1}_{i} - \mu^{k+1}_{j}} \\
      &=\norm{\bar \mu_{a^*_i} - \bar \mu_{a^*_j} + \left(1 - \frac{2 \gamma}{n}\right)^{k+1} (\mu^0_{i} - \mu^0_{j} - (\bar \mu_{a^*_i} - \bar \mu_{a^*_j})) - E_{i}^k + E_{j}^k} \\
      &\geq \norm{\bar \mu_{a^*_i} - \bar \mu_{a^*_j} + \left(1 - \frac{2 \gamma}{n}\right)^{k+1} (\mu^0_{i} - \mu^0_{j} - (\bar \mu_{a^*_i} - \bar \mu_{a^*_j}))} - 2 E_1.
   \end{align*}
   If $a^*_i = a^*_j,$ then
   \begin{align*}
      \norm{\mu^{k+1}_{i} - \mu^{k+1}_{j}} \geq \left(1 - \frac{2 \gamma}{n}\right)^{k+1}\norm{\mu^0_{i} - \mu^0_{j}} - 2 E_1 \geq \frac{4 \sqrt{\bar{\Delta} \Delta_2}}{R^0} - 2 E_1,
   \end{align*}
   where we use \eqref{eq:BVNUOpDnYoTieBnA} and \eqref{eq:EvPlrGmTnwYCRQa}. If $a^*_i \neq a^*_j,$ then
   \begin{align*}
      &\norm{\mu^{k+1}_{i} - \mu^{k+1}_{j}} \\
      &\geq \norm{\bar \mu_{a^*_i} - \bar \mu_{a^*_j} + \left(1 - \frac{2 \gamma}{n}\right)^{k+1} (\mu^0_{i} - \mu^0_{j} - (\bar \mu_{a^*_i} - \bar \mu_{a^*_j}))} - 2 E_1.
   \end{align*}
   Notice that $\bar \mu_{a^*_i} + \left(1 - \frac{2 \gamma}{n}\right)^{k+1} (\mu^0_{i} - \bar \mu_{a^*_i}) \in [\mu^0_i, \bar{\mu}_{a^*_i}]$ for all $i \in [n].$ Using \eqref{eq:yZUzuDZUiFFPhgyiwgy}, we get
   \begin{align*}
      &\norm{\mu^{k+1}_{i} - \mu^{k+1}_{j}} \geq \Delta_3 - 2 E_1.
   \end{align*}
   In total,
   \begin{align*}
      &\norm{\mu^{k+1}_{i} - \mu^{k+1}_{j}} \geq \min\left\{\frac{4 \sqrt{\bar{\Delta} \Delta_2}}{R^0}, \Delta_3\right\} - 2 E_1
   \end{align*}
   and 
   \begin{align*}
      \norm{\mu^{k+1}_{i} - \mu^{k+1}_{j}}^2 \geq \hat\Delta_2
   \end{align*}
   due to $\frac{4 \sqrt{\bar{\Delta} \Delta_2}}{R^0} \geq 2 \sqrt{\hat\Delta_2}, \Delta_3 \geq 2 \sqrt{\hat\Delta_2}, E_1 \leq \frac{\sqrt{\hat\Delta_2}}{2}$ (see \eqref{eq:tdbKKNEBxox}).
   It remains to prove \eqref{eq:Delta1} for the step $k \to k + 1.$ Indeed, for all $b \neq a^*_\ell,$
\begin{align*}
D_{\ell,ab}^{k+1}
&\eqdef
\|\mu_\ell^{k+1}-\bar\mu_b\|^2
-
\|\mu_\ell^{k+1}-\bar\mu_{a^*_\ell}\|^2 \\
&=
\|\mu_\ell^{k} - \gamma \nabla_{\mu^k_\ell}\mathcal L_{\mathrm{rev}}(\mu^k) - \bar\mu_b\|^2
-
\|\mu_\ell^{k} - \gamma \nabla_{\mu^k_\ell}\mathcal L_{\mathrm{rev}}(\mu^k) - \bar\mu_{a^*_\ell}\|^2 \\
&=
D_{\ell,ab}^{k}
+
2\gamma
\left\langle
\nabla_{\mu^k_\ell}\mathcal L_{\mathrm{rev}}(\mu^k),
\bar\mu_b-\bar\mu_{a^*_\ell}
\right\rangle \\
&=
D_{\ell,ab}^{k}
+
\frac{4\gamma}{n}
\mathbb E_{x \sim \mathcal N(\mu^k_\ell,\mI)}
\left[
\left\langle
\mu^k_\ell-\bar\mu_{a^*_\ell}
+
\bar{\mathcal{E}}^k_{\ell}(x) + \mathcal{E}^k_{\ell}(x),
\bar\mu_b-\bar\mu_{a^*_\ell}
\right\rangle
\right] \\
&=
D_{\ell,ab}^{k}
+
\frac{4\gamma}{n}
\left\langle
\mu^k_\ell-\bar\mu_{a^*_\ell},
\bar\mu_b-\bar\mu_{a^*_\ell}
\right\rangle
+
\frac{4\gamma}{n}
\mathbb E_{x \sim \mathcal N(\mu^k_\ell,\mI)}
\left[
\left\langle
\bar{\mathcal{E}}^k_{\ell}(x) + \mathcal{E}^k_{\ell}(x),
\bar\mu_b-\bar\mu_{a^*_\ell}
\right\rangle
\right] \\
&=
D_{\ell,ab}^{k}
+
\frac{2\gamma}{n}
\left(
\|\bar\mu_b-\bar\mu_{a^*_\ell}\|^2
-
D_{\ell,ab}^{k}
\right)
+
\frac{4\gamma}{n}
\mathbb E_{x \sim \mathcal N(\mu^k_\ell,\mI)}
\left[
\left\langle
\bar{\mathcal{E}}^k_{\ell}(x) + \mathcal{E}^k_{\ell}(x),
\bar\mu_b-\bar\mu_{a^*_\ell}
\right\rangle
\right] \\
&=
\left(1-\frac{2\gamma}{n}\right)
D_{\ell,ab}^{k}
+
\frac{2\gamma}{n}
\|\bar\mu_b-\bar\mu_{a^*_\ell}\|^2
+
\frac{4\gamma}{n}
\mathbb E_{x \sim \mathcal N(\mu^k_\ell,\mI)}
\left[
\left\langle
\bar{\mathcal{E}}^k_{\ell}(x) + \mathcal{E}^k_{\ell}(x),
\bar\mu_b-\bar\mu_{a^*_\ell}
\right\rangle
\right].
\end{align*}
Thus 
\begin{align*}
   D_{\ell,ab}^{k+1} - \|\bar\mu_b-\bar\mu_{a^*_\ell}\|^2 = \left(1-\frac{2\gamma}{n}\right)
\left(D_{\ell,ab}^{k} - \|\bar\mu_b-\bar\mu_{a^*_\ell}\|^2\right)
+
\frac{4\gamma}{n}
\mathbb E_{x \sim \mathcal N(\mu^k_\ell,\mI)}
\left[
\left\langle
\bar{\mathcal{E}}^k_{\ell}(x) + \mathcal{E}^k_{\ell}(x),
\bar\mu_b-\bar\mu_{a^*_\ell}
\right\rangle
\right]
\end{align*}
and 
\begin{align*}
   &D_{\ell,ab}^{k+1} - \|\bar\mu_b-\bar\mu_{a^*_\ell}\|^2 = \left(1-\frac{2\gamma}{n}\right)^{k+1}
\left(D_{\ell,ab}^{0} - \|\bar\mu_b-\bar\mu_{a^*_\ell}\|^2\right) \\
&\quad + \frac{4\gamma}{n}
\sum_{j=0}^{k} \left(1 - \frac{2 \gamma}{n}\right)^{k - j} \mathbb E_{x \sim \mathcal N(\mu^j_\ell,\mI)}
\left[
\left\langle
\bar{\mathcal{E}}^j_{\ell}(x) + \mathcal{E}^j_{\ell}(x),
\bar\mu_b-\bar\mu_{a^*_\ell}
\right\rangle
\right].
\end{align*}
Notice that
\begin{align*}
   D_{\ell,ab}^{0}
\eqdef
\|\mu_\ell^{0}-\bar\mu_b\|^2
-
\|\mu_\ell^{0}-\bar\mu_{a^*_\ell}\|^2 \geq \Delta_1 \geq \delta^2_{\max} \geq \|\bar\mu_b-\bar\mu_{a^*_\ell}\|^2
\end{align*}
due to \eqref{eq:zolMqjeooMG} and $\Delta_1 \geq \delta^2_{\max}$ (see \eqref{eq:tdbKKNEBxox}). Using this inequality,
\begin{align*}
D_{\ell,ab}^{k+1} - \|\bar\mu_b-\bar\mu_{a^*_\ell}\|^2 
&\geq \frac{4\gamma}{n}
\sum_{j=0}^{k} \left(1 - \frac{2 \gamma}{n}\right)^{k - j} \mathbb E_{x \sim \mathcal N(\mu^j_\ell,\mI)}
\left[
\left\langle
\bar{\mathcal{E}}^j_{\ell}(x) + \mathcal{E}^j_{\ell}(x),
\bar\mu_b-\bar\mu_{a^*_\ell}
\right\rangle
\right] \\
& \geq - \frac{4\gamma}{n}
\sum_{j=0}^{k} \left(1 - \frac{2 \gamma}{n}\right)^{k - j} 2 R \mathbb E_{x \sim \mathcal N(\mu^j_\ell,\mI)}
\left(\norm{\bar{\mathcal{E}}^j_{\ell}(x) + \mathcal{E}^j_{\ell}(x)}\right) \\
& \geq - 4 R \max_{j \in \{0, \dots, k\}} \mathbb E_{x \sim \mathcal N(\mu^j_\ell,\mI)}
\left(\norm{\bar{\mathcal{E}}^j_{\ell}(x) + \mathcal{E}^j_{\ell}(x)}\right) \\
& \overset{\eqref{eq:KBsuZNNToRYMmnxS}}{\geq} - 4 R E_1.
\end{align*}
Due to \eqref{eq:kjdsqYLZ} and $\delta_{\min} \geq 1$, $4 R E_1 \leq \frac{\delta^2_{\min}}{2}$ and 
\begin{align*}
   \|\mu_\ell^{k+1}-\bar\mu_b\|^2 - \|\mu_\ell^{k+1}-\bar\mu_{a^*_\ell}\|^2 \geq \frac{\|\bar\mu_b-\bar\mu_{a^*_\ell}\|^2}{2}
\end{align*}
for all $b \neq a^*_{\ell},$ and we have proved the mathematical induction for $k \to k + 1.$

Using \eqref{eq:rwPVzXoMnmQzZXKel}, \eqref{eq:TUyQcvjVGXjffeyd}, and \eqref{eq:KBsuZNNToRYMmnxS},
\begin{align*}
   \norm{\mu^{\bar{K}}_{\ell} - \bar \mu_{a^*_\ell}} 
   &\leq \left(1 - \frac{2 \gamma}{n}\right)^{\bar{K}} \norm{\mu^0_{\ell} - \bar \mu_{a^*_\ell}} + \norm{E_{\ell}^{\bar{K} - 1}} \\
   &\leq \frac{8 \sqrt{\bar{\Delta}}}{R^0} \norm{\mu^0_{\ell} - \bar \mu_{a^*_\ell}} + E_1 \leq 16 \sqrt{\bar{\Delta}} + E_1
\end{align*}
because $\norm{\mu^0_{\ell}} = R^0$ and $\norm{\bar \mu_{a^*_\ell}} \leq R^0.$ Due to \eqref{eq:tdbKKNEBxox} and \eqref{eq:kjdsqYLZ},
\begin{align}
   \label{eq:BUOoEmxScVtFXWFu}
   \norm{\mu^{\bar{K}}_{\ell} - \bar \mu_{a^*_\ell}} \leq 16 \sqrt{\bar{\Delta}} + E_1 \leq 17 \sqrt{\bar{\Delta}}
\end{align}
for all $\ell \in [n].$ 
Let $I_a = \{i_{a,1}, \dots, i_{a,n_a}\}.$ The inequality ensures that after $\bar{K}$ steps, each component converges to a small-$\bar{\Delta}$ neighborhood of the corresponding target component.

\textbf{(Local convergence stage).} The proof now would be similar to that of Theorem~\ref{thm:main_one}, but with careful and technical error control.
For all $a \in A^*,$ we define 
\begin{align*}
   \mZ^{k}_a &\eqdef
   \begin{bmatrix}
      (\mu^{k}_{i_{a,1}} - \bar \mu_{a})^\top \\
      \vdots \\
      (\mu^{k}_{i_{a,n_a}} - \bar \mu_{a})^\top \\
   \end{bmatrix}.
\end{align*}
Then,
\begin{align*}
   \mZ^{\bar{K}}_a &=
   \begin{bmatrix}
      (\mu^{\bar{K}}_{i_{a,1}} - \bar \mu_{a})^\top \\
      \vdots \\
      (\mu^{\bar{K}}_{i_{a,n_a}} - \bar \mu_{a})^\top \\
   \end{bmatrix} 
   =
   \begin{bmatrix}
      \left(\left(1 - \frac{2 \gamma}{n}\right)^{\bar{K}} (\mu^0_{i_{a,1}} - \bar \mu_{a}) - E_{i_{a,1}}^{\bar{K} - 1}\right)^\top \\
      \vdots \\
      \left(\left(1 - \frac{2 \gamma}{n}\right)^{\bar{K}} (\mu^0_{i_{a,n_a}} - \bar \mu_{a}) - E_{i_{a,n_a}}^{\bar{K} - 1}\right)^\top \\
   \end{bmatrix} \\
   &=
   \left(1 - \frac{2 \gamma}{n}\right)^{\bar{K}} \begin{bmatrix}
      \left(\mu^0_{i_{a,1}}\right)^\top \\
      \vdots \\
      \left(\mu^0_{i_{a,n_a}}\right)^\top \\
   \end{bmatrix}
   -
   \left(1 - \frac{2 \gamma}{n}\right)^{\bar{K}} \begin{bmatrix}
      \left(\bar \mu_{a}\right)^\top \\
      \vdots \\
      \left(\bar \mu_{a}\right)^\top \\
   \end{bmatrix}
   -
   \begin{bmatrix}
      \left(E_{i_{a,1}}^{\bar{K} - 1}\right)^\top \\
      \vdots \\
      \left(E_{i_{a,n_a}}^{\bar{K} - 1}\right)^\top \\
   \end{bmatrix}.
\end{align*}
Thus,
\begin{align*}
   \norm{\mZ^{\bar{K}}_a}_{\textnormal{op}} 
   &\leq \left(1 - \frac{2 \gamma}{n}\right)^{\bar{K}} \norm{\begin{bmatrix}
      \left(\mu^0_{i_{a,1}}\right)^\top \\
      \vdots \\
      \left(\mu^0_{i_{a,n_a}}\right)^\top \\
   \end{bmatrix}}_{\textnormal{op}}
   +
   \left(1 - \frac{2 \gamma}{n}\right)^{\bar{K}} \norm{\begin{bmatrix}
      \left(\bar \mu_{a}\right)^\top \\
      \vdots \\
      \left(\bar \mu_{a}\right)^\top \\
   \end{bmatrix}}_{\textnormal{op}}
   +
   \norm{\begin{bmatrix}
      \left(E_{i_{a,1}}^{\bar{K} - 1}\right)^\top \\
      \vdots \\
      \left(E_{i_{a,n_a}}^{\bar{K} - 1}\right)^\top \\
   \end{bmatrix}}_{\textnormal{op}} \\
   &\leq \left(1 - \frac{2 \gamma}{n}\right)^{\bar{K}} \norm{\begin{bmatrix}
      \left(\mu^0_{i_{a,1}}\right)^\top \\
      \vdots \\
      \left(\mu^0_{i_{a,n_a}}\right)^\top \\
   \end{bmatrix}}_{\textnormal{op}}
   +
   \sqrt{n_a} \left(1 - \frac{2 \gamma}{n}\right)^{\bar{K}} \norm{\bar \mu_{a}}
   +
   \sqrt{n_a} E_1
\end{align*}
Using the choice of $\bar{K},$ $\norm{\bar \mu_{a}} \leq \bar{R},$ and \eqref{eq:KfgolNiUc},
\begin{equation}
\label{eq:XgWawiQIXsOKPBdfn}
\begin{aligned}
   \norm{\mZ^{\bar{K}}_a}_{\textnormal{op}}
   &\leq \frac{8 \sqrt{\bar{\Delta}} \Delta_4}{R^0} + \frac{8 \sqrt{n_a \bar{\Delta}} \bar{R}}{R^0} + \sqrt{n_a} E_1 \\
   &\leq \bar{z} \eqdef \frac{8 \sqrt{\bar{\Delta}} \Delta_4}{R^0} + \frac{8 \sqrt{n \bar{\Delta}} \bar{R}}{R^0} + \sqrt{n} E_1
\end{aligned}
\end{equation}
for all $a \in A^*.$ For all $a \in A^*,$ consider the subgroup $I_a.$ Due to \eqref{eq:grad_1}, for all $\ell \in I_a,$
\begin{align*}
   \mu^{k+1}_{\ell} - \bar\mu_{a}
   &= \mu^{k}_{\ell} - \bar\mu_{a} - \gamma \nabla_{\mu^k_\ell}\mathcal L_{\mathrm{rev}}(\mu^k) \\
   &= \mu^{k}_{\ell} - \bar\mu_{a} - \frac{2 \gamma}{n} \mathbb E_{x \sim \mathcal N(\mu^k_\ell,\mI)} \left[m^k_{a}(x) - \bar\mu_{a} + \mC^k_{a}(x)\bigl(\mu^k_\ell-m^k_{a}(x)\bigr)\right] - \frac{2 \gamma}{n} \mathbb E_{x \sim \mathcal N(\mu^k_\ell,\mI)} \left[\mathcal{E}^k_{\ell}(x)\right] \\
   &= \mu^{k}_{\ell} - \bar\mu_{a} \\
   &\quad - \frac{2 \gamma}{n} \mathbb E_{x \sim \mathcal N(\mu^k_\ell,\mI)} \left[\sum_{i\in I_{a}} r_i^{k,a}(x)(\mu^k_i - \bar\mu_{a}) + \sum_{i\in I_{a}} r_i^{k,a}(x) \bigl(\mu^k_i-m^k_{a}(x)\bigr) \bigl(\mu^k_i-m^k_{a}(x)\bigr)^\top \bigl(\mu^k_\ell-m^k_{a}(x)\bigr)\right] \\
   &\quad - \frac{2 \gamma}{n} \mathbb E_{x \sim \mathcal N(\mu^k_\ell,\mI)} \left[\mathcal{E}^k_{\ell}(x)\right].
\end{align*}
Since $\sum_{i\in I_{a}} r_i^{k,a}(x) b \bigl(\mu^k_i-m^k_{a}(x)\bigr)^\top = 0$ for all $b \in \R^d,$
\begin{align*}
   \mu^{k+1}_{\ell} - \bar\mu_{a}
   &= \mu^{k}_{\ell} - \bar\mu_{a} \\
   &\quad - \frac{2 \gamma}{n} \mathbb E_{x \sim \mathcal N(\mu^k_\ell,\mI)} \left[\sum_{i\in I_{a}} r_i^{k,a}(x)(\mu^k_i - \bar\mu_{a}) + \sum_{i\in I_{a}} r_i^{k,a}(x) \bigl(\mu^k_i - \bar\mu_{a}\bigr) \bigl(\mu^k_i - m^k_{a}(x)\bigr)^\top \bigl(\mu^k_\ell-m^k_{a}(x)\bigr)\right] \\
   &\quad - \frac{2 \gamma}{n} \mathbb E_{x \sim \mathcal N(\mu^k_\ell,\mI)} \left[\mathcal{E}^k_{\ell}(x)\right] \\
   &= \mu^{k}_{\ell} - \bar\mu_{a} \\
   &\quad - \frac{2 \gamma}{n}  \sum_{i\in I_{a}} \mathbb E_{x \sim \mathcal N(\mu^k_\ell,\mI)} \left[r_i^{k,a}(x) + r_i^{k,a}(x) \bigl(\mu^k_i - m^k_{a}(x)\bigr)^\top \bigl(\mu^k_\ell-m^k_{a}(x)\bigr) \right] \bigl(\mu^k_i - \bar\mu_{a}\bigr) \\
   &\quad - \frac{2 \gamma}{n} \mathbb E_{x \sim \mathcal N(\mu^k_\ell,\mI)} \left[\mathcal{E}^k_{\ell}(x)\right].
\end{align*}
We can conclude that 
\begin{align*}
   \mu^{k+1}_{\ell} - \bar\mu_{a}
   &= \mu^{k}_{\ell} - \bar\mu_{a} - \frac{2 \gamma}{n} \sum_{i \in I_{a}} \mB^{k}_{a,\ell i} (\mu^{k}_{i} - \bar\mu_{a}) - \frac{2 \gamma}{n} \mathbb E_{x \sim \mathcal N(\mu^k_\ell,\mI)} \left[\mathcal{E}^k_{\ell}(x)\right],
\end{align*}
and 
\begin{align}
   \label{eq:BQUgXvTShGIDzvYJH}
   \mZ^{k+1}_a = \left(\mI - \frac{2 \gamma}{n} \mB_a^k\right) \mZ^{k}_a - \frac{2 \gamma}{n} \mE^k_a,
\end{align}
where $\mE^k_a = \left[\mathbb E_{x \sim \mathcal N(\mu^k_{i_{a,1}},\mI)} \left[\mathcal{E}^k_{i_{a,1}}(x)\right]^\top; \dots; \mathbb E_{x \sim \mathcal N(\mu^k_{i_{a,n_a}},\mI)} \left[\mathcal{E}^k_{i_{a,n_a}}(x)\right]^\top\right]$ and $\mB^{k}_a \in \R^{n_{a} \times n_{a}}$ such that
\begin{equation}
\label{eq:aegiybPjoBoRKeigLO}
\begin{aligned}
   \mB^{k}_{a, \ell j} 
   &\eqdef \mathbb E_{x \sim \mathcal N(\mu^k_{i_{a,\ell}},\mI)}
   \left[r_{i_{a,j}}^{k,a}(x) + r_{i_{a,j}}^{k,a}(x) \bigl(\mu^k_{i_{a,j}}-m^k_{a}(x)\bigr)^\top \bigl(\mu^k_{i_{a,\ell}}-m^k_{a}(x)\bigr)\right] \\
   &=\int_{x \in \R^d}
   \left[r_{i_{a,j}}^{k,a}(x) + r_{i_{a,j}}^{k,a}(x) \bigl(\mu^k_{i_{a,j}}-m^k_{a}(x)\bigr)^\top \bigl(\mu^k_{i_{a,\ell}}-m^k_{a}(x)\bigr)\right] \varphi(x - \mu^k_{i_{a,\ell}}) \, dx\\
   &=n_{a} \int_{x \in \R^d}
   \left[r_{i_{a,j}}^{k,a}(x) r^{k,a}_{{i_{a,\ell}}} (x) + r_{i_{a,j}}^{k,a}(x) r^{k,a}_{{i_{a,\ell}}} (x) \bigl(\mu^k_{i_{a,j}}-m^k_{a}(x)\bigr)^\top \bigl(\mu^k_{i_{a,\ell}}-m^k_{a}(x)\bigr)\right] p^{k,a}_{\mu}(x) \, dx \\
   &=n_{a} \mathbb E_{x \sim p^{k,a}_{\mu}}
   \left[r_{i_{a,j}}^{k,a}(x) r^{k,a}_{{i_{a,\ell}}} (x) + r_{i_{a,j}}^{k,a}(x) r^{k,a}_{{i_{a,\ell}}} (x) \bigl(\mu^k_{i_{a,j}}-m^k_{a}(x)\bigr)^\top \bigl(\mu^k_{i_{a,\ell}}-m^k_{a}(x)\bigr)\right],
\end{aligned}
\end{equation}
where $p^{k,a}_{\mu}(x) \eqdef \frac{1}{n_{a}} \sum_{i \in I_{a}} \varphi(x - \mu^k_i).$ Thus, $\mB^k_a$ is symmetric and non-negative definite. Let us define $R^k_a \eqdef \max_{i \in I_a} \norm{\mu^k_i - \bar\mu_a}$ for all $a \in A^*.$ For all $v \in \R^{n_{a}},$
\begin{align*}
   v^\top \mB_a^k v 
   &= n_{a} \int_{x \in \R^d}
   \left[\left(\sum_{j=1}^{n_{a}} r_{i_{a,j}}^{k,a}(x) v_j\right)^2 + \norm{\sum_{j=1}^{n_{a}} r_{i_{a,j}}^{k,a}(x) v_j (\mu^k_{i_{a,j}}-m^k_{a}(x))}^2\right] p^{k,a}_{\mu}(x) \, dx \\
   &\leq n_{a} \int_{x \in \R^d}
   \left[\sum_{j=1}^{n_{a}} r_{i_{a,j}}^{k,a}(x) v_j^2 + \sum_{j=1}^{n_{a}} r_{i_{a,j}}^{k,a}(x) v_j^2 \norm{\mu^k_{i_{a,j}}-m^k_{a}(x)}^2\right] p^{k,a}_{\mu}(x) \, dx \\
   &=
   \sum_{j=1}^{n_{a}} \left[ \int_{x \in \R^d} n_{a} p^{k,a}_{\mu}(x) r_{i_{a,j}}^{k,a}(x) \, dx\right] v_j^2 + \sum_{j=1}^{n_{a}} \left[ \int_{x \in \R^d} \norm{\mu^k_{i_{a,j}}-m^k_{a}(x)}^2 n_{a} p^{k,a}_{\mu}(x) r_{i_{a,j}}^{k,a}(x) \, dx \right] v_j^2 \\
   &\leq 
   \sum_{j=1}^{n_{a}} v_j^2 + 4 (R_a^k)^2 \sum_{j=1}^{n_{a}} v_j^2 = (1 + 4 (R_a^k)^2) \norm{v}^2
   \end{align*}
   since $n_{a} p^{k,a}_{\mu}(x) r_{i_{a,j}}^{k,a}(x) = \varphi(x - \mu^k_{i_{a,j}}),$ and
   \[
      \norm{\mu^k_{i_{a,j}} - m^k_{a}(x)}^2 = \norm{\mu^k_{i_{a,j}} - \bar\mu_a - (m^k_{a}(x) - \bar\mu_a)}^2 \leq 4 (R_a^k)^2.
   \]
Thus,
\begin{align*}
   \norm{\mB_a^k}_{\textnormal{op}} \leq 1 + 4 (R_a^k)^2
\end{align*}
for all $k \geq 0$ and $a \in A^*.$ Using mathematical induction, for all $\bar{K} \leq k \leq \hat{K},$ we assume that
\begin{align}
   &\norm{\mZ^{k}_a}_{\textnormal{op}} \leq 2 \bar{z}, \label{eq:Z} 
\end{align}
for all $a \in A^*,$
where
\begin{align}
   \label{eq:wkBVdTRSfZPDiUDTm}
   \hat{K} \eqdef \bar{K} + \ceil{\frac{n \min\{n, d\} \max\left\{4 n, 8 \bar{z}^2\right\}}{\gamma \varepsilon}}.
\end{align}
We also define 
\begin{align}
   \label{eq:gzBeKKVwmhdsVvkWWVK}
   E_2 \eqdef \min\left\{\frac{n \bar{z}}{2 \gamma \hat{K}}, \frac{\varepsilon^2}{16 \min\{n, d\} \sqrt{\min\{n, d\}} \bar{z}} \min\left\{\frac{1}{4 n}, \frac{1}{8 \bar{z}^2}\right\}, \sqrt{\frac{n \varepsilon^2}{8 \gamma \min\{n, d\}} \min\left\{\frac{1}{4 n}, \frac{1}{8 \bar{z}^2}\right\}}\right\}.
\end{align}
We consider the following lemma, proved in Section~\ref{sec:auxlemmas}.
\begin{restatable}{lemma}{LEMMAPARAMSTWO}
For our choice of parameters, 
\begin{equation}
\label{eq:VheRdmk}
\begin{aligned}
   &\delta_{\min} \geq 8 \bar{z}, \quad \bar{z} \geq 1,\\
   &\bar{z} \leq \tilde{z} \eqdef \hatz,\\
   &\delta_{\min} \geq \sqrt{128 \log\left(\frac{156 \sqrt{n} R^3 \max\left\{n, m\right\}}{E_2}\right)}, \\
   &\hat{K} \leq \tilde{K} \eqdef \hatK, \\
   &R \geq 2 \bar{z} + \bar{R}.
\end{aligned}
\end{equation}
\end{restatable}
For $k = \bar{K}$, the induction hypothesis holds due to \eqref{eq:XgWawiQIXsOKPBdfn}. We now prove the next step of the induction.
Due to \eqref{eq:Z} and $R \geq 2 \bar{z} + \bar{R}$ (see \eqref{eq:VheRdmk}),
\begin{align*}
   \norm{\mu^k_{\ell} - \bar\mu_{a}} \leq 2 \bar{z}
\end{align*}
and 
\begin{align}
   \label{eq:BYoUQZrvqr}
   \norm{\mu^k_{\ell}} \leq R
\end{align}
for all $a \in A^*$ and $\ell \in I_a.$ Thus, for all $i \not\in I_a,$
\begin{equation}
\label{eq:eCdJeeUtKACLbQPRWOi}
\begin{aligned}
   &\norm{\mu^k_{\ell} - \mu^k_i} = \norm{(\bar\mu_{a^*_{\ell}} - \bar\mu_{a^*_i}) + (\mu^k_{\ell} - \bar\mu_{a^*_{\ell}}) - (\mu^k_i - \bar\mu_{a^*_i})} \\
   &\geq \norm{\bar\mu_{a^*_{\ell}} - \bar\mu_{a^*_i}} - \norm{\mu^k_{\ell} - \bar\mu_{a^*_{\ell}}} - \norm{\mu^k_i - \bar\mu_{a^*_i}} 
   \geq \delta_{\min} - 4 \bar{z} \geq \frac{\delta_{\min}}{2}
\end{aligned}
\end{equation}
since $\delta_{\min} \geq 8 \bar{z}$ (see \eqref{eq:VheRdmk}).
Similarly, for all $a \neq b \in [m]$ and for all $\ell \in I_a,$
\begin{equation}
\label{eq:WXZaIXLfmWzKu}
\begin{aligned}
   &\frac{1}{2}\|\mu_\ell^k-\bar\mu_b\|^2 - \frac{1}{2}\|\mu_\ell^k-\bar\mu_{a}\|^2 =\frac{1}{2}\|\mu_\ell^k - \bar\mu_a - (\bar\mu_b - \bar\mu_a)\|^2 - \frac{1}{2}\|\mu_\ell^k-\bar\mu_{a}\|^2 \\
   &=-\inp{\mu_\ell^k - \bar\mu_a}{\bar\mu_b - \bar\mu_a} + \frac{1}{2}\|\bar\mu_b - \bar\mu_a\|^2 \geq \frac{1}{2}\|\bar\mu_b - \bar\mu_a\|^2 - \norm{\mu_\ell^k - \bar\mu_a} \norm{\bar\mu_b - \bar\mu_a} \\
   &\geq \frac{1}{4}\|\bar\mu_b - \bar\mu_a\|^2.
\end{aligned}
\end{equation}
because $\frac{1}{4}\|\bar\mu_b - \bar\mu_a\| \geq \frac{1}{4} \delta_{\min} \geq 2 \bar{z} \geq \norm{\mu_\ell^k - \bar\mu_a}.$ 
For all $a \in A^*,$
\begin{equation}
\label{eq:vjwBAjebeuMpvpliqDeV}
\begin{aligned}
   &\norm{\mE^k_a}_{F} \leq \sqrt{n_a} \max_{\ell \in I_a} \norm{\mathbb E_{x \sim \mathcal N(\mu^k_\ell,\mI)} \left[\mathcal{E}^k_{\ell}(x)\right]} \leq \sqrt{n_a} \max_{\ell \in I_a} \mathbb E_{x \sim \mathcal N(\mu^k_\ell,\mI)} \left[\norm{\mathcal{E}^k_{\ell}(x)}\right] \\
   &\leq 52 \sqrt{n} \max\{R, R^3\} \max_{\ell \in I_a} \left(\mathbb E_{x \sim \mathcal N(\mu^k_\ell,\mI)} \hat\tau^k_{a}(x) +  \mathbb E_{x \sim \mathcal N(\mu^k_\ell,\mI)} \bar \tau_{a}(x)\right)
\end{aligned}
\end{equation}
where we use Lemma~\ref{lemma:bound1}, \eqref{eq:BYoUQZrvqr}, and $n_a \leq n.$ Using Lemma~\ref{lemma:dist_1},
   \begin{align*}
      &\mathbb E_{x \sim \mathcal N(\mu^k_\ell,\mI)} \hat\tau^k_{a}(x) = \sum_{i \in [n], i \not\in I_a} \mathbb E_{x \sim \mathcal N(\mu^k_\ell,\mI)} r_i^{k}(x) \\
      &\leq \sum_{i \in [n], i \not\in I_a} \mathbb E_{x \sim \mathcal N(\mu^k_\ell,\mI)} \left[\min\left\{\exp\left(\frac12\|x-\mu^k_{\ell}\|^2-\frac12\|x-\mu^k_i\|^2\right), 1\right\}\right] \\
      &\leq 2 \sum_{i \in [n], i \not\in I_a} \exp\left(-\frac{1}{32} \norm{\mu^k_{\ell} - \mu^k_{i}}^2\right) \overset{\eqref{eq:eCdJeeUtKACLbQPRWOi}}{\leq} 2 n \exp\left(-\frac{\delta^2_{\min}}{128}\right).
   \end{align*}
   Next,
   \begin{align*}
      \mathbb E_{x \sim \mathcal N(\mu^k_\ell,\mI)} \bar \tau_{a}(x) = \sum_{b \in [m] , b \neq a} \mathbb E_{x \sim \mathcal N(\mu^k_\ell,\mI)} \bar r_b(x).
   \end{align*}
   Due to Lemma~\ref{lemma:dist_2},
   \begin{align*}
      \mathbb E_{x \sim \mathcal N(\mu^k_\ell,\mI)} \bar{r}_b(x) 
      &\leq \mathbb E_{x \sim \mathcal N(\mu^k_\ell,\mI)} \left[\min\left\{\exp\left(
-\frac12
\left(
\|x-\bar\mu_b\|^2-\|x-\bar\mu_{a}\|^2
\right)
\right), 1\right\}\right] \\
&\leq \exp\left(-\frac{\delta_{ab}^k}{2}\right) + 2\exp\left(-\frac{(\delta_{ab}^k)^2}{2\Delta_{ab}^2}
\right),
   \end{align*}
   where $\delta_{ab}^k
\eqdef
\frac{1}{2}\|\mu_\ell^k-\bar\mu_b\|^2
-
\frac{1}{2}\|\mu_\ell^k-\bar\mu_{a}\|^2$
and $
\Delta_{ab} \eqdef \|\bar\mu_{a}-\bar\mu_b\|.$
Using \eqref{eq:WXZaIXLfmWzKu}, $\delta_{ab}^k \geq \frac{1}{4} \Delta_{ab}^2$ and 
\begin{align*}
      &\mathbb E_{x \sim \mathcal N(\mu^k_\ell,\mI)} \bar{r}_b(x) 
\leq \exp\left(-\frac{\Delta_{ab}^2}{8}\right) + 2\exp\left(-\frac{\Delta_{ab}^2}{32}\right) \leq 3\exp\left(-\frac{\Delta_{ab}^2}{32}\right) \leq 3\exp\left(-\frac{\delta_{\min}^2}{32}\right).
   \end{align*}
Therefore,
\begin{align*}
   \mathbb E_{x \sim \mathcal N(\mu^k_\ell,\mI)} \bar \tau_{a}(x) \leq 3 m \exp\left(-\frac{\delta_{\min}^2}{32}\right).
\end{align*}
Substituting the derived bounds into \eqref{eq:vjwBAjebeuMpvpliqDeV},
\begin{equation}
\label{eq:SWEjGFxcW}
\begin{aligned}
   \norm{\mE^k_a}_{F} 
   &\leq 52 \sqrt{n} \max\{R, R^3\} \left(2 n \exp\left(-\frac{\delta^2_{\min}}{128}\right) + 3 m \exp\left(-\frac{\delta_{\min}^2}{32}\right)\right) \\
   &\leq 156 \sqrt{n} R^3 \max\left\{n, m\right\} \exp\left(-\frac{\delta^2_{\min}}{128}\right).
\end{aligned}
\end{equation}
Due to $R \geq 1$ and $\delta_{\min} \geq \sqrt{128 \log\left(\frac{156 \sqrt{n} R^3 \max\left\{n, m\right\}}{E_2}\right)}$ (see \eqref{eq:VheRdmk}),
\begin{align}
   \label{eq:ckwxZENYRIMy}
   \norm{\mE^k_a}_{\textnormal{op}} \leq \norm{\mE^k_a}_{F} \leq E_2
\end{align}
for all $a \in A^*.$ Note that \eqref{eq:ckwxZENYRIMy} is true for $k - 1, \dots, \bar{K}.$ Due to \eqref{eq:Z}, for all $a \in A^*$ and $i \in I_a,$ $\norm{\mu^k_i - \bar\mu_a} \leq 2 \bar{z},$ $R^k_a \leq 2 \bar{z},$ and $\norm{\mB_a^k}_{\textnormal{op}} \leq 1 + 16 \bar{z}^2.$ Using this and \eqref{eq:BQUgXvTShGIDzvYJH},
\begin{align*}
   \norm{\mZ^{k+1}_a}_{\textnormal{op}} 
   &= \norm{\left(\mI - \frac{2 \gamma}{n} \mB_a^k\right) \mZ^{k}_a - \frac{2 \gamma}{n} \mE^k_a}_{\textnormal{op}} \leq \norm{\mZ^{k}_a}_{\textnormal{op}} + \frac{2 \gamma}{n} \norm{\mE^k_a}_{\textnormal{op}}
\end{align*}
since $\gamma \leq \frac{n}{1 + 16 \bar{z}^2}.$ Unrolling the recursion and using \eqref{eq:XgWawiQIXsOKPBdfn} and \eqref{eq:ckwxZENYRIMy},
\begin{align*}
   \norm{\mZ^{k+1}_a}_{\textnormal{op}} 
   &\leq \bar{z} + \frac{2 \gamma \hat{K}}{n} E_2 \leq 2 \bar{z}
\end{align*}
since $k + 1 \leq \hat{K}$ and $E_2 \leq \frac{n \bar{z}}{2 \gamma \hat{K}}$ (see \eqref{eq:gzBeKKVwmhdsVvkWWVK}). Thus, we have proved the next step of the mathematical induction \eqref{eq:Z}. 

We now fix $a \in A^*.$ For all $\bar{K} \leq k \leq \hat{K} - 1,$ consider the update under the squared Frobenius norm:
\begin{align}
   \norm{\mZ^{k+1}_a}^2_{F} 
   &= \norm{\left(\mI - \frac{2 \gamma}{n} \mB_a^k\right) \mZ^{k}_a - \frac{2 \gamma}{n} \mE^k_a}^2_{F} \nonumber \\
   &= \norm{\left(\mI - \frac{2 \gamma}{n} \mB_a^k\right) \mZ^{k}_a}^2_{F} - \frac{4 \gamma}{n} \inp{\left(\mI - \frac{2 \gamma}{n} \mB_a^k\right) \mZ^{k}_a}{\mE^k_a} + \frac{4 \gamma^2}{n^2} \norm{\mE^k_a}^2_{F}. \label{eq:YRozNjjpHQiskNi}
\end{align}
Consider the first norm separately:
\begin{align}
   \norm{\left(\mI - \frac{2 \gamma}{n} \mB_a^k\right) \mZ^{k}_a}^2_{F} &=\norm{\mZ^{k}_a}^2_{F} - \frac{4 \gamma}{n} \inp{\mZ^{k}_a}{\mB_a^k \mZ^{k}_a} + \frac{4 \gamma^2}{n^2} \norm{\mB_a^k \mZ^{k}_a}^2_{F} \nonumber \\
   &\leq \norm{\mZ^{k}_a}^2_{F} - \frac{4 \gamma}{n} \inp{\mZ^{k}_a}{\mB_a^k \mZ^{k}_a} + \frac{4 \gamma^2}{n^2} \norm{\mB_a^k}_{\textnormal{op}} \inp{\mZ^{k}_a}{\mB_a^k \mZ^{k}_a} \nonumber \\
   &\leq \norm{\mZ^{k}_a}^2_{F} - \frac{4 \gamma}{n} \inp{\mZ^{k}_a}{\mB_a^k \mZ^{k}_a} + \frac{2 \gamma}{n} \inp{\mZ^{k}_a}{\mB_a^k \mZ^{k}_a} \nonumber \\
   &= \norm{\mZ^{k}_a}^2_{F} - \frac{2 \gamma}{n} \inp{\mZ^{k}_a}{\mB_a^k \mZ^{k}_a}. \label{eq:mheRAEqzBS}
\end{align}
since $\frac{\gamma}{n} \leq \frac{1}{2 (1 + 16 \bar{z}^2)} \leq \frac{1}{2 \norm{\mB_a^k}_{\textnormal{op}}}.$ We now use the following lemma separately proved in Section~\ref{sec:auxlemmas}.
\begin{restatable}{lemma}{LEMMAZ}
   For all $k \geq \bar{K}$ and $a \in A^*,$
   \begin{align*}
      \inp{\mZ^{k}_a}{\mB_a^k \mZ^{k}_a} \geq \min\left\{\frac{\kappa_a}{4 n_a^2} \norm{\mZ^{k}_a}^4_F, \frac{1}{2} \norm{\mZ^{k}_a}^2_F\right\},
   \end{align*}
   where $\kappa_a \eqdef \frac{n_a}{\min\{n_a, d\}}.$
\end{restatable}

Substituting this bound into \eqref{eq:YRozNjjpHQiskNi} and \eqref{eq:mheRAEqzBS},
\begin{align*}
   \norm{\mZ^{k+1}_a}^2_{F} 
   &\leq \norm{\mZ^{k}_a}^2_{F} - \frac{2 \gamma}{n} \min\left\{\frac{\kappa_a}{4 n_a^2} \norm{\mZ^{k}_a}^4_F, \frac{1}{2} \norm{\mZ^{k}_a}^2_F\right\} - \frac{4 \gamma}{n} \inp{\left(\mI - \frac{2 \gamma}{n} \mB_a^k\right) \mZ^{k}_a}{\mE^k_a} + \frac{4 \gamma^2}{n^2} \norm{\mE^k_a}^2_{F} \\
   &\leq \norm{\mZ^{k}_a}^2_{F} - \frac{2 \gamma}{n} \min\left\{\frac{\kappa_a}{4 n_a^2} \norm{\mZ^{k}_a}^4_F, \frac{1}{2} \norm{\mZ^{k}_a}^2_F\right\} + \frac{4 \gamma}{n} \norm{\left(\mI - \frac{2 \gamma}{n} \mB_a^k\right) \mZ^{k}_a}_F \norm{\mE^k_a}_F + \frac{4 \gamma^2}{n^2} \norm{\mE^k_a}^2_{F} \\
   &\leq \norm{\mZ^{k}_a}^2_{F} - \frac{2 \gamma}{n} \min\left\{\frac{\kappa_a}{4 n_a^2} \norm{\mZ^{k}_a}^4_F, \frac{1}{2} \norm{\mZ^{k}_a}^2_F\right\} + \frac{4 \gamma}{n} \norm{\mZ^{k}_a}_F \norm{\mE^k_a}_F + \frac{4 \gamma^2}{n^2} \norm{\mE^k_a}^2_{F},
\end{align*}
since $\norm{\left(\mI - \frac{2 \gamma}{n} \mB_a^k\right)}_{\textnormal{op}} \leq 1.$ Using \eqref{eq:Z}, $\norm{\mZ^{k}_a}_{F} \leq 2 \sqrt{\min\{n_a, d\}} \bar{z}$ and
\begin{align*}
   \norm{\mZ^{k+1}_a}^2_{F} 
   &\leq \norm{\mZ^{k}_a}^2_{F} - \frac{2 \gamma}{n} \min\left\{\frac{\kappa_a}{4 n_a^2} \norm{\mZ^{k}_a}^4_F, \frac{1}{2} \norm{\mZ^{k}_a}^2_F\right\} + \frac{8 \gamma \sqrt{\min\{n_a, d\}} \bar{z}}{n} \norm{\mE^k_a}_F + \frac{4 \gamma^2}{n^2} \norm{\mE^k_a}^2_{F} \\
   &\overset{\eqref{eq:ckwxZENYRIMy}}{\leq} \norm{\mZ^{k}_a}^2_{F} - \frac{2 \gamma}{n} \min\left\{\frac{\kappa_a}{4 n_a^2} \norm{\mZ^{k}_a}^4_F, \frac{1}{2} \norm{\mZ^{k}_a}^2_F\right\} + \frac{8 \gamma \sqrt{\min\{n, d\}} \bar{z} E_2}{n} + \frac{4 \gamma^2 E_2^2}{n^2}.
\end{align*}
Let us temporarily define $E_3 \eqdef 8 \sqrt{\min\{n, d\}} \bar{z} E_2 + \frac{4 \gamma E_2^2}{n}$ to get
\begin{align*}
   \norm{\mZ^{k+1}_a}^2_{F} 
   &\leq \norm{\mZ^{k}_a}^2_{F} - \frac{2 \gamma}{n} \min\left\{\frac{\kappa_a}{4 n_a^2} \norm{\mZ^{k}_a}^4_F, \frac{1}{2} \norm{\mZ^{k}_a}^2_F\right\} + \frac{\gamma E_3}{n}.
\end{align*}
Due to $\norm{\mZ^{k}_a}_{F} \leq 2 \sqrt{\min\{n, d\}} \bar{z}$ and $\frac{\kappa_a}{n_a^2} = \frac{1}{n_a \min\{n_a, d\}} \geq \frac{1}{n \min\{n, d\}}$,
\begin{align*}
   \norm{\mZ^{k+1}_a}^2_{F} 
   &\leq \norm{\mZ^{k}_a}^2_{F} - \frac{2 \gamma}{n \min\{n, d\}} \min\left\{\frac{1}{4 n}, \frac{1}{8 \bar{z}^2}\right\} \norm{\mZ^{k}_a}^4_F + \frac{\gamma E_3}{n},
\end{align*}
for all $\bar{K} \leq k \leq \hat{K} - 1.$ 
Using proof by contradiction, assume that $\norm{\mZ^{k}_a}_F^2 > \varepsilon$ for all $\bar{K} \leq k \leq \hat{K}.$ 
Since $E_3 \leq \frac{\varepsilon^2}{\min\{n, d\}} \min\left\{\frac{1}{4 n}, \frac{1}{8 \bar{z}^2}\right\}$ (see \eqref{eq:gzBeKKVwmhdsVvkWWVK}), we get
\begin{align*}
   \norm{\mZ^{k+1}_a}^2_{F} 
   &\leq \norm{\mZ^{k}_a}^2_{F} - \frac{\gamma}{n \min\{n, d\}} \min\left\{\frac{1}{4 n}, \frac{1}{8 \bar{z}^2}\right\} \norm{\mZ^{k}_a}^4_F.
\end{align*}
Applying Lemma~\ref{lemma:seq}, we get
\begin{align*}
   \norm{\mu^{\hat{K}}_{\ell} - \bar\mu_{a}}^2\leq \norm{\mZ^{\hat{K}}_a}^2_{F} \leq \frac{n \min\{n, d\} \max\left\{4 n, 8 \bar{z}^2\right\}}{\gamma} \frac{1}{\hat{K} - \bar{K}} \leq \varepsilon,
\end{align*}
for all $\ell \in I_a$ due to the definition of $\hat{K}$ in \eqref{eq:wkBVdTRSfZPDiUDTm}. We have proved a contradiction.
Thus, $\norm{\mZ^{\bar{k}}_a}^2_{F} \leq \varepsilon$ for some $\bar{k} \in \{\bar{K}, \dots, \hat{K}\}$ and for all $\ell \in I_a.$ It remains to prove that this inequality holds for all $k > \bar{k}.$ Notice that 
\begin{align*}
   \norm{\mZ^{\bar{k}}_a}^2_{\textnormal{op}} \leq \norm{\mZ^{\bar{k}}_a}^2_{F} \leq \varepsilon \leq 2 \bar{z}
\end{align*}
since $\bar{z} \geq 1$ (see \eqref{eq:VheRdmk}). Similarly to the previous steps, \eqref{eq:ckwxZENYRIMy} holds for $k = \bar{k}$ and 
\begin{align*}
   \norm{\mZ^{\bar{k} + 1}_a}^2_{F} 
   &\leq \norm{\mZ^{\bar{k}}_a}^2_{F} - \frac{2 \gamma}{n \min\{n, d\}} \min\left\{\frac{1}{4 n}, \frac{1}{8 \bar{z}^2}\right\} \norm{\mZ^{\bar{k}}_a}^4_F + \frac{\gamma E_3}{n} \\
   &\leq \norm{\mZ^{\bar{k}}_a}^2_{F} - \frac{2 \gamma}{n \min\{n, d\}} \min\left\{\frac{1}{4 n}, \frac{1}{8 \bar{z}^2}\right\} \norm{\mZ^{\bar{k}}_a}^4_F + \frac{\gamma \varepsilon^2}{n \min\{n, d\}} \min\left\{\frac{1}{4 n}, \frac{1}{8 \bar{z}^2}\right\}.
\end{align*}
Since 
\begin{align*}
&\max\limits_{0 \leq x \leq \varepsilon}\left\{g(x) \eqdef x - \frac{2 \gamma}{n \min\{n, d\}} \min\left\{\frac{1}{4 n}, \frac{1}{8 \bar{z}^2}\right\} x^2 + \frac{\gamma \varepsilon^2}{n \min\{n, d\}} \min\left\{\frac{1}{4 n}, \frac{1}{8 \bar{z}^2}\right\}\right\} \\
&\overset{x^* = \varepsilon}{=} \varepsilon - \frac{2 \gamma}{n \min\{n, d\}} \min\left\{\frac{1}{4 n}, \frac{1}{8 \bar{z}^2}\right\} \varepsilon^2 + \frac{\gamma \varepsilon^2}{n \min\{n, d\}} \min\left\{\frac{1}{4 n}, \frac{1}{8 \bar{z}^2}\right\} \leq \varepsilon
\end{align*}
because the (unconstrained) optimal point of $g$ on $\R$ is $\frac{n \min\{n, d\}}{4 \gamma \min\left\{\frac{1}{4 n}, \frac{1}{8 \bar{z}^2}\right\}} \geq \frac{n}{4 \gamma} \geq 9 \bar{z}^2 \geq 1 \geq \varepsilon,$ we can conclude that 
\begin{align}
   \label{eq:ggAsdtpeC}
   \norm{\mZ^{\bar{k} + 1}_a}^2_{F} \leq \varepsilon.
\end{align}
Inductively, repeating these steps, we get $\norm{\mu^{k}_{\ell} - \bar\mu_{a}}^2 \leq \varepsilon$ for all $k \in \{\bar{k}, \dots, \hat{K}\}$ and $\ell \in I_a.$ Notice that  
\begin{align}
   \norm{\mZ^{\hat{K}}_a}^2_{F} \leq \varepsilon \leq 2 \bar{z}
\end{align}
for all $a \in A^*.$ Repeating the previous steps again, we can get $\norm{\mu^{k}_{\ell} - \bar\mu_{a}}^2 \leq \norm{\mZ^{k}_a}^2_{F} \leq \varepsilon$ for all $k \geq \hat{K},$ $a \in A^*,$ and $\ell \in I_a.$ 
\end{proof}

\subsection{Auxiliary lemmas}
\label{sec:auxlemmas}
The following lemmas are used in the proof of Theorem~\ref{thm:mainmany} and rely on the definitions and notation introduced therein. 
\LEMMABOUND*
\begin{proof}
   Clearly,
   \begin{align*}
      \norm{m_{\mu^k}(x) - m^k_{a^*_\ell}(x)} = \norm{\sum_{i=1}^n r^k_i(x)\mu^k_i - \sum_{i\in I_{a^*_\ell}} r_i^{k,a^*_\ell}(x)\mu^k_i} = \norm{\sum_{i \in I_{a^*_\ell}} (r^k_i(x) - r_i^{k,a^*_\ell}(x))\mu^k_i + \sum_{i\not\in I_{a^*_\ell}} r_i^{k}(x)\mu^k_i}.
   \end{align*}
   Since 
   \begin{align*}
      r^k_i(x) - r_i^{k,a^*_\ell}(x) 
      &= \frac{\exp\left(-\frac12\|x-\mu^k_i\|^2\right)}{\sum_{j=1}^n\exp\left(-\frac12\|x-\mu^k_j\|^2\right)} - \frac{\exp\left(-\frac12\|x-\mu^k_i\|^2\right)}{\sum_{j\in I_{a^*_\ell}}\exp\left(-\frac12\|x-\mu^k_j\|^2\right)} \\
      &= r_i^{k,a^*_\ell}(x)\left(\frac{\sum_{j\in I_{a^*_\ell}}\exp\left(-\frac12\|x-\mu^k_j\|^2\right)}{\sum_{j=1}^n\exp\left(-\frac12\|x-\mu^k_j\|^2\right)} - 1\right) \\
      &= - r_i^{k,a^*_\ell}(x)\left(\frac{\sum_{j\not\in I_{a^*_\ell}}\exp\left(-\frac12\|x-\mu^k_j\|^2\right)}{\sum_{j=1}^n\exp\left(-\frac12\|x-\mu^k_j\|^2\right)}\right) = - r_i^{k,a^*_\ell}(x) \sum_{j\not\in I_{a^*_\ell}} r^k_j(x),
   \end{align*}
   for all $i \in I_{a^*_\ell},$
   \begin{align}
      \label{eq:umshVQY1}
      \norm{m_{\mu^k}(x) - m^k_{a^*_\ell}(x)} = \norm{\sum_{i\not\in I_{a^*_\ell}} r_i^{k}(x)\left(\mu^k_i - \sum_{i \in I_{a^*_\ell}} r_i^{k,a^*_\ell}(x) \mu^k_i\right)} \leq 2 \hat\tau^k_{a^*_\ell}(x) R_{\max}.
   \end{align}
   Next,
   \begin{align}
      \label{eq:umshVQY}
      &\norm{\bar m(x) - \bar \mu_{a^*_\ell}} = \norm{\sum_{b=1}^m \bar r_b(x) (\bar\mu_b - \bar \mu_{a^*_\ell})} \leq \bar \tau_{a^*_\ell}(x) \max_{b \in [m]}\norm{\bar\mu_b - \bar \mu_{a^*_\ell}} \leq 2 \bar \tau_{a^*_\ell}(x) R_{\max},
   \end{align}
   \begin{align*}
&\norm{\mC_{\mu^k}(x) - \mC^k_{a_\ell^*}(x)}_{\textnormal{op}} \\
&=\norm{\sum_{i=1}^n r^k_i(x) \bigl(\mu^k_i-m_{\mu^k}(x)\bigr) \bigl(\mu^k_i-m_{\mu^k}(x)\bigr)^\top - \sum_{i\in I_{a^*_{\ell}}} r_i^{k,a^*_{\ell}}(x) \bigl(\mu^k_i-m^k_{a^*_{\ell}}(x)\bigr) \bigl(\mu^k_i-m^k_{a^*_{\ell}}(x)\bigr)^\top}_{\textnormal{op}} \\
&\leq
\norm{
\sum_{i\notin I_{a_\ell^*}} r_i^k(x)
\bigl(\mu_i^k-m_{\mu^k}(x)\bigr)
\bigl(\mu_i^k-m_{\mu^k}(x)\bigr)^\top
}_{\textnormal{op}} \\
&\quad+
\norm{
\sum_{i\in I_{a_\ell^*}}
\left(r_i^k(x)-r_i^{k,a_\ell^*}(x)\right)
\bigl(\mu_i^k-m_{\mu^k}(x)\bigr)
\bigl(\mu_i^k-m_{\mu^k}(x)\bigr)^\top
}_{\textnormal{op}} \\
&\quad+
\norm{
\sum_{i\in I_{a_\ell^*}} r_i^{k,a_\ell^*}(x)
\left[
\bigl(\mu_i^k-m_{\mu^k}(x)\bigr)
\bigl(\mu_i^k-m_{\mu^k}(x)\bigr)^\top
-
\bigl(\mu_i^k-m^k_{a_\ell^*}(x)\bigr)
\bigl(\mu_i^k-m^k_{a_\ell^*}(x)\bigr)^\top
\right]
}_{\textnormal{op}}.
\end{align*}
First,
\begin{align*}
&\norm{
\sum_{i\notin I_{a_\ell^*}} r_i^k(x)
\bigl(\mu_i^k-m_{\mu^k}(x)\bigr)
\bigl(\mu_i^k-m_{\mu^k}(x)\bigr)^\top
}_{\textnormal{op}} \leq
\sum_{i\notin I_{a_\ell^*}} r_i^k(x)
\norm{\mu_i^k-m_{\mu^k}(x)}^2 \leq
4R_{\max}^2 \sum_{i\notin I_{a_\ell^*}} r_i^k(x)
=
4R_{\max}^2\hat\tau_{a_\ell^*}^k(x).
\end{align*}
Second, using
\[
r_i^k(x)-r_i^{k,a_\ell^*}(x)
=
-r_i^{k,a_\ell^*}(x)\hat\tau_{a_\ell^*}^k(x),
\qquad i\in I_{a_\ell^*},
\]
we get
\begin{align*}
&\norm{
\sum_{i\in I_{a_\ell^*}}
\left(r_i^k(x)-r_i^{k,a_\ell^*}(x)\right)
\bigl(\mu_i^k-m_{\mu^k}(x)\bigr)
\bigl(\mu_i^k-m_{\mu^k}(x)\bigr)^\top
}_{\textnormal{op}} \\
&\leq
\sum_{i\in I_{a_\ell^*}}
\abs{r_i^k(x)-r_i^{k,a_\ell^*}(x)}
\norm{\mu_i^k-m_{\mu^k}(x)}^2 \\
&=
\hat\tau_{a_\ell^*}^k(x)
\sum_{i\in I_{a_\ell^*}} r_i^{k,a_\ell^*}(x)
\norm{\mu_i^k-m_{\mu^k}(x)}^2.
\end{align*}
Finally, since
\[
\mu_i^k-m_{\mu^k}(x)
=
\mu_i^k-m_{a_\ell^*}^k(x)
+
m_{a_\ell^*}^k(x)-m_{\mu^k}(x),
\]
and
\[
\sum_{i\in I_{a_\ell^*}}
r_i^{k,a_\ell^*}(x)
\bigl(\mu_i^k-m_{a_\ell^*}^k(x)\bigr)
=0,
\]
the cross terms vanish after averaging. Therefore,
\begin{align*}
&\norm{
\sum_{i\in I_{a_\ell^*}} r_i^{k,a_\ell^*}(x)
\left[
\bigl(\mu_i^k-m_{\mu^k}(x)\bigr)
\bigl(\mu_i^k-m_{\mu^k}(x)\bigr)^\top
-
\bigl(\mu_i^k-m_{a_\ell^*}^k(x)\bigr)
\bigl(\mu_i^k-m_{a_\ell^*}^k(x)\bigr)^\top
\right]
}_{\textnormal{op}} \\
&=
\norm{
\bigl(m_{a_\ell^*}^k(x)-m_{\mu^k}(x)\bigr)
\bigl(m_{a_\ell^*}^k(x)-m_{\mu^k}(x)\bigr)^\top
}_{\textnormal{op}}.
\end{align*}
Using these bounds,
\begin{align*}
&\norm{\mC_{\mu^k}(x) - \mC^k_{a_\ell^*}(x)}_{\textnormal{op}} \\
&\leq
4R_{\max}^2\hat\tau^k_{a_\ell^*}(x)
+
\hat\tau^k_{a_\ell^*}(x)
\sum_{i\in I_{a_\ell^*}} r_i^{k,a_\ell^*}(x)
\norm{\mu_i^k-m_{\mu^k}(x)}^2 +
\norm{
\left(m^k_{a_\ell^*}(x)-m_{\mu^k}(x)\right)
\left(m^k_{a_\ell^*}(x)-m_{\mu^k}(x)\right)^\top
}_{\textnormal{op}} \\
&\leq
4R_{\max}^2\hat\tau^k_{a_\ell^*}(x)
+
4R_{\max}^2\hat\tau^k_{a_\ell^*}(x)
+
\norm{m^k_{a_\ell^*}(x)-m_{\mu^k}(x)}^2 \\
&\leq
8R_{\max}^2\hat\tau^k_{a_\ell^*}(x)
+
4R_{\max}^2\left(\hat\tau^k_{a_\ell^*}(x)\right)^2 \\
&\leq
12R_{\max}^2\hat\tau^k_{a_\ell^*}(x),
\end{align*}
where we used \eqref{eq:umshVQY1}. Using \eqref{eq:umshVQY},
   \begin{align*}
      \norm{\bar \mC(x)}_{\textnormal{op}} = \norm{\sum\limits_{b=1}^m \bar r_b(x) \bigl(\bar\mu_b-\bar m(x)\bigr) \bigl(\bar\mu_b-\bar m(x)\bigr)^\top}_{\textnormal{op}} \leq 4 (\bar \tau_{a^*_\ell}(x))^2 R_{\max}^2 + 4 \bar \tau_{a^*_\ell}(x) R_{\max}^2 \leq 8 \bar \tau_{a^*_\ell}(x) R_{\max}^2
   \end{align*}
   because $\norm{m_{\mu^k}(x)} = \norm{\sum_{i=1}^n r^k_i(x)\mu_i^k} \leq R_{\max}$ and $\norm{\bar m(x)} \leq \norm{\sum_{a=1}^m \bar r_{a}(x)\bar\mu_{a}} \leq R_{\max}.$ 
   Combining all bounds together and using the triangle inequality,
   \begin{align*}
      &\norm{\mathcal{E}^k_{\ell}(x)} \\
      &\leq 2 \hat\tau^k_{a^*_\ell}(x) R_{\max} + 2 \bar \tau_{a^*_\ell}(x) R_{\max} + 12R_{\max}^2\hat\tau^k_{a_\ell^*}(x) \norm{\mu^k_\ell - m^k_{a^*_\ell}(x)} + 2 \hat\tau^k_{a^*_\ell}(x) R_{\max} \norm{\mC_{\mu^k}(x)}_{\textnormal{op}} \\
      &\quad + 8 \bar \tau_{a^*_\ell}(x) R_{\max}^2 \norm{\mu^k_\ell-\bar m(x)} \\
      &\leq 2 \hat\tau^k_{a^*_\ell}(x) R_{\max} + 2 \bar \tau_{a^*_\ell}(x) R_{\max} + 24 R_{\max}^3 \hat\tau^k_{a_\ell^*}(x) \\
      &\quad + 8 \hat\tau^k_{a^*_\ell}(x) R_{\max}^3 + 16 \bar \tau_{a^*_\ell}(x) R_{\max}^3 \leq 52 \max\{R_{\max}, R_{\max}^3\} \max\{\hat\tau^k_{a^*_\ell}(x), \bar \tau_{a^*_\ell}(x)\}
   \end{align*}
   because
   \begin{align*}
      \norm{\mC_{\mu^k}(x)}_{\textnormal{op}} \leq \sum_{i=1}^n r^k_i(x) \norm{\bigl(\mu^k_i-m_{\mu^k}(x)\bigr) \bigl(\mu^k_i-m_{\mu^k}(x)\bigr)^\top}_{\textnormal{op}} \leq 4 R_{\max}^2.
   \end{align*}
\end{proof}

\LEMMABOUNDTWO*

\begin{proof}
   Indeed,
   \begin{align*}
      \norm{m^k_{a^*_\ell}(x) - \mu^k_{\ell}} 
      &= \norm{\sum_{i\in I_{a^*_\ell}} r_i^{k,a^*_\ell}(x)\mu^k_i - \mu^k_{\ell}} = \norm{\sum_{i\in I_{a^*_\ell}, i \neq \ell} r_i^{k,a^*_\ell}(x)(\mu^k_i - \mu^k_{\ell})} \\
      &\leq \sum_{i\in I_{a^*_\ell}, i \neq \ell} r_i^{k,a^*_\ell}(x) \norm{\mu^k_i - \mu^k_{\ell}} \leq 2 R_{\max} \tau^k_{\ell}(x).
   \end{align*}
   Thus,
   \begin{align*}
      \norm{\bar{\mathcal{E}}^k_{\ell}(x)} \leq 2 R_{\max} \tau^k_{\ell}(x) + 2 R_{\max} \tau^k_{\ell}(x) \norm{\mC^k_{a^*_\ell}(x)}_{\textnormal{op}} \leq 10 \max\{R_{\max}, R_{\max}^3\} \tau^k_{\ell}(x)
   \end{align*}
   because
   \begin{align*}
      \norm{\mC^k_{a^*_\ell}(x)}_{\textnormal{op}} \leq \sum_{i\in I_{a^*_{\ell}}} r_i^{k,a^*_{\ell}}(x) \norm{\bigl(\mu^k_i-m^k_{a^*_{\ell}}(x)\bigr) \bigl(\mu^k_i-m^k_{a^*_{\ell}}(x)\bigr)^\top}_{\textnormal{op}} \leq 4 R_{\max}^2.
   \end{align*}
\end{proof}

\LEMMAPARAMS*
\begin{proof}
      Indeed,
      \begin{align*}
         \Delta_1 = \frac{2 R^0 \delta_{\min} \rho}{n m^2 \sqrt{d}} \geq \delta^2_{\max}
      \end{align*}
      due to the choice of $R^0$ and $\delta_{\min} \geq 1.$ Next,
      \begin{align*}
         \Delta_2 = (R^0)^2 \left(\frac{\sqrt{2\pi}\rho}{\sqrt d\, n^2}\right)^{\frac{2}{d-1}} \geq \frac{(R^0)^2}{2}.
      \end{align*}
      due to the condition on $d.$ Since $(R^0)^2 \geq 256 \log(55552 \cdot n),$ we get
      \begin{align*}
         \Delta_2 \geq \frac{(R^0)^2}{2} \geq 32 \log\left(55552 n (R^0)^4\right) = \hat\Delta_2.
      \end{align*}
      Next,
      \begin{align*}
         32 \log(434 n R^3 / E_1) = 32 \log(55552 n (R^0)^4 / \delta_{\min}) \leq 32 \log\left(55552 n (R^0)^4\right) = \hat\Delta_2
      \end{align*}
      because $\delta_{\min} \geq 1.$ Notice that 
      \begin{align*}
         \sqrt{32 \log\left(434 R^3 m / E_1\right)} = \sqrt{32 \log\left(55552 (R^0)^4 m / \delta_{\min}\right)} \leq \delta_{\min}
      \end{align*}
      because $\delta_{\min} \geq 2^{10} \sqrt{\log\left(2^{19} m (R^0)^4\right)}.$
      Since 
      \begin{align}
         \label{eq:mfiuPrmGKoEuWJb}
         R^0 \geq \frac{n m^2 \sqrt{d} \delta^2_{\max}}{2 \rho} \geq \delta_{\min},
      \end{align}
      we get 
      \begin{align*}
         E_1 = \frac{\delta_{\min}}{16 R^0} \leq R^0 = \frac{R}{2}.
      \end{align*}
      Due to the definitions of $\bar{\Delta}$ and $\Delta_2,$
      \begin{align*}
         \frac{4 \sqrt{\bar{\Delta} \Delta_2}}{R^0} = 2 \sqrt{\hat\Delta_2} \times \frac{2 \sqrt{\left(\frac{\sqrt{2\pi}\rho}{\sqrt d\, n^2}\right)^{\frac{2}{d-1}}}}{\sqrt{2}} \geq 2 \sqrt{\hat\Delta_2},
      \end{align*}
      where we use the condition on $d.$ Also, due to the condition on $d$ and $\delta_{\min},$ 
      \begin{align*}
         \Delta_3 = \frac{\delta_{\min}}{8} \left(\frac{\rho}{m n^2}\right)^{\frac{1}{d-2}} \geq \frac{\delta_{\min}}{16} \geq 2 \sqrt{\hat\Delta_2}.
      \end{align*}
      The next inequality
      \begin{align*}
         E_1 = \frac{\delta_{\min}}{8 R} \leq \frac{\sqrt{\hat\Delta_2}}{2}
      \end{align*}
      holds because $R = 2 R^0 \overset{\eqref{eq:mfiuPrmGKoEuWJb}}{\geq} \frac{\delta_{\min}}{4}$ and $\hat\Delta_2 \geq 1.$ 
   \end{proof}

\LEMMAPARAMSTWO*

\begin{proof}
   Due to the definition of $\bar{z},$
   \begin{align}
      \bar{z} 
      &\eqdef \frac{8 \sqrt{\bar{\Delta}} \Delta_4}{R^0} + \frac{8 \sqrt{n \bar{\Delta}} \bar{R}}{R^0} + \sqrt{n} E_1 \nonumber \\
      &\overset{\eqref{eq:kjdsqYLZ}}{\leq} \frac{8 \sqrt{16 \log\left(3472 n R^4\right)} \Delta_4}{R^0} + \frac{8 \sqrt{16 n \log\left(3472 n R^4\right)} \bar{R}}{R^0} + \frac{\sqrt{n} \delta_{\min}}{16 R^0} \nonumber \\
      &\overset{\textnormal{Def. of } \Delta_4}{\leq} 8 \sqrt{16 \log\left(3472 n R^4\right)} \sqrt{\frac{n}{d} + \max\left\{\frac{4}{3}\log\left(\frac{2 d}{\rho}\right), \sqrt{\frac{4 n}{d} \log\left(\frac{2 d}{\rho}\right)}\right\}} \nonumber \\
      &\qquad + \frac{8 \sqrt{16 n \log\left(3472 n R^4\right)} \bar{R}}{R^0} + \frac{\sqrt{n} \delta_{\min}}{16 R^0}. \label{eq:JWcjWjNHPFFHUk}
   \end{align}
   Notice that 
   \begin{align*}
      &\frac{8 \sqrt{16 n \log\left(3472 n R^4\right)} \bar{R}}{R^0} \overset{\eqref{eq:kjdsqYLZ}}{=} \frac{8 \sqrt{16 n \log\left(55552 n (R^0)^4\right)} \bar{R}}{R^0} \leq \frac{\delta_{\min}}{16}
   \end{align*}
   since
   \begin{align*}
      R^0 \geq 1024 \sqrt{16 n \max\{\log\left(2^{66} n^3 \bar{R}^4\right), 1\}} \bar{R}
   \end{align*} and $\delta_{\min} \geq 1,$ where we use Lemma~\ref{lemma:log} and the definition of $R^0.$
   Moreover, $R^0 \geq \frac{2 n m^2 \sqrt{d} \delta^2_{\max}}{\rho} \geq 2 \sqrt{n}$ due to the definition of $R^0,$ and 
   \begin{align*}
      \delta_{\min} 
      &\geq 256 \sqrt{16 \log\left(55552 n (R^0)^4\right)} \sqrt{\frac{n}{d} + \max\left\{\frac{4}{3}\log\left(\frac{2 d}{\rho}\right), \sqrt{\frac{4 n}{d} \log\left(\frac{2 d}{\rho}\right)}\right\}} \\
      &= 256 \sqrt{16 \log\left(3472 n R^4\right)} \sqrt{\frac{n}{d} + \max\left\{\frac{4}{3}\log\left(\frac{2 d}{\rho}\right), \sqrt{\frac{4 n}{d} \log\left(\frac{2 d}{\rho}\right)}\right\}}.
   \end{align*} 
   Thus, $\delta_{\min} \geq 8 \bar{z}.$ 
   Notice that 
   \begin{align*}
      \bar{z} 
      &\eqdef \frac{8 \sqrt{\bar{\Delta}} \Delta_4}{R^0} + \frac{8 \sqrt{n \bar{\Delta}} \bar{R}}{R^0} + \sqrt{n} E_1 \\
      &\geq 8 \sqrt{16 \log\left(55552 n (R^0)^4\right)} \sqrt{\frac{n}{d} + \max\left\{\frac{4}{3}\log\left(\frac{2 d}{\rho}\right), \sqrt{\frac{4 n}{d} \log\left(\frac{2 d}{\rho}\right)}\right\}} \\
      &\geq 8 \sqrt{16 \log\left(55552 n (R^0)^4\right)} \sqrt{\frac{4}{3}\log\left(\frac{2 d}{\rho}\right)} \geq 1.
   \end{align*}
   Next, we bound $\bar{z}$ using \eqref{eq:JWcjWjNHPFFHUk}:
   \begin{align*}
      \bar{z} 
      &\leq 8 \sqrt{16 \log\left(3472 n R^4\right)} \sqrt{\frac{n}{d} + \max\left\{\frac{4}{3}\log\left(\frac{2 d}{\rho}\right), \sqrt{\frac{4 n}{d} \log\left(\frac{2 d}{\rho}\right)}\right\}} \\
      &\qquad + \frac{8 \sqrt{16 n \log\left(3472 n R^4\right)} \bar{R}}{R^0} + \frac{\sqrt{n} \delta_{\min}}{16 R^0}.
   \end{align*}
   Since $R = 2 R^0,$ $R^0 \geq \sqrt{n} \bar{R},$ and $R^0 \geq \frac{2 n m^2 \sqrt{d} \delta_{\max}}{\rho} \geq 2 \sqrt{n} \delta_{\max},$
   \begin{align*}
      \bar{z} 
      \leq 65 \sqrt{\log\left(2^{16} n (R^0)^4\right)} \sqrt{\frac{n}{d} + \max\left\{\frac{4}{3}\log\left(\frac{2 d}{\rho}\right), \sqrt{\frac{4 n}{d} \log\left(\frac{2 d}{\rho}\right)}\right\}}.
   \end{align*}
   Next,
   \begin{align*}
      \frac{1}{E_2} 
      &= \max\left\{\frac{2 \gamma \hat{K}}{n \bar{z}}, \frac{16 \min\{n, d\} \sqrt{\min\{n, d\}} \bar{z}}{\varepsilon^2} \max\left\{4 n, 8 \bar{z}^2\right\}, \sqrt{\frac{8 \gamma \min\{n, d\}}{n \varepsilon^2} \max\left\{4 n, 8 \bar{z}^2\right\}}\right\} \\
   &\overset{\textnormal{Def. of } \gamma}{\leq} \max\left\{\frac{2 \gamma \hat{K}}{n \bar{z}}, \frac{16 \min\{n, d\} \sqrt{\min\{n, d\}} \bar{z}}{\varepsilon^2} \max\left\{4 n, 8 \bar{z}^2\right\}, \sqrt{\frac{2 \min\{n, d\}}{9 \bar{z}^2 \varepsilon^2} \max\left\{4 n, 8 \bar{z}^2\right\}}\right\} \\
      &\overset{\bar{z} \geq 1}{\leq} \max\left\{\frac{2 \gamma \hat{K}}{n \bar{z}}, \frac{16 \min\{n, d\} \sqrt{\min\{n, d\}} \bar{z}}{\varepsilon^2} \max\left\{4 n, 8 \bar{z}^2\right\}, \sqrt{\frac{2 n \min\{n, d\}}{\varepsilon^2}}\right\} \\
      &\overset{\bar{z} \geq 1}{\leq} \max\left\{\frac{2 \gamma \hat{K}}{n \bar{z}}, \frac{128 n^{5/2} \bar{z}^3}{\varepsilon^2}, \sqrt{\frac{2 n^2}{\varepsilon^2}}\right\}.
   \end{align*}
   Using \eqref{eq:TUyQcvjVGXjffeyd} and \eqref{eq:wkBVdTRSfZPDiUDTm},
   \begin{align*}
      \hat{K} 
      &\leq \frac{n \log\left(\frac{R^0}{4 \sqrt{16 \log\left(55552 n (R^0)^4\right)}}\right)}{2 \gamma} + \frac{2 n \min\{n, d\} \max\left\{4 n, 8 \bar{z}^2\right\}}{\gamma \varepsilon} \\
      &\leq \frac{n \log\left(R^0\right)}{2 \gamma} + \frac{2 n \min\{n, d\} \max\left\{4 n, 8 \bar{z}^2\right\}}{\gamma \varepsilon}.
   \end{align*}
   Thus
   \begin{align*}
      \frac{1}{E_2} 
      &\overset{R^0 \geq 1}{\leq} \max\left\{\frac{2 \gamma}{n \bar{z}} \left(\frac{n \log\left(R^0\right)}{2 \gamma} + \frac{2 n^2 \max\left\{4 n, 8 \bar{z}^2\right\}}{\gamma \varepsilon}\right), \frac{128 n^{5/2} \bar{z}^3}{\varepsilon^2}, \sqrt{\frac{2 n^2}{\varepsilon^2}}\right\} \\
      &\leq \max\left\{\frac{2 \log\left(R^0\right)}{\bar{z}}, \frac{8 n \max\left\{4 n, 8 \bar{z}^2\right\}}{\varepsilon \bar{z}}, \frac{128 n^{5/2} \bar{z}^3}{\varepsilon^2}, \sqrt{\frac{2 n^2}{\varepsilon^2}}\right\} \\
      &\overset{\bar{z} \geq 1}{\leq} \max\left\{2 \log\left(R^0\right), \frac{8 n \max\left\{4 n, 8 \bar{z}\right\}}{\varepsilon}, \frac{128 n^{5/2} \bar{z}^3}{\varepsilon^2}, \sqrt{\frac{2 n^2}{\varepsilon^2}}\right\} \\
      &\leq \frac{128 n^{5/2} \log\left(R^0\right)}{\min\{\varepsilon^2, 1\}} \bar{z}^3.
   \end{align*}
   Next,
   \begin{align*}
      &\sqrt{128 \log\left(\frac{156 \sqrt{n} R^3 \max\left\{n, m\right\}}{E_2}\right)} \\
      &\leq \sqrt{128 \log\left(1248 \sqrt{n} (R^0)^3 \max\left\{n, m\right\} \frac{128 n^{5/2} \log\left(R^0\right)}{\min\{\varepsilon^2, 1\}} \bar{z}^3\right)} \\
      &\leq \sqrt{128 \log\left(\frac{159744 n^{3} (R^0)^4 \max\left\{n, m\right\}}{\min\{\varepsilon^2, 1\}} \bar{z}^3\right)} \\
      &\overset{\bar{z} \leq \frac{\delta_{\min}}{8}}{\leq} \sqrt{128 \log\left(\frac{312 n^{3} (R^0)^4 \max\left\{n, m\right\}}{\min\{\varepsilon^2, 1\}} \delta^3_{\min}\right)}.
   \end{align*}
   Using Lemma~\ref{lemma:log-linear}, it sufficient for $\delta_{\min}$ to satisfy
   \begin{align*}
      \delta_{\min} \geq 46 \sqrt{\log\left(\frac{2^{20} n^{3} (R^0)^4 \max\left\{n, m\right\}}{\min\{\varepsilon^2, 1\}}\right)}
   \end{align*}
   and ensure that the fourth inequality of the lemma holds, which indeed holds under our assumption about $\delta_{\min}$. 
   Finally,
   \begin{align*}
      R \geq 2 \bar{z} + \bar{R}
   \end{align*}
   since $R \geq 2 \bar{R}$ and 
   \begin{align*}
      2 \bar{z} 
      &\leq 130 \sqrt{\log\left(2^{16} n (R^0)^4\right)} \sqrt{\frac{n}{d} + \max\left\{\frac{4}{3}\log\left(\frac{2 d}{\rho}\right), \sqrt{\frac{4 n}{d} \log\left(\frac{2 d}{\rho}\right)}\right\}} \\
      &\leq 130 \sqrt{5 \log\left(2^{16} n (R^0)^4\right)} \sqrt{n \log\left(\frac{2 d}{\rho}\right)} \\
      &\leq 130 \sqrt{5 \cdot 16 \log\left(2\right) + 5 n + 5 \log\left((R^0)^4\right)} \sqrt{n \log\left(\frac{2 d}{\rho}\right)} \\
      &\leq \left(1300 \sqrt{n} + 260 \sqrt{5} \sqrt{\log\left(R^0\right)}\right) \sqrt{n \log\left(\frac{2 d}{\rho}\right)} \leq R^0 = \frac{R}{2}
   \end{align*}
   because $R^0 \geq \frac{2^{20} n \sqrt{d}}{\rho} \geq \frac{2^{19} n}{\sqrt{2}} \log\left(\frac{2 d}{\rho}\right)$ and $R^0 \geq \sqrt{\frac{2^{19}}{\sqrt{2}}} \sqrt{\log(R^0)} \sqrt{n \log\left(\frac{2 d}{\rho}\right)}.$
\end{proof}

\LEMMAZ*

\begin{proof}
   Using \eqref{eq:aegiybPjoBoRKeigLO},
\begin{align*}
   &\inp{\mZ^{k}_a}{\mB_a^k \mZ^{k}_a} \\
   &= n_{a} \sum_{i \in I_a} \sum_{j \in I_a} \mathbb E_{x \sim p^{k,a}_{\mu}}
   \left[r_i^{k,a}(x) r^{k,a}_{j} (x)  + r_i^{k,a}(x) r^{k,a}_{j} (x) \bigl(\mu^k_i-m^k_{a}(x)\bigr)^\top \bigl(\mu^k_j-m^k_{a}(x)\bigr)\right] (\mu^{k}_{i} - \bar \mu_{a})^\top (\mu^{k}_{j} - \bar \mu_{a}) \\
   &= n_{a} \mathbb E_{x \sim p^{k,a}_{\mu}}
   \left[\sum_{i \in I_a} r_i^{k,a}(x) (\mu^{k}_{i} - \bar \mu_{a})^\top \left(\sum_{j \in I_a} r^{k,a}_{j} (x) (\mu^{k}_{j} - \bar \mu_{a})\right)\right] \\
   &\quad  + n_a \mathbb E_{x \sim p^{k,a}_{\mu}} 
   \left[\sum_{i \in I_a} r_i^{k,a}(x) \bigl(\mu^k_i-m^k_{a}(x)\bigr)^\top \left(\sum_{j \in I_a}  r^{k,a}_{j}(x) \bigl(\mu^k_j-m^k_{a}(x)\bigr) (\mu^{k}_{j} - \bar \mu_{a})^\top \right) (\mu^{k}_{i} - \bar \mu_{a})\right] \\
   &= n_{a} \mathbb E_{x \sim p^{k,a}_{\mu}}
   \left[\norm{\sum_{i \in I_a} r_i^{k,a}(x) (\mu^{k}_{i} - \bar \mu_{a})}^2\right] + n_a \mathbb E_{x \sim p^{k,a}_{\mu}} 
   \left[\norm{\sum_{i \in I_a}  r^{k,a}_{i}(x) \bigl(\mu^k_i-m^k_{a}(x)\bigr) (\mu^{k}_{i} - \bar \mu_{a})^\top}_F^2\right].
\end{align*}
Using the definition of $m^k_{a}(x)$ and $\sum_{i\in I_{a}} r_i^{k,a}(x) \bigl(\mu^k_i-m^k_{a}(x)\bigr) b^\top = 0$ for all $b \in \R^d,$
\begin{align*}
   &\inp{\mZ^{k}_a}{\mB_a^k \mZ^{k}_a} \\
   &= n_{a} \mathbb E_{x \sim p^{k,a}_{\mu}}
   \left[\norm{m^k_{a}(x) - \bar \mu_{a}}^2\right] + n_a \mathbb E_{x \sim p^{k,a}_{\mu}} 
   \left[\norm{\sum_{i \in I_a}  r^{k,a}_{i}(x) \bigl(\mu^k_i-m^k_{a}(x)\bigr) (\mu^{k}_{i} - m^k_{a}(x))^\top}_F^2\right] \\
   &= n_{a} \mathbb E_{x \sim p^{k,a}_{\mu}}
   \left[\norm{m^k_{a}(x) - \bar \mu_{a}}^2\right] + n_a \mathbb E_{x \sim p^{k,a}_{\mu}} 
   \textnormal{tr}\left(\left(\sum_{i \in I_a}  r^{k,a}_{i}(x) \bigl(\mu^k_i-m^k_{a}(x)\bigr) (\mu^{k}_{i} - m^k_{a}(x))^\top\right)^2\right)
\end{align*}
Let $\kappa_a \eqdef \frac{n_a}{\min\{n_a, d\}}.$ Due to $\left(\textnormal{tr}(\mA)\right)^2 \leq \textnormal{rank}(\mA) \textnormal{tr}(\mA^2)$ for all symmetric $\mA \in \R^{d \times d},$ we get
\begin{align*}
   \inp{\mZ^{k}_a}{\mB_a^k \mZ^{k}_a} 
   &\geq n_{a} \mathbb E_{x \sim p^{k,a}_{\mu}}
   \left[\norm{m^k_{a}(x) - \bar \mu_{a}}^2\right] + \kappa_a \mathbb E_{x \sim p^{k,a}_{\mu}} 
   \left[\left(\sum_{i \in I_a}  r^{k,a}_{i}(x) \norm{\mu^k_i-m^k_{a}(x)}^2\right)^2\right] \\
   &= n_{a} \mathbb E_{x \sim p^{k,a}_{\mu}}
   \left[\norm{m^k_{a}(x) - \bar \mu_{a}}^2\right] + \kappa_a \mathbb E_{x \sim p^{k,a}_{\mu}} 
   \left[\left(\sum_{i \in I_a}  r^{k,a}_{i}(x) \norm{\mu^k_i - \bar\mu_a}^2 - \norm{m^k_{a}(x) - \bar\mu_a}^2\right)^2\right] \\
   &\geq n_{a} \mathbb E_{x \sim p^{k,a}_{\mu}}
   \left[\norm{m^k_{a}(x) - \bar \mu_{a}}^2\right] + \kappa_a \left(\mathbb E_{x \sim p^{k,a}_{\mu}} 
   \left[\sum_{i \in I_a}  r^{k,a}_{i}(x) \norm{\mu^k_i - \bar\mu_a}^2\right] - \mathbb E_{x \sim p^{k,a}_{\mu}} 
   \left[\norm{m^k_{a}(x) - \bar\mu_a}^2\right]\right)^2 \\
   &= n_{a} \mathbb E_{x \sim p^{k,a}_{\mu}}
   \left[\norm{m^k_{a}(x) - \bar \mu_{a}}^2\right] + \kappa_a \left(\frac{1}{n_a} \norm{\mZ^{k}_a}^2_F - \mathbb E_{x \sim p^{k,a}_{\mu}} 
   \left[\norm{m^k_{a}(x) - \bar\mu_a}^2\right]\right)^2
\end{align*}
since 
\begin{align*}
   &\mathbb E_{x \sim p^{k,a}_{\mu}} \left[\sum_{i \in I_a}  r^{k,a}_{i}(x) \norm{\mu^k_i - \bar\mu_a}^2\right] = \sum_{i \in I_a} \int_{x \in \R^d} \left[r^{k,a}_{i}(x) p^{k,a}_{\mu}(x) \norm{\mu^k_i - \bar\mu_a}^2\right] \, d x \\
   &= \frac{1}{n_a} \sum_{i \in I_a} \int_{x \in \R^d} \left[\norm{\mu^k_i - \bar\mu_a}^2\right] \varphi(x - \mu^k_i) \, d x = \frac{1}{n_a} \sum_{i \in I_a} \norm{\mu^k_i - \bar\mu_a}^2 = \frac{1}{n_a} \norm{\mZ^{k}_a}^2_F
\end{align*}
Notice also that 
\begin{align*}
   \mathbb E_{x \sim p^{k,a}_{\mu}} \left[\norm{m^k_{a}(x) - \bar\mu_a}^2\right] \leq \mathbb E_{x \sim p^{k,a}_{\mu}} \left[\sum_{i \in I_a}  r^{k,a}_{i}(x) \norm{\mu^k_i - \bar\mu_a}^2\right] = \frac{1}{n_a} \norm{\mZ^{k}_a}^2_F.
\end{align*}
Thus, 
\begin{align*}
   \inp{\mZ^{k}_a}{\mB_a^k \mZ^{k}_a} 
   &\geq \min_{0 \leq x \leq \bar{x}} \left[n_{a} x + \kappa_a \left(\bar{x} - x\right)^2\right],
\end{align*}
where $\bar{x} \eqdef \frac{1}{n_a} \norm{\mZ^{k}_a}^2_F.$ Since 
\begin{align*}
   \min_{0 \leq x \leq \bar{x}} \left[n_{a} x + \kappa_a \left(\bar{x} - x\right)^2\right] \geq \min\left\{\min_{0 \leq x \leq \frac{\bar{x}}{2}}\kappa_a \left(\bar{x} - x\right)^2, \min_{\frac{\bar{x}}{2} \leq x \leq \bar{x}} n_a x\right\} = \min\left\{\frac{\kappa_a}{4} \bar{x}^2, \frac{n_a}{2} \bar{x}\right\},
\end{align*}
we obtain
\begin{align*}
   \inp{\mZ^{k}_a}{\mB_a^k \mZ^{k}_a} \geq \min\left\{\frac{\kappa_a}{4 n_a^2} \norm{\mZ^{k}_a}^4_F, \frac{1}{2} \norm{\mZ^{k}_a}^2_F\right\}.
\end{align*}
\end{proof}

\subsection{Corollaries}

The following corollary states and simplifies the result of Theorem~\ref{thm:mainmany} assuming that $d \geq n$ and ignoring logarithmic factors.

\THMSIMPL*

\begin{proof}
   It is sufficient to use $d \geq n$ and simplify the terms from the theorem.
\end{proof}

\THEOREMTV*

\begin{proof}
   In addition to the events $\Omega_1, \dots, \Omega_4,$ we also assume that the event $\Omega_5$ from Lemma~\ref{lemma:dist_clust} holds. Then, the result of Theorem~\ref{thm:mainmany} holds with probability at least $1 - 5\rho.$ Using the notations from the proof of Theorem~\ref{thm:mainmany},
   \begin{align*}
      \textnormal{TV}(p_{\mu^k}, \bar{p}) 
      &\eqdef \frac{1}{2} \int_{\R^d} \abs{p_{\mu^k}(x) - \bar{p}(x)} \, dx \\
      &= \frac{1}{2} \int_{\R^d} \abs{\frac1n\sum_{i=1}^n \varphi(x - \mu^k_i) - \frac1m\sum_{a=1}^m \varphi(x - \bar\mu_a)} \, dx \\
      &= \frac{1}{2} \int_{\R^d} \abs{\frac{1}{n} \sum_{a=1}^m \sum_{\ell \in I_a} \varphi(x - \mu^k_{\ell}) - \frac1m\sum_{a=1}^m \varphi(x - \bar\mu_a)} \, dx \\
      &= \frac{1}{2} \int_{\R^d} \abs{\sum_{a = 1}^m \left(\frac{1}{n} \sum_{\ell \in I_a} \varphi(x - \mu^k_{\ell}) - \frac1m \varphi(x - \bar\mu_a)\right)} \, dx \\
      &= \frac{1}{2} \int_{\R^d} \abs{\sum_{a = 1}^m \left(\frac{1}{n} \sum_{\ell \in I_a} (\varphi(x - \mu^k_{\ell}) - \varphi(x - \bar\mu_a)) + \left(\frac{n_a}{n} - \frac1m\right)\varphi(x - \bar\mu_a)\right)} \, dx \\
      &\leq \frac{1}{2 n} \sum_{a = 1}^m \sum_{\ell \in I_a}  \int_{\R^d} \abs{\varphi(x - \mu^k_{\ell}) - \varphi(x - \bar\mu_a)} \, dx + \frac{1}{2} \sum_{a = 1}^m  \abs{\frac{n_a}{n} - \frac1m}.
   \end{align*}
   It is well-known \citep{barsov1987estimates,devroye2018total} that for all $a, b \in \R^d,$
   \begin{align*}
      &\frac{1}{2}\int_{\R^d} \abs{\varphi(x - a) - \varphi(x - b)}\,dx = \Prob{-\frac{\norm{a - b}}{2} \leq \xi \leq \frac{\norm{a - b}}{2}},
   \end{align*}
   where $\xi$ is a standard normal random variable. Thus
   \begin{align*}
      \frac{1}{2} \int_{\R^d} \abs{\varphi(x - a) - \varphi(x - b)}\,dx \leq \frac{1}{\sqrt{2 \pi}} \norm{a - b}
   \end{align*}
   and 
   \begin{align*}
      \textnormal{TV}(p_{\mu^k}, \bar{p}) 
      &\leq \frac{1}{\sqrt{2 \pi} n} \sum_{a = 1}^m \sum_{\ell \in I_a} \norm{\mu^k_{\ell} - \bar\mu_a} + \frac{1}{2} \sum_{a = 1}^m  \abs{\frac{n_a}{n} - \frac1m}.
   \end{align*}
   Using the result of Theorem~\ref{thm:mainmany},
   \begin{align*}
      \textnormal{TV}(p_{\mu^k}, \bar{p}) \leq \frac{\sqrt{\varepsilon}}{\sqrt{2 \pi}} + \frac{1}{2} \sum_{a = 1}^m  \abs{\frac{n_a}{n} - \frac1m},
   \end{align*}
   for all $k \geq \tilde{K}.$ Next,
   \begin{align*}
      \textnormal{TV}(p_{\mu^k}, \bar{p}) 
      \leq \frac{\sqrt{\varepsilon}}{\sqrt{2 \pi}} + \frac{1}{2} \sum_{a = 1}^m  \abs{\frac{n_a}{n} - p_a} + \frac{1}{2} \sum_{a = 1}^m \abs{p_a - \frac1m},
   \end{align*}
   where $p_a$ is the probability that $\bar\mu_a$ is the closest vector to $\mu^0_{1}$ among $\{\bar\mu_b\}_{b \in [m]}.$
   Using \eqref{eq:IMEDMwOCneYdqoINSF},
   \begin{align*}
      \textnormal{TV}(p_{\mu^k}, \bar{p}) 
      \leq \frac{\sqrt{\varepsilon}}{2} + \sqrt{\frac{m^2}{8 n}\log\frac{2m}{\rho}} + \frac{1}{2} \sum_{a = 1}^m \abs{p_a - \frac1m}.
   \end{align*}
   Since $\frac{m^2}{8 n}\log\frac{2m}{\rho} \leq \frac{\varepsilon}{4},$ we get the result of the theorem.
\end{proof}

\COROLTV*

\begin{proof}
   In the symmetric case, $p_a = \frac{1}{m}$ for all $a \in [m]$, so
   \[
   \frac12 \sum_{a=1}^m \abs{p_a - \frac{1}{m}} = 0.
   \]
   By Theorem~\ref{thm:thm_tv}, this gives $\textnormal{TV}(p_{\mu^k}, \bar p) \leq \sqrt{\varepsilon}$.

   For orthogonal equal-norm modes, write $\bar\mu_a = r e_a$ in some orthonormal basis. Then
   \[
   a^*(\eta) = \argmin_{a \in [m]} \norm{R^0 \eta - \bar\mu_a}^2
   = \arg\max_{a \in [m]} \inp{\eta}{\bar\mu_a} = \arg\max_{a \in [m]} \eta_a,
   \]
   where $\eta \sim \mathrm{Unif}(\mathbb S^{d-1})$. The coordinates $\eta_1, \dots, \eta_m$ are exchangeable. Thus
   \[
   p_a = \Prob{a^*(\eta) = a} = \frac1m
   \]
   for all $a \in [m].$

\end{proof}

\section{Proof of Theorem~\ref{thm:lower_bound}}

\begin{restatable}[Lower bound]{theorem}{THMLOWERBOUND}
   \label{thm:lower_bound}
   Let $\varepsilon \in (0, 1]$ and $\rho \in (0, 1].$ We consider the problem \eqref{eq:main_gener} with $m = 1,$ $\bar\mu_{1,1} = \bar\mu_{2,1},$ $d \geq 8 + \max\left\{n_{\max}, 2 \log\left(\nicefrac{n_{\max}^2}{\rho}\right), 8 \log\left(\nicefrac{n_{\max}^4}{2 \pi \rho^2}\right)\right\},$ and $n_{\min} < \nicefrac{n_{\max}}{32}.$ Assume that the initialization radius satisfies
   $R^0 = \max\left\{\sqrt{2048 \log(55552 \cdot n_{\max})}, 1024 \sqrt{16 n_{\max} \max\{\log\left(2^{66} n_{\max}^3 \bar{R}^4\right), 1\}} \bar{R}\right\},$
   where $\bar{R} \eqdef \max\{\norm{\bar \mu_{1,1}}, \norm{\bar \mu_{2,1}}\}.$ We sample the initial parameters as $\mu^0_{i,\ell} = R^0 \eta_{i,\ell} / \norm{\eta_{i,\ell}}$, where $\eta_{i,\ell} \sim \mathcal{N}(0, \mI)$ for all $i \in \{1,2\}, \ell \in [n_i]$, and the random vectors $\{\eta_{i,\ell}\}_{i \in \{1,2\},\ell \in [n_i]}$ are i.i.d. If the step size $\gamma \leq 2 n_{\min},$ then, with probability at least $1 - \rho,$ for all $k < \tilde{K}_1 \eqdef \frac{n_{\max}}{16 n_{\min}} \times \log\left(\frac{R^0}{16 \sqrt{\log\left(55552 n_{\max} (R^0)^4\right)}}\right),$ there exists $i \in \{1,2\}$ and $\ell \in [n_i]$ such that 
   $\norm{\mu^{k}_{i,\ell} - \bar\mu_{i,1}}^2 > \varepsilon.$
    If instead $\gamma > 2 n_{\min},$ then, with probability at least $1 - \rho,$ the method diverges.
\end{restatable}

\begin{proof}
   We define $\bar\mu_{1} \eqdef \bar\mu_{1,1} = \bar\mu_{2,1}.$ If the step size $\gamma \leq 2 n_{\min},$
   then the proof is the same as in Theorem~\ref{thm:mainmany}.
   Consider an event $\Omega = \Omega_2,$ where
   \begin{align*}
      \Omega_2 &\eqdef \left\{ \forall i \neq j \in [n_{\max}]: \norm{\mu^0_{2,i} - \mu^0_{2,j}}^2 \geq \Delta_2 \right\}.
   \end{align*}
   Using Lemma~\ref{lemma:far}, we get $\Prob{\Omega} \geq 1 - \rho$ with $\Delta_2 = (R^0)^2 \left(\frac{\sqrt{2\pi}\rho}{\sqrt d\, n_{\max}^2}\right)^{\frac{2}{d-1}}.$ 
   Similarly to the proof of Theorem~\ref{thm:mainmany}, we obtain
   \begin{align}
      \label{eq:OereByxtlYEseHusytmz2}
      \mu^{k+1}_{2,i} = \bar \mu_{1} + \left(1 - \frac{2 \gamma}{n_{\max}}\right)^{k+1} (\mu^0_{2,i} - \bar \mu_{1}) - E_{2,i}^k,
   \end{align}
   where
   \begin{align*}
      E_{2,\ell}^k \eqdef \frac{2 \gamma}{n_{\max}}\sum_{j=0}^{k} \left(1 - \frac{2 \gamma}{n_{\max}}\right)^{k - j} \mathbb E_{x \sim \mathcal N(\mu^j_{2,\ell},\mI)}
      \left[\bar{\mathcal{E}}^j_{2,\ell}(x)\right],
   \end{align*}
   and
   \begin{align*}
      \bar{\mathcal{E}}^k_{2,\ell}(x) \eqdef m_{\mu_2^k}(x) - \mu^k_{2,\ell} + \mC_{\mu_2^k}(x)\bigl(\mu^k_{2,\ell}-m_{\mu_2^k}(x)\bigr).
   \end{align*}
   Moreover
   \begin{align}
      \norm{E_{2,\ell}^k} \leq \max_{j \in \{0, \dots, k\}}\mathbb E_{x \sim \mathcal N(\mu^j_{2,\ell},\mI)}
      \left[\norm{\bar{\mathcal{E}}^j_{2,\ell}(x)}\right].
   \end{align}
   This holds since $\gamma < \frac{n_{\max}}{2}.$
   Similarly, for all $0 \leq k \leq \bar{K},$ we can prove that 
   \begin{align}
      &\norm{\mu^k_{2,\ell}} \leq R, \label{eq:R2}\\
      &\norm{\mu^{k}_{2,\ell} - \mu^{k}_{2,j}}^2 \geq \hat\Delta_2 \label{eq:Delta22} 
   \end{align}
   for all $\ell \neq j \in [n_{\max}],$ where 
   \begin{align}
      \label{eq:kjdsqYLZ2}
      R \eqdef 2 R^0, \quad \hat\Delta_2 \eqdef 32 \log\left(3472 n_{\max} R^4\right), \quad \bar{\Delta} \eqdef \frac{\hat\Delta_2}{2},
   \end{align}
   and
   \begin{align*}
      \bar{K} \eqdef \flr{\log\left(\frac{R^0}{4 \sqrt{\bar{\Delta}}}\right) / \left(-\log\left(1 - \frac{2 \gamma}{n_{\max}}\right)\right)}.
   \end{align*}
   Moreover,
   \begin{align*}
      \norm{E_{2,\ell}^k} \leq 248 n_{\max} \max\{R, R^3\} \exp\left(-\hat\Delta_2 / 32\right) \leq \frac{1}{4}
   \end{align*}
   for all $0 \leq k \leq \bar{K} - 1.$
   Using \eqref{eq:OereByxtlYEseHusytmz2}, for all $1 \leq k \leq \bar{K},$
   \begin{align*}
      \norm{\mu^{k}_{2,i} - \bar \mu_{1}} 
      &\geq \left(1 - \frac{2 \gamma}{n_{\max}}\right)^{k} \norm{(\mu^0_{2,i} - \bar \mu_{1})} - \norm{E_{2,i}^{k-1}} \\
      &\geq 2 \sqrt{\bar{\Delta}} - \frac{1}{4} \geq 2 > \sqrt{\varepsilon}
   \end{align*}
   due to the choice of $\bar{K}.$ Thus, the method required at least 
   \begin{align*}
      &\flr{\log\left(\frac{R^0}{4 \sqrt{\bar{\Delta}}}\right) / \left(-\log\left(1 - \frac{2 \gamma}{n_{\max}}\right)\right)} \geq \frac{n_{\max}}{16 n_{\min}} \log\left(\frac{R^0}{16 \sqrt{\log\left(55552 n_{\max} (R^0)^4\right)}}\right)
   \end{align*}
   iterations.

   We now consider the case when $\gamma > 2 n_{\min}.$ 
   Consider an event $\Omega = \Omega_1,$ where
   \begin{align}
      \Omega_1 &\eqdef \left\{ \forall i \neq j \in [n_{\min}]: \norm{\mu^0_{1,i} - \mu^0_{1,j}}^2 \geq \Delta_1 \right\} \label{eq:EvPlrGmTnwYCRQa2}
   \end{align}
   and $\Delta_1 = (R^0)^2 \left(\frac{\sqrt{2\pi}\rho}{\sqrt d\, n_{\min}^2}\right)^{\frac{2}{d-1}}.$ Using Lemma~\ref{lemma:far}, we get \(\Prob{\Omega_1} \geq 1-\rho\).
   Using mathematical induction, we assume that
   \begin{align}
      &\norm{\mu^{k}_{1,i}} \leq 2 R \left(\abs{1 - \frac{2 \gamma}{n_{\min}}}\right)^{k}, \label{eq:bGabYxOeWSPolMlw}\\
      &\norm{\mu^{k}_{1,i} - \mu^{k}_{1,j}}^2 \geq \left(\abs{1 - \frac{2 \gamma}{n_{\min}}}\right)^{2 k} \frac{\Delta_1}{2} \label{eq:nSjoGnew}
   \end{align}
   for all $k \geq 0,$ and for all $i \neq j \in [n_{\min}],$ where 
   \begin{align*}
      R \eqdef 2 R^0.
   \end{align*}
   The base case holds due to $\norm{\mu^{0}_{1,i}} = R^0$ and \eqref{eq:EvPlrGmTnwYCRQa2}. Similarly to the proof of Theorem~\ref{thm:mainmany}, we get 
   \begin{align}
      \label{eq:OereByxtlYEseHusytmz}
      \mu^{k+1}_{1,i} = \bar \mu_{1} + \left(1 - \frac{2 \gamma}{n_{\min}}\right)^{k+1} (\mu^0_{1,i} - \bar \mu_{1}) - E_{1,i}^k,
   \end{align}
   where
   \begin{align*}
      E_{1,\ell}^k \eqdef \frac{2 \gamma}{n_{\min}}\sum_{j=0}^{k} \left(1 - \frac{2 \gamma}{n_{\min}}\right)^{k - j} \mathbb E_{x \sim \mathcal N(\mu^j_{1,\ell},\mI)}
      \left[\bar{\mathcal{E}}^j_{1,\ell}(x)\right],
   \end{align*}
   and
   \begin{align*}
      \bar{\mathcal{E}}^k_{1,\ell}(x) \eqdef m_{\mu_1^k}(x) - \mu^k_{1,\ell} + \mC_{\mu_1^k}(x)\bigl(\mu^k_{1,\ell}-m_{\mu_1^k}(x)\bigr).
   \end{align*}
   Let us define $r_{1,i}^{k}(x)
   \eqdef
   \frac{
   \exp\left(-\frac12\|x-\mu^k_{1,i}\|^2\right)
   }{
   \sum_{j\in [n_{\min}]}
   \exp\left(-\frac12\|x-\mu^k_{1,j}\|^2\right)
   }$ and $\tau^k_{1, \ell}(x) \eqdef \sum_{i \in [n_{\min}], i \neq \ell} r_{1,i}^{k}(x).$ Using the same arguments as in Lemma~\ref{lemma:bound2} and \eqref{eq:RpwaEvNpBJdp},
   \begin{align*}
      \norm{\bar{\mathcal{E}}^k_{1,\ell}(x)} \leq 80 R^3 \left(\abs{1 - \frac{2 \gamma}{n_{\min}}}\right)^{3 k} \tau^k_{1,\ell}(x)
   \end{align*}
   with
   \begin{align*}
      &\mathbb E_{x \sim \mathcal N(\mu^k_{1,\ell},\mI)} \tau^k_{1,\ell}(x) \leq 2 n_{\min} \exp\left(-\left(\abs{1 - \frac{2 \gamma}{n_{\min}}}\right)^{2 k} \frac{\Delta_1}{64}\right),
   \end{align*}
   where we use \eqref{eq:nSjoGnew}. Therefore,
   \begin{align*}
      \norm{E_{1,\ell}^k} 
      &\leq \norm{\frac{2 \gamma}{n_{\min}}\sum_{j=0}^{k} \left(1 - \frac{2 \gamma}{n_{\min}}\right)^{k - j} \mathbb E_{x \sim \mathcal N(\mu^j_{1,\ell},\mI)}
      \left[\bar{\mathcal{E}}^j_{1,\ell}(x)\right]} \\
      &\leq \frac{2 \gamma}{n_{\min}} \sum_{j=0}^{k} \left(\abs{1 - \frac{2 \gamma}{n_{\min}}}\right)^{k - j} \mathbb E_{x \sim \mathcal N(\mu^j_{1,\ell},\mI)}\norm{\bar{\mathcal{E}}^j_{1,\ell}(x)} \\
      &\leq 320 R^3 \gamma\sum_{j=0}^{k} \left(\abs{1 - \frac{2 \gamma}{n_{\min}}}\right)^{k - j} \left(\abs{1 - \frac{2 \gamma}{n_{\min}}}\right)^{3 j} \exp\left(-\left(\abs{1 - \frac{2 \gamma}{n_{\min}}}\right)^{2 j} \frac{\Delta_1}{64}\right).
   \end{align*}
   Let $x \eqdef \abs{1 - \frac{2 \gamma}{n_{\min}}} > 2,$ then
   \begin{align*}
      \norm{E_{1,\ell}^k} 
      &\leq 320 R^3 \gamma  \left(\abs{1 - \frac{2 \gamma}{n_{\min}}}\right)^k \sum_{j=0}^{k} x^{2 j} \exp\left(-x^{2 j} \frac{\Delta_1}{64}\right).
   \end{align*}
   Due to $x \geq 2,$
   \begin{align*}
      \norm{E_{1,\ell}^k} 
      &\leq 320 R^3 \gamma  \exp\left(-\frac{\Delta_1}{128}\right) \left(\abs{1 - \frac{2 \gamma}{n_{\min}}}\right)^k \sum_{j=0}^{k} x^{2 j} \exp\left(-x^{2 j} \frac{\Delta_1}{128}\right).
   \end{align*}
   Since $\Delta_1 \geq 128$ for our choice of $R^0,$ $$\sum_{j=0}^{k} x^{2 j} \exp\left(-x^{2 j} \frac{\Delta_1}{128}\right) \leq \sum_{j=0}^{k} x^{2 j} \exp\left(-x^{2 j}\right) \leq e^{-1} + \int_{1}^{\infty} t e^{-t} \, d t \leq 2,$$
   and
   \begin{align*}
      \norm{E_{1,\ell}^k} 
      &\leq \frac{640 R^3 \gamma}{\exp\left(\frac{\Delta_1}{128}\right)} \left(\abs{1 - \frac{2 \gamma}{n_{\min}}}\right)^k = \frac{640 R^3 n_{\min}}{2 \exp\left(\frac{\Delta_1}{128}\right)} \times \frac{2 \gamma}{n_{\min}}  \left(\abs{1 - \frac{2 \gamma}{n_{\min}}}\right)^k \\
      &\leq \frac{640 R^3 n_{\min}}{\exp\left(\frac{\Delta_1}{128}\right)} \left(\abs{1 - \frac{2 \gamma}{n_{\min}}}\right)^{k+1},
   \end{align*}
   where we use that $\frac{2 \gamma}{n_{\min}} \leq 2 \abs{1 - \frac{2 \gamma}{n_{\min}}}.$
   Recall that 
   \begin{align*}
      \Delta_1 = (R^0)^2 \left(\frac{\sqrt{2\pi}\rho}{\sqrt d\, n_{\min}^2}\right)^{\frac{2}{d-1}} \geq \frac{(R^0)^2}{2}
   \end{align*}
   due to the condition on $d.$ Thus 
   \begin{align*}
      \norm{E_{1,\ell}^k} \leq \frac{5120 (R^0)^3 n_{\min}}{\exp\left(\frac{\Delta_1}{128}\right)} \left(\abs{1 - \frac{2 \gamma}{n_{\min}}}\right)^{k+1} \leq \frac{5120 (R^0)^3 n_{\min}}{\exp\left(\frac{(R^0)^2}{512}\right)} \left(\abs{1 - \frac{2 \gamma}{n_{\min}}}\right)^{k+1}.
   \end{align*}
   Since $R^0 \geq \sqrt{2048 \log(55552 \cdot n_{\max})},$ we get
   \begin{align}
      \label{eq:TiIdzOw}
      \norm{E_{1,\ell}^k} \leq \frac{1}{4} \left(\abs{1 - \frac{2 \gamma}{n_{\min}}}\right)^{k+1}.
   \end{align}
   Using \eqref{eq:OereByxtlYEseHusytmz},
   \begin{align*}
      \norm{\mu^{k+1}_{1,i} - \mu^{k+1}_{1,j}} 
      &\geq \abs{\left(1 - \frac{2 \gamma}{n_{\min}}\right)^{k+1}}\norm{\mu^0_{1,i} - \mu^0_{1,j}} - \norm{E_{1,i}^k} - \norm{E_{1,j}^k} \\
      &\geq \left(\abs{1 - \frac{2 \gamma}{n_{\min}}}\right)^{k+1} \sqrt{\Delta_1} - \norm{E_{1,i}^k} - \norm{E_{1,j}^k} \geq \left(\abs{1 - \frac{2 \gamma}{n_{\min}}}\right)^{k+1} \frac{\sqrt{\Delta_1}}{\sqrt{2}}
   \end{align*}
   since $\Delta_1 \geq 16,$
   and 
   \begin{align*}
      \norm{\mu^{k+1}_{1,\ell}} 
      &= \norm{\bar \mu_{1} + \left(1 - \frac{2 \gamma}{n_{\min}}\right)^{k+1} (\mu^0_{1,\ell} - \bar \mu_{1}) - E_{1,\ell}^k} \\
      &\leq \norm{\bar \mu_{1}} + \left(\abs{1 - \frac{2 \gamma}{n_{\min}}}\right)^{k+1}(\norm{\mu^0_{1,\ell}} + \norm{\bar \mu_{1}}) + \norm{E_{1,\ell}^k}.
   \end{align*}
   Since $\norm{\bar \mu_{1}} \leq \frac{R}{2},$ $\norm{\mu^0_{1,\ell}} = R^0 \leq \frac{R}{2},$ and $\abs{1 - \frac{2 \gamma}{n_{\min}}} \geq 1,$
   \begin{align*}
      \norm{\mu^{k+1}_{1,\ell}} \leq \frac{3 R}{2}\left(\abs{1 - \frac{2 \gamma}{n_{\min}}}\right)^{k+1} + \norm{E_{1,\ell}^k} \leq 2 R \left(\abs{1 - \frac{2 \gamma}{n_{\min}}}\right)^{k+1}
   \end{align*}
   because $R \geq 1.$ We have proved the next step of the mathematical induction. For all $k \geq 0,$ 
   \begin{align*}
      \norm{\mu^{k+1}_{1,\ell}} 
      &\geq \norm{\bar \mu_{1} + \left(1 - \frac{2 \gamma}{n_{\min}}\right)^{k+1} (\mu^0_{1,\ell} - \bar \mu_{1}) - E_{1,\ell}^k} \\
      &\geq \left(\abs{1 - \frac{2 \gamma}{n_{\min}}}\right)^{k+1}(\norm{\mu^0_{1,\ell}} - \norm{\bar \mu_{1}}) - \norm{\bar \mu_{1}} - \norm{E_{1,\ell}^k} \\
      &\geq \left(\abs{1 - \frac{2 \gamma}{n_{\min}}}\right)^{k+1}(R^0 - \bar{R}) - \bar{R} - \frac{1}{4} \left(\abs{1 - \frac{2 \gamma}{n_{\min}}}\right)^{k+1} \\
      &\geq \left(\abs{1 - \frac{2 \gamma}{n_{\min}}}\right)^{k+1} \frac{R^0}{8} \geq \frac{2^{k+1} R^0}{8}
   \end{align*}
   since $R^0 \geq 16 \bar{R}$ and $R^0 \geq 2.$ The last inequality proves that $\{\mu^{k}_{1,\ell}\}_{k \geq 0}$ diverges.
\end{proof}

\end{document}

%% file: macros.tex
\usepackage{graphicx}
\usepackage{apptools}
\usepackage[flushleft]{threeparttable}
\usepackage{array,booktabs,makecell}
\usepackage{multirow}
\usepackage{nicefrac}
\usepackage{wrapfig}
\usepackage{caption}
\usepackage{siunitx}
\usepackage{nccmath}
\usepackage{empheq}
\usepackage{bbm}
\usepackage{tabularx}

\usepackage{algorithmic}
\usepackage{algorithm}

\usepackage{color}
\usepackage{colortbl}
\definecolor{bgcolor}{rgb}{0.76,0.88,0.50}
\definecolor{bgcolor0}{rgb}{0.93,0.99,1}
\definecolor{bgcolor1}{rgb}{0.8,1,1}
\definecolor{bgcolor2}{rgb}{0.8,1,0.8}
\definecolor{bgcolor3}{rgb}{0.50,0.90,0.50}
\usepackage{tcolorbox}
\usepackage{pifont}

\definecolor{mydarkgreen}{RGB}{39,130,67}
\definecolor{mydarkorange}{RGB}{236,147,14}
\definecolor{mydarkred}{RGB}{192,47,25}
\definecolor{ruby}{RGB}{155,17,30}
\definecolor{chili}{RGB}{191,0,0}
\definecolor{sangria}{RGB}{146,0,10}
\definecolor{burgundy}{RGB}{128,0,32} 
\definecolor{darkred}{RGB}{132,0,0} 
\definecolor{cherry}{RGB}{192,0,0} 
\definecolor{grey}{RGB}{80,80,80} 

\definecolor{blue}{RGB}{0,0,255}

\usepackage[textsize=tiny]{todonotes}

\usepackage{xspace}

\usepackage[scaled=0.86]{helvet}

\newcommand{\norm}[1]{\left\| #1 \right\|}

\newcommand{\inp}[2]{\left\langle#1,#2\right\rangle} 
\newcommand{\abs}[1]{\left| #1 \right|}

\newcommand{\R}{\mathbb{R}} 

\newcommand{\Exp}[1]{{\mathbb{E}}\left[#1\right]}

\newcommand{\Prob}[1]{\mathbb{P}\left(#1\right)} 

\newcommand{\cO}{\mathcal{O}}

\newcommand{\cZ}{\mathcal{Z}}

\newcommand{\mA}{\mathbf{A}}
\newcommand{\mB}{\mathbf{B}}
\newcommand{\mC}{\mathbf{C}}

\newcommand{\mI}{\mathbf{I}}

\newcommand{\mL}{\mathbf{L}}

\newcommand{\mP}{\mathbf{P}}
\newcommand{\mZ}{\mathbf{Z}}
\newcommand{\mU}{\mathbf{U}}
\newcommand{\mE}{\mathbf{E}}

\newcommand{\eqdef}{:=}

\makeatletter
\newcommand{\vast}{\bBigg@{4}}

\DeclareMathOperator*{\argmin}{arg\,min}

\newcommand{\flr}[1]{\left\lfloor #1\right\rfloor} 
\newcommand{\ceil}[1]{\left\lceil #1\right\rceil}

\usepackage{longtable}